\documentclass[11pt,letterpaper]{article}
\usepackage[margin=1in]{geometry}

\usepackage{graphicx}
\usepackage{booktabs}
\usepackage{url}
\usepackage{cite}
\usepackage[title]{appendix}
\usepackage[utf8]{inputenc} 
\usepackage[T1]{fontenc}    
\usepackage{hyperref}       
\usepackage{url}            
\usepackage{booktabs}       
\usepackage{amsfonts}       
\usepackage{nicefrac}       
\usepackage{microtype}      

\usepackage{algorithm}
\usepackage{algorithmic}
\usepackage{enumitem}
\usepackage{authblk}
\usepackage{breakcites}
\usepackage{amsmath}
\usepackage{bbm}
\usepackage{verbatim}
\usepackage[font=small]{caption}
\usepackage{amsthm}
\usepackage{subcaption}
\usepackage{mathtools}
\usepackage{xcolor}
\usepackage[capitalize,noabbrev]{cleveref}
\usepackage{thm-restate}

\DeclareMathOperator*{\argmin}{arg\,min}
\newcommand{\norm}[1]{\left\lVert#1\right\rVert}
\newcommand{\normt}[1]{\left\lVert#1\right\rVert_2}
\newcommand{\normf}[1]{\left\lVert#1\right\rVert_F}
\newcommand{\Ex}[1]{\mathbb{E}\left[#1\right]}
\newcommand{\prb}[1]{\mathbb{P}\left(#1\right)}
\newcommand{\Inp}[1]{\left\langle #1 \right\rangle}

\newcommand{\ssigmin}{\bar{\sigma}_{\min,*}}
\newcommand{\ssigmax}{\bar{\sigma}_{\max,*}}
\newcommand{\sigmin}{{\sigma}_{\min}}
\newcommand{\sigmax}{{\sigma}_{\max}}
\newcommand{\var}{{\sigma}}
\newcommand{\sumi}{{\sum_{i=1}^n}}

\newcommand{\w}{\mathbf{w}}
\newcommand{\x}{\mathbf{x}}
\newcommand{\q}{\mathbf{q}}
\newcommand{\cE}{\mathcal{E}}
\newcommand{\hd}{\hat{\delta}}
\newcommand{\W}{\mathbf{W}}
\newcommand{\X}{\mathbf{X}}
\newcommand{\Y}{\mathbf{Y}}
\newcommand{\y}{\mathbf{y}}
\newcommand{\Z}{\mathbf{Z}}
\newcommand{\A}{\mathbf{A}}
\newcommand{\B}{\mathbf{B}}
\newcommand{\bS}{\mathbf{S}}
\newcommand{\E}{\mathbf{E}}
\newcommand{\bH}{\mathbf{H}}
\newcommand{\hB}{\hat{\mathbf{B}}}
\newcommand{\I}{\mathcal{I}}
\newcommand{\mT}{\mathcal{T}}
\newcommand{\mC}{\mathcal{C}}
\newcommand{\bI}{\mathbf{I}}
\newcommand{\R}{\mathbf{R}}
\newcommand{\G}{\mathbf{G}}
\newcommand{\F}{\mathbf{F}}
\newcommand{\U}{\mathbf{U}}
\newcommand{\V}{\mathbf{V}}
\newcommand{\bv}{\mathbf{v}}

\newcommand{\n}{n}
\newcommand{\bu}{\mathbf{u}}
\newcommand{\p}{\mathbf{p}}
\newcommand{\maxp}{\max_{\p:\normt{\p}=1}}
\newcommand{\Q}{\mathbf{Q}}
\newcommand{\OO}{\mathcal{O}}

\newcommand{\dist}{\text{dist}}
\newcommand{\distB}{\dist\pr{\hB, \hB^*}}
\newcommand{\spn}{\textit{span}}
\newcommand{\coloneqq}{:=}
\newcommand{\br}[1]{\left[#1\right]}
\newcommand{\tr}[1]{\text{tr}\left[#1\right]}
\newcommand{\pr}[1]{\left(#1\right)}
\newcommand{\ag}[1]{\left\{#1\right\}}
\newcommand{\EXP}[2][]{
    \ifthenelse{\equal{#1}{}}
    {\mathbb{E}\left[#2\right]}
    {\mathop{\mathbb{E}}_{#1}\left[#2\right]}
}
\newcommand{\TT}[1]{\Ex{\mT_{n_02^{#1}}}}
\renewcommand{\O}{\mathcal{O}}

\theoremstyle{plain}
\newtheorem{theorem}{Theorem}[section]

\newtheorem{lemma}[theorem]{Lemma}
\newtheorem{corollary}[theorem]{Corollary}
\theoremstyle{definition}
\newtheorem{definition}[theorem]{Definition}
\newtheorem{assumption}[theorem]{Assumption}
\theoremstyle{remark}
\newtheorem{remark}[theorem]{Remark}

\title{Straggler-Resilient Personalized Federated Learning}

\title{Straggler-Resilient Personalized Federated Learning}
\author[1]{Isidoros Tziotis\thanks{E-mail: \texttt{isidoros\_13@utexas.edu}.}}
\author[2]{Zebang Shen}
\author[3]{Ramtin Pedarsani}
\author[2]{Hamed Hassani}
\author[1]{Aryan Mokhtari}
\affil[1]{The University of Texas at Austin}
\affil[2]{University of Pennsylvania}
\affil[3]{UC Santa Barbara}

\begin{document}

\maketitle
\begin{abstract}
\emph{Federated Learning} is an emerging learning paradigm that allows training models from samples distributed across a large network of clients while respecting privacy and communication restrictions. Despite its success, federated learning faces several challenges related to its decentralized nature. In this work, we develop a novel algorithmic procedure with theoretical speedup guarantees that simultaneously handles two of these hurdles, namely (i) \textit{data heterogeneity}, i.e., data distributions can vary substantially across clients, and (ii) \textit{system heterogeneity}, i.e., the computational power of the clients could differ significantly. Our method relies on  ideas from representation learning theory to find a global common representation using all clients' data and learn a user-specific set of parameters leading to a personalized solution for each client.
Furthermore, our method mitigates the effects of stragglers by adaptively selecting clients based on their computational characteristics and statistical significance, thus achieving, for the first time, near optimal sample complexity and provable logarithmic speedup. Experimental results support our theoretical findings showing the superiority of our method over alternative personalized federated schemes in system and data heterogeneous environments.
\end{abstract}
\section{Introduction}

Due to growing concerns on data privacy and communication cost, Federated Learning (FL) has become an emerging learning paradigm as it allows for training machine learning models without collecting local data from the clients. Due to its decentralized nature, a major challenge in designing efficient solvers for FL, is heterogeneity of local devices which can be categorized into two different types: (i) \textit{data heterogeneity} where the underlying data distributions of clients vary substantially, and (ii) \textit{system heterogeneity} where the computational and storage capabilities  of devices differ significantly. In fact, it has been observed that the seminal Federated Averaging (FedAvg) method suffers from slow convergence to a high quality solution when facing highly heterogeneous datasets~\cite{mcmahan2017communication} as well as heterogeneous systems \cite{li2018federated, kairouz2019advances}. 

In this paper, we aim to address these two challenges simultaneously by proposing a generic framework that includes algorithms which exhibit robust performance in the presence of those forms of clients' heterogeneity. Specifically, we propose a meta-algorithm that produces personalized solutions and handles data heterogeneity by leveraging a global representation shared among all clients. Further, our method circumvents the delays introduced due to the presence of stragglers in the network, by carefully selecting participating nodes based on their computational speeds. In early stages, only a few of the fastest nodes are chosen to participate and sequentially slower devices are included in the training process until the target accuracy is achieved. One might be concerned that our scheme introduces unfairness or bias towards the frequently selected, faster nodes. However, we highlight that our method achieves speedup without any penalty in terms of accuracy. The most significant contribution of our work is achieving near-optimal sample complexity with heterogeneous data alongside a provable logarithmic speedup in running time. Next, we summarize our contributions.

\vspace{-0.5mm}
\begin{enumerate}
\vspace{-1mm}
    \item \textbf{\texttt{SRPFL} Algorithm}. We propose and analyze for the first time the Straggler-Resilient Personalized Federated Learning (\texttt{SRPFL}) meta-algorithm, an adaptive-node-participation method which accommodates gradient-based subroutines and leverages ideas from representation learning theory to find personalized models in a federated, straggler-resilient fashion. 
    \vspace{-1mm}
    \item \textbf{Logarithmic Speedup}. Assuming that clients’ speeds are drawn from the exponential distribution, we prove that \texttt{SRPFL} guarantees logarithmic speedup in the linear representation setting, outperforming established, straggler-prone benchmarks while maintaining the state of the art sample complexity per client $m=\O((d/N+\log(N)))$, where $d$ and $N$ denote the feature vector size and number of active nodes, respectively. Our results hold for non-convex loss functions, heterogeneous data and dynamically changing client's speeds.
    \vspace{-1mm}
    \item \textbf{Numerical Results}. Our experiments on different datasets (CIFAR10, CIFAR100, EMNIST, FEMNIST) support our theoretical results and show that: (i) \texttt{SRPFL} significantly boosts the performance of different subroutines designed for personalized FL both in full and in partial participation settings and (ii) \texttt{SRPFL} presents superior performance in system and data heterogeneous settings compared to state-of-the-art benchmarks.
    \vspace{-2mm}
\end{enumerate}
\vspace{-0.9mm}
\subsection{Related Work}
\vspace{-0.9mm}

\noindent \textbf{Data heterogeneity.} In data heterogeneous settings, if one aims at minimizing the aggregate loss in the network using the classic FedAvg method or more advanced algorithms, which utilize control-variate techniques, such as \texttt{SCAFFOLD} \cite{karimireddy2019scaffold}, \texttt{FedDyn} \cite{acar2021federated}, or \texttt{FEDGATE} \cite{haddadpour2020federated}, the resulting solution could perform poorly for some of the clients. This is an unavoidable hurdle simply due to the fact that there is no single model that works well for all clients when their underlying data distributions are diverse.\
A common technique that addresses this issue and works well in practice is  fine-tuning the derived global model to each local task by following a few steps of SGD updates~\cite{wang2019federated, yu2020salvaging}. Based on this observation, \cite{fallah2020personalized} showed that one might need to train models that work well after fine-tuning and showed its connections to Model-Agnostic Meta-Learning (MAML). In \cite{cho2022towards,balakrishnan2021diverse} novel client-sampling schemes are explored achieving increased efficiency in regimes with data
heterogeneity. Another line of work for personalized FL is learning additive mixtures of local and global models \cite{deng2020adaptive,mansour2020approaches, hanzely2021federated}. These methods learn local models for clients that are close to each other in some norm, an idea closely related to multi-task FL~\cite{smith2017federated, hu2021private}.  The works in~\cite{chen2022actperfl, lee2022partial} study the interplay of local and global models utilizing Bayesian hierarchical models and partial participation, respectively.
An alternative approach was presented in~\cite{collins2021exploiting}, where instead of enforcing local models to be close, the authors assume that models among clients share a common representation. Using this perspective, they presented a novel iterative method called \texttt{FedRep} that provably learns the common structure in the linear representation setting. 
However, in all of the aforementioned methods, (a subset of) clients participate regardless of their computational capabilities. Thus, in the presence of stragglers, the speed of the whole training process significantly goes down as the server waits, at every communication round, for the slowest participating node to complete its local updates.



\noindent \textbf{System heterogeneity.}\
Several works have attempted to address the issue of system heterogeneity.\ Specifically, asynchronous methods, which rely on bounded staleness of slow clients, have demonstrated significant gains in distributed data centers \cite{xie2019asynchronous,stich2019local, so2021secure}. In FL frameworks, however, stragglers could be arbitrarily slow casting these methods inefficient. To manually control staleness, deadline-based computation has been proposed \cite{reisizadeh2019robust}, however in the worst case scenario the deadline is still determined by the slowest client in the network. Active sampling is another approach in which the server aims to aggregate as many local updates as possible within a predefined time span \cite{nishio2019client}. In~\cite{cho2021personalized} clients use heterogeneous model architectures and transfer knowledge to nodes with similar data distributions.  More recently, normalized averaging methods were proposed in \cite{wang2020tackling,horvath2022fedshuffle} that rectify the objective inconsistency created by the mismatch in clients' updates. 
A novel approach to mitigate the effects of system heterogeneity was recently proposed in \cite{reisizadeh2020stragglerresilient}. Their proposed scheme, called \texttt{FLANP}, employs adaptive node participation where clients are selected to take part in different stages of the training according to their computational speeds.\  
Alas, all of the above methods yield improvement only in data-homogeneous settings and they are not applicable in regimes with data heterogeneity.

\section{Problem Formulation and Setup}\label{Section_ProbForm}

In this section, we introduce our setting and define the data and system heterogeneity models that we study.
Consider the Federated Learning framework where $M$ clients interact with a central server. We focus on  a supervised, data-heterogeneous setting where client $i$ draws data from distribution $\mathcal{D}_i$, potentially with $\mathcal{D}_i \neq \mathcal{D}_j$.\ Further, consider the learning model of the $i$-th client as $q_i:\mathbb{R}^d\xrightarrow{} \mathcal{Y}$ which maps inputs $\x_i \in \mathbb{R}^d$ to predicted labels $q_i(\x_i) \in \mathcal{Y}$. 
The objective function of client $i$ is defined as  $f_i(q_i):=\mathbb{E}_{(\x_i,y_i)\sim\mathcal{D}_i}\br{\ell (q_i(\x_i) ,y_i))}$, where the loss $\ell : \mathcal{Y} \times \mathcal{Y} \xrightarrow{} \mathbb{R}$ penalizes the gap between the predicted label $q_i(\x_i)$ and the true label $y_i$. In the most general case, clients aim to solve
\begin{equation} \label{Personalized_Form1}
    \min_{\pr{q_1,...,q_M} \in \mathcal{Q}} \frac{1}{M}\sum_{i=1}^M f_i (q_i),
\end{equation}
where $\mathcal{Q}$ is the space of feasible tuples of $M$ mappings $(q_1,...,q_M)$. Traditionally in federated learning, methods focused on learning a single shared model $q=q_1=...=q_M$ that performs well on average, across clients \cite{DBLP:journals/corr/abs-1812-06127, pmlr-v54-mcmahan17a}. Although such a solution may be satisfactory in data-homogeneous environments, it leads to undesirable local models when the data distributions are diverse. Indeed, in the presence of data heterogeneity the loss functions $f_i$ have different forms and their minimizers could be far from each other. This justifies the formulation given in \eqref{Personalized_Form1} and necessitates the search for personalized solutions that can be learned in federated fashion.


\textbf{Low Dimensional Common Representation}. There have been numerous examples in image classification and word prediction where tasks with heterogeneous data share a common, low dimensional representation, despite having different labels \cite{6472238, article}. Based on that, a reasonable choice for $\mathcal{Q}$ is a set in which all $q_i$ share a common map, coupled with a personalized map that fits their local data. To formalize this, suppose the ground-truth map can be written for each client $i$ as $q_i=h_i \circ \phi$ where  $\phi:\mathbb{R}^d\xrightarrow{}\mathbb{R}^k$ is a shared global representation which maps $d$-dimensional data points to a lower dimensional space of size $k$ and $h_i:\mathbb{R}^k\xrightarrow{}\mathcal{Y}$, which maps from the lower dimensional subspace to the space of labels. Typically $k\ll d$ and thus given any fixed representation $\phi$, the client specific heads $h_i: \mathbb{R}^k\xrightarrow{} \mathcal{Y}$ are easy to optimize locally. With this common structure into consideration, \eqref{Personalized_Form1} can be reformulated as  $\min_{\phi \in \Phi} \frac{1}{M}\sum_{i=1}^M \min_{h_i\in \mathcal{H}}f_i (h_i \circ \phi)$,
where $\Phi$ is the class of feasible representation and $\mathcal{H}$ is the class of feasible heads. This formulation leads to good local solutions, if the underlying data generation models for the clients share a low dimensional common representation, i.e., $y_i= h_i \circ \phi(\x_i)+z_i$, where $z_i$ is some additive noise. 

The server and clients collaborate to learn the common representation $\phi$, while locally each client learns their unique head $h_i$. Since clients often do not have access to their data distributions, instead of minimizing their expected loss, they settle for minimizing the empirical loss associated with their local samples. Specifically, we assume client $i$ has access to $S_i$ samples $\{\x_i^1, \x_i^2,...,\x_i^{S_i} \}$, and its local empirical loss is  $\hat{f}_i (h_i \circ \phi)= \frac{1}{S_i}\sum_{s=1}^{S_i}\ell (h_i \circ \phi(\x_i^s) ,y_i^s)$. Hence, the global problem becomes
\vspace{-1mm}
\begin{equation} \label{eqn_FRL_empirical}
    \min_{\phi \in \Phi} \frac{1}{M}\sum_{i=1}^M \min_{h_i\in \mathcal{H}}\left\{ \hat{f}_i (h_i \circ \phi):= \frac{1}{S_i}\sum_{s=1}^{S_i}\ell (h_i \circ \phi(\x_i^s) ,y_i^s)\right\}
\end{equation}
\textbf{System Heterogeneity Model}. 
In  most federated learning settings, thousands of clients participate in the training process, each with different devices and computational capabilities. Thus, a fixed computational task such as gradient computation or local model update could require different processing time, for different clients. Formally, for each client $i \in [M]$, we denote by $\mathcal{T}_i$ the time required to compute a local model update. 
Thus, at each round, when a subset of clients participate in the learning process, the computational time is determined by the slowest participating node. 
Naturally, as the number of nodes in the network grows, we expect more straggling nodes. This phenomenon calls for developing straggler-resilient methods in system-heterogeneous settings.

\section{Straggler-Resilient Personalized FL}\label{SRPFL}

\begin{figure*}[t]
\centering
\begin{minipage}{.5\linewidth}
    \centering
    \includegraphics[width=0.9\linewidth]{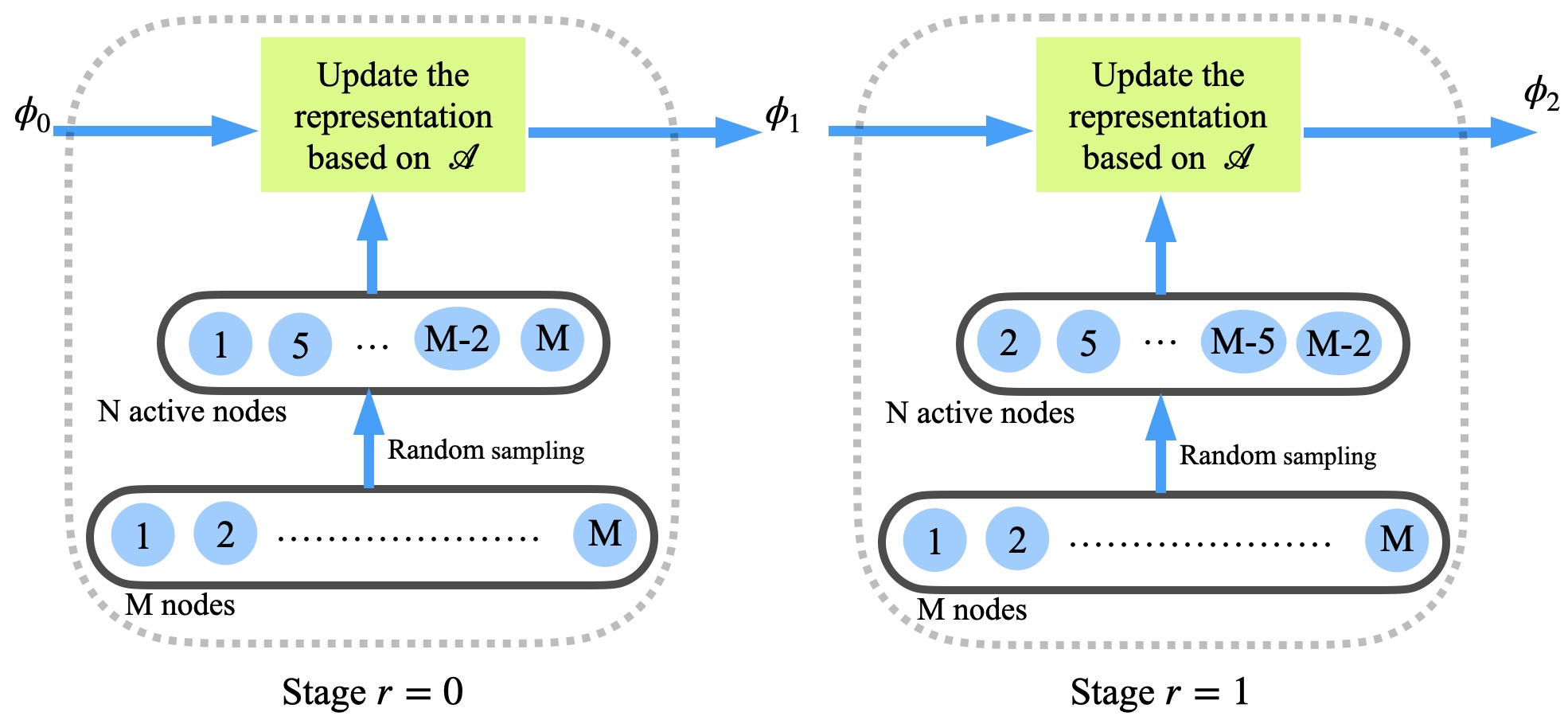}
    \caption{Classic FL schemes for solving  \eqref{eqn_FRL_empirical}}
    \label{Fig_Trad_FL}
\end{minipage}%
\begin{minipage}{.5\linewidth}
    \centering
    \includegraphics[width=0.9\linewidth]{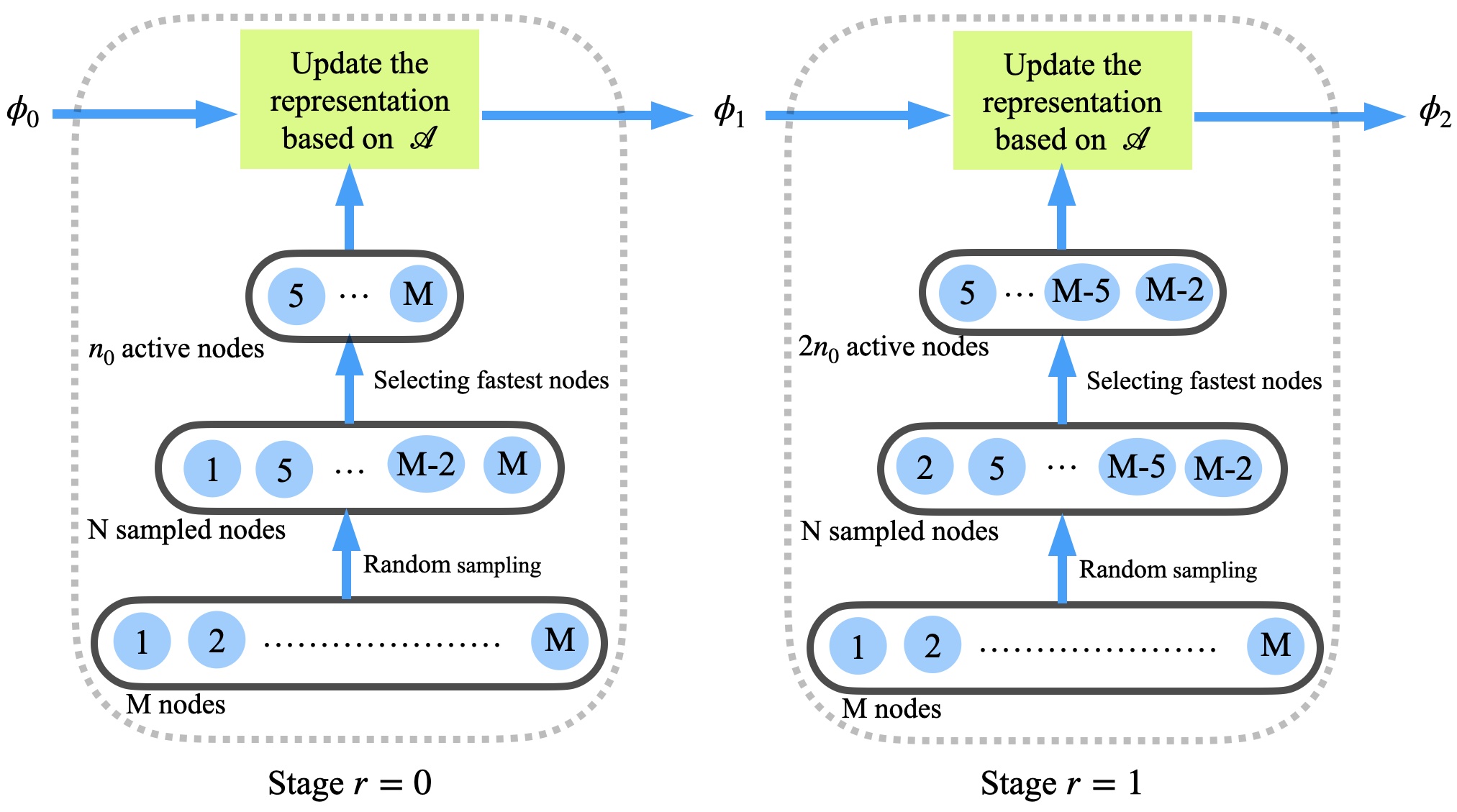}
    \caption{\texttt{SRPFL} for solving  \eqref{eqn_FRL_empirical}}
    \label{Fig_SRPFL}
\end{minipage}
\end{figure*}

In the shared representation setting, we face the challenge of finding an algorithm that coordinates server and clients in order to learn a common representation and a set of personalized parameters in a federated and straggler-resilient fashion. To this end, we propose a method that tackles problem \eqref{eqn_FRL_empirical} with limited sample access and provably superior performance over naive,  straggler-prone methods. Specifically, we propose the Straggler-Resilient Personalized Federated Learning ({\texttt{SRPFL}}), a meta-algorithm designed to mitigate the effects of system heterogeneity in environments with non-convex loss functions and heterogeneous data while accommodating a variety of methods as subroutines. In a nutshell, {\texttt{SRPFL}} iteratively solves problem  \eqref{eqn_FRL_empirical}, while adaptively increasing the set of participating nodes based on their computational capabilities. As a result the process of learning the common representation across clients' tasks is accelerated, without compromising the resulting accuracy. 

To simplify the exposition, henceforth we denote by $\mathcal{A}$ some federated algorithm of choice, designed to solve \eqref{eqn_FRL_empirical}. As depicted in Figure \ref{Fig_Trad_FL}, in standard FL frameworks, out of all $M$ clients in the network, one often selects uniformly at random a subset of size $N$. Subsequently, a few iterations of algorithm $\mathcal{A}$ are performed to approximately solve a version of  \eqref{eqn_FRL_empirical} which corresponds to those $N$ selected clients i.e., $\min_{\phi \in \Phi} \frac{1}{N}\sum_{i=1}^N \min_{h_i\in \mathcal{H}}\hat{f}_i (h_i \circ \phi)$. 
In every following stage, a new subset of $N$ nodes is sampled and the same process is repeated. Although such a procedure eventually learns the global representation across all tasks, it is susceptible to delays caused by stragglers, as the server has to wait for the slowest client (among the $N$ selected ones) to complete its updates. Hence, when $N$ is large the training procedure could become prohibitively slow.

\texttt{SRPFL} follows a different procedure to mitigate the effects of straggling clients. An overview of the selecting scheme is provided in Figure \ref{Fig_SRPFL}. At each stage, $N$ clients are randomly selected, but only a small, fast subset of them is used to solve their corresponding learning problem.  More precisely, suppose that at stage $r$, each client $i$ in the sampled subset of size $N$, is associated with a computational speed $\mathcal{T}_i^r$. For simplicity we assume that the sampled nodes are re-indexed at every stage so that they maintain a decreasing ordering w.r.t. their speeds, i.e. $\mathcal{T}_1^r \leq \mathcal{T}_2^r\leq...\leq\mathcal{T}_N^r$.
Initially, only the $n_0$ fastest nodes, $\{ 1,2,...,n_0\}$, are included in the learning procedure, with $n_0$ much smaller than $N$. 

At every communication round, the set of $n_0$ nodes perform the model updates indicated by algorithm $\mathcal{A}$, to solve their version of \eqref{eqn_FRL_empirical}, i.e., $\min_{\phi \in \Phi} \frac{1}{n_0}\sum_{i=1}^{n_0} \min_{h_i\in \mathcal{H}}\hat{f}_i (h_i \circ \phi)$. At this stage of the algorithm, the server waits only for the slowest client among the participating ones, i.e., node $n_0$. 
Once the current stage terminates, a new batch of $N$ clients is sampled and the $2n_0$ fastest nodes are chosen to participate in the learning process. Since speeds vary between stages, consecutive sets of participating nodes could have small or no overlap. However, the representations learned in previous stages still operate as good starting points for following stages. This is possible since nodes are homogeneous w.r.t. their representations (see \ref{Section_ProbForm}). Thus, utilizing the representation model learned from the previous stage as a warm start, nodes $\{1,2,...,2n_0 \}$ continue the learning process deriving a model with improved accuracy.\ The procedure of geometrically increasing the number of participating nodes continues until the target accuracy is achieved. Hence, \texttt{SRPFL} utilizes the data of stragglers only at the latest stages of the algorithm when an accurate approximation is required.
\begin{remark}
For simplicity of exposition we assume that the set of $N$ sampled nodes maintains connectivity to the server throughout each stage. However, our analysis remains unaffected even if a new set of nodes is sampled at every round.
\end{remark}
\begin{remark}\label{fast_nodes}
In practice, the knowledge of clients' computational power is not required to figure out the $n_0$ fastest nodes. Instead, the server sends the global representation model to all $N$ sampled clients and updates the common representation once the first $n_0$ new models are received. Indeed, these representations belong to the fastest $n_0$ nodes. 
\end{remark}
We proceed to characterize the class of federated algorithms able to employ \texttt{SRPFL} to enhance their performance in system heterogeneous environments. Any iterative method that solves an instance of~\eqref{eqn_FRL_empirical} can be combined with our adaptive node participation scheme, however in this paper, we focus on the class of alternating gradient-based update methods, presented in \cite{collins2021exploiting}. As the name suggests, in each round, clients update their representation models and heads in an alternative manner. After a certain number of gradient updates, all clients send their derived representations to the server where the models are averaged and broadcasted back to the clients. Next, we rigorously state the procedure.

\textbf{Alternating Update Scheme}. At round $t$, the server communicates a common representation $\phi^t$ to the clients and a subset of them $\mathcal{I}^t$, are selected to participate by performing the following updates.

\begin{algorithm}[tb] 
\caption{\texttt{SRPFL}}
\begin{algorithmic}[1] \label{alg:5}
\STATE \textbf{Input:} Initial number of nodes $n_0$; step size $\eta$; number of local updates for head $\tau_h$ and for representation $\tau_\phi$.  
 \STATE \textbf{Initialization:} $n \gets n_0$, $\phi_0,h_1^{0,\tau_h},...,h_N^{0,\tau_h}$
 \FOR{$r=0, 1 ,2, \dots, \log(N/n_0)$}
  \STATE $\phi^0 \gets \phi_r$
 \FOR{$t= 1 ,2, \dots, \tau_{r}$}
  \STATE Server sends representation $\phi^{t-1}$ to $N$ clients sampled from $[M]$.
  \FOR{$i \in \mathcal{I}^r$} 
  \STATE   Client $i$ initializes $h^{t,0}_i \gets h^{t-1,\tau_h}_i$ and runs $\tau_h$ updates $h_i^t = h_i^t - \eta \nabla_{h_i^t} f_i (h_i^{t}, \phi^{t-1})$ 
  \STATE Client $i$ initializes $\phi^{t,0}_i \gets \phi^{t-1}$ and runs $\tau_\phi$ updates $\phi_i^{t}= \phi_i^t - \eta \nabla_{\phi_i^t}f_i(h_i^t, \phi_i^t) $
    \STATE Client $i$ sends $\phi_i^{t, \tau_\phi}$ to the server 
  \ENDFOR
  \FOR{each client $i \notin \mathcal{I}^r$}
  \STATE $h_i^{t,\tau_h} \gets h_i^{t-1,\tau_h} $
  \ENDFOR
  \STATE Server computes $\phi^t  \gets  \frac{1}{n}\sum_{i =1}^n \phi_i^{t,\tau_\phi}$
 \ENDFOR
 \STATE Server sets $n \gets \min\{N, 2n\}$ and $\phi_{r+1}\gets \phi^{\tau_r}$
\ENDFOR
\end{algorithmic}
\end{algorithm}

{\emph{Client Head Update}}.  Each client $i\in \mathcal{I}^t$ performs $\tau_h$ local gradient-based updates optimizing their head parameter, given the received representation $\phi^t$. Concretely, for $s~=~1~,...,\tau_h$ client $i$ updates their head model as follows
\begin{equation}\label{Head_update}
    h_i^{t,s}=\text{GRD}\pr{f_i \pr{h_i^{t,s-1}, \phi^t}, h_i^{t,s-1}, \eta}.
\end{equation}
{\emph{Client Representation Update}}. Once the updated local heads  $h_i^{t,\tau_h}$
are obtained, each client $i$ executes $\tau_\phi$ local updates on their representation parameters. That is for $s=1,..., \tau_\phi$
\begin{equation}\label{Rep_update}
    \phi_i^{t,s} = \text{GRD}\pr{f_i \pr{h_i^{t,\tau_h}, \phi_i^{t,s-1}}, \phi_i^{t,s-1}, \eta}.
\end{equation}
In the above expressions, $\text{GRD}(f,h,\eta)$ captures the generic notion of an update of variable $h$ using the gradient of function $f$ with respect to $h$ and step size $\eta$. This notation allows the inclusion of algorithms such as stochastic gradient descent, gradient descent with momentum, etc.

{\emph{Server Update}}. Each client $i$ sends their derived representation models $\phi_i^{t,\tau_{\phi}}$ to the server, where they are averaged producing the next representation model $\phi^{t+1}$.

Coupling \texttt{SRPFL} with a generic subroutine that falls into the Alternating Update Scheme, gives rise to Algorithm \ref{alg:5}. Every stage $r$, is characterized by a participating set of size $2^r \cdot n_0$, denoted by $\mathcal{I}^r$. At the beginning of each round the server provides a representation model to the participating clients. The clients update their models in lines $8$ and $9$ and sent their representations back to the server where they are averaged in order to be used in the following round. The number of rounds at stage $r$ is denoted by $\tau_r$ and the numbers of local updates $\tau_h$, $\tau_\phi$ depend on the subroutine method of choice.

\begin{remark}
In \cite{reisizadeh2020stragglerresilient} a similar adaptive node participation scheme was proposed, however, their approach differs from ours and their results apply to more restrictive regimes. Specifically, the
analysis of FLANP in \cite{reisizadeh2020stragglerresilient} relies heavily on connecting the ERM solutions between consecutive
stages. As a result (i) clients who participate in early stages
are required to remain active in all future stages and thus,
computational speeds that remain fixed throughout the training process are necessary. Further, (ii) data homogeneity
across all clients is required as well as (iii) strongly convex
loss functions, in order to control the statistical accuracy of
the corresponding ERM problems. The above restrictions
are detrimental in the FL regime and undermine the applicability of the algorithm. In contrast, in \texttt{SRPFL} we follow a different approach controlling an analogue of the statistical accuracy
in terms of principal angle distance, therefore connecting
the common representation (and overall solution) at each
stage with the ground truth representation.\ This novel approach allows our algorithm to accommodate (i) clients with
dynamically changing speeds (or equivalently clients that
are replaced by new ones) between stages or rounds,
(ii) data heterogeneity and (iii) non-convex loss functions.
Additionally, a significant part of our technical contribution
focuses on (iv) deriving the optimal number of rounds
per stage, thus producing a simple and efficient doubling scheme.
In \cite{reisizadeh2020stragglerresilient} the authors suggest instead, a
suboptimal criterion which tracks the norm of the gradient
and requires knowledge of the strong convexity parameter.\ Hence, \texttt{SRPFL}
enjoys crucial benefits eluding prior literature.

\end{remark}

\begin{section}{\texttt{SRPFL} in the Linear Representation Learning Case}\label{Section_linear}

\begin{algorithm}[tb] 
\caption{\texttt{SRPFL} in Linear Representation}
\begin{algorithmic}[1] \label{alg:3}
\STATE \textbf{Input:} Step size $\eta$; Batch size $m$; Initial nodes $n_0$ 
 \STATE \textbf{Initialization:}  Client $i\!\in\![N]$ sends to server:
 $\mathbf{P}_i\! \coloneqq\!\frac{1}{m}\sum_{j=1}^m (y_i^{0,j})^2 \x_i^{0,j} (\x_i^{0,j})^\top \qquad n \gets n_0$
 \STATE Server finds $ \mathbf{U} \mathbf{D}\mathbf{U}^\top\leftarrow \text{rank-}k \text{ SVD}(\tfrac{1}{N}\textstyle{\sum_{i=1}^N\mathbf{P}_i)}$ 
 \FOR{$r=0, 1 ,2, \dots, \log(N/n_0)$}
 \STATE Server initializes $\B^{r, 0} \leftarrow \mathbf{U}$
 \FOR{$t=0, 1 ,2, \dots, \tau_{r}$}
  \STATE Server sends representation $\B^{r, t}$ to $N$ clients sampled from $[M]$. 
  \FOR{$i \in \ag{1,..,n}$} 
  \STATE Client $i$ samples a fresh batch of $m$ samples.
  \STATE Client $i$ finds $ \w_i^{r,t+1} \gets \argmin_{\w} \hat{f}_i^t(\w, \B^{r,t})$
  \STATE  $ \B_i^{r,t+1} \gets  \B^{r,t} - \eta \nabla_{\B}\hat{f}_i^{t}(\w_i^{t+1}, \B^{r,t}) $  is computed by client $i$ and sent to the server.
  \ENDFOR
  \STATE Server computes $\mathbf{\bar{\B}}^{r,t+1}  \gets  \frac{1}{n}\sum_{i \in \mathcal{I}^t} \B^{r,t+1}_i$
  \STATE Server computes
  $ \B^{r,t+1}, \R^{r, t+1} \gets  \text{QR}(\mathbf{\bar{\B}}^{r,t+1})$
 \ENDFOR
 \STATE Server sets $\mathbf{U} \gets \bar \B^{r,t+1} $ and  $n \gets \min\{N, 2n\}$ 
    \ENDFOR
\end{algorithmic}
\end{algorithm}

Our theoretical analysis focuses on a specific instance of \eqref{Personalized_Form1}, where clients strive to solve a linear representation learning problem. Concretely 
we assume that $f_i$ is the quadratic loss, $\phi$ is a projection onto a $k$-dimensional subspace of $\mathbb{R}^d$, given by matrix $\B \in \mathbb{R}^{d \times k}$ and the $i$-th client's head $h_i$, is a vector $\w_i \in \mathbb{R}^k$. We model local data of client $i$ such that $y_i = \w_i^{* \top} \B^{* \top} \x_i + z_i$, for some ground truth representation $\B^{* } \in \mathbb{R}^{d \times k}$, local heads $\w_i^{* } \in \mathbb{R}^k$ and $z_i \sim \mathcal{N}(0,\var^2)$ capturing the noise in the measurements. Hence, all clients' optimal solutions lie in the same $k$-dimensional subspace. 
Under these assumptions the global expected risk is
\begin{align}\label{Expected_Risk}
\min_{\B, \W} \frac{1}{2M} \sum_{i=1}^M  \mathbb{E}_{(\x_i, y_i) \sim \mathcal{D}_i}\left[ \left( y_i- \w_i^\top \B^\top \x_i\right)^2\right],
\end{align}
where $\W=[\w_1^\top,...,\w_N^\top]\in \mathbb{R}^{N \times k}$ is the concatenation of client specific heads.
Since distributions $\mathcal{D}_i$'s are unknown, algorithms strive to minimize the empirical risk instead. Thus, the global empirical risk over all clients is  
\vspace{-1mm}
\begin{equation} \label{Empirical_Loss_General} \frac{1}{M} \sum_{i=1}^M \!\hat{f}_i(\w_i, \B)\! =\! \frac{1}{2Mm} \sum_{i=1}^M  \sum_{j=1}^{m} \left( y_i^{j}\!-\! \w_i^{t \top} \B^{ \top} \x_i^{j}\right)^2\!\!\!,
\end{equation}
where $m$ is the number of samples at each client.
The global loss in \eqref{Empirical_Loss_General} is nonconvex and has many global minima, including all pairs of $\pr{\W^* \Q^{-1}, \B^* \Q^\top}$, where $\Q \in \mathbb{R}^{k\times k}$ is invertible. Thus, server aims to retrieve the column space of $\B^*$, instead of the ground truth factors $\pr{\W^*, \B^*}$. To measure closeness between column spaces, we adopt the metric of principal angle distance  \cite{jain2012lowrank}. 

\begin{definition}
 The principle angle distance between the column spaces of $~ \B_1, \B_2 \in \mathbb{R}^{d \times k}$ is 
 $\dist(\B_1, \B_2) := \|\hat{\B}^\top _{1,\bot}, \hat{\B}_2\|_2$,
 where $\hat{\B}_{1,\bot}$ and $\hat{\B}_2$ are orthonormal matrices s.t. $\spn(\hat{\B}_{1,\bot})= \spn(\B_1)^\bot$ and $\spn(\hat{\B}_2)= \spn(\B_2)$.
\end{definition}

Federated Representation Learning (\texttt{FedRep}) is an alternating minimization-descent, federated algorithm, recently proposed in \cite{collins2021exploiting} for the Linear Shared Representation framework. \texttt{SRPFL} coupled with \texttt{FedRep} gives rise to \cref{alg:3}. Below we summarize the points of interest.

In the initialization phase the method of moments is invoked obtaining an initial model, sufficiently close to the optimal representation. At every round $t$, client $i$ samples a fresh batch of samples $\{\x_i^{t,j}, y_i^{t,j} \}_{j=1}^m$ from its local distribution $\mathcal{D}_i$ (line $9$). Thus the local loss is given by
$\hat{f}_i(\w_i \circ \B) := \frac{1}{2m}  \sum_{j=1}^{m} ( y_i^{t,j}- \w_i^{ \top} \B^{ \top} \x_i^{t,j})^2$.
 Utilizing as a warm start the obtained representation from previous rounds, client $i$ derives the optimal head $\w_i$ (via a large number of gradient updates), as can be observed in line $10$. Fixing the newly computed $\w_i$, client $i$ proceeds to update its global representation model with one step of gradient descent (line $11$) and consequently sends the new representation to the server. As depicted in lines $13$ and $14$ the parameter server averages the models received and orthogonalizes the resulting matrix. Furthermore, we observe that the number of representation model updates, $\tau_\phi$, is set to $1$ whereas the number of head updates, $\tau_h$, is a number sufficiently large for deriving the optimal heads. This imbalance takes advantage of the fact that local updates w.r.t. $\w_i$ are inexpensive in our regime. 
 The number of communication rounds per stage, $\tau_r$ is a small, a priori known, constant specified by our analysis.  


\end{section}
\subsection{Theoretical Results}\label{Theoretical_Results}

In this subsection, we provide rigorous analysis of \texttt{SRPFL} in the linear representation case. First we present the notion of `Wall Clock Time' which is the measure of performance for our algorithm. Subsequently, we highlight the contraction which determines the rate at which the distance to the optimal representation diminishes. Finally, we conclude that \texttt{SRPFL} outperforms its straggler-prone variant by a factor of $\log(N)$, under the standard assumption that clients' speeds are drawn from an exponential distribution.
Before we proceed, we present useful notation and assumptions.
\begin{align}
    &E_0 := 1-\dist^2\pr{\B^0, \B^*},\\
   \ssigmax := \hspace{-0.25in} \max_{\mathcal{I}\in [N], |\mathcal{I}|=n, n_0\leq n\leq N} \hspace{-0.1in} \sigma_{\max} &\pr{\frac{1}{\sqrt{n}}\W^*_{\mathcal{I}}}, \quad \hspace{-0.05in} \textit{and} \quad \hspace{-0.05in}\ssigmin:= \hspace{-0.25in} \min_{\mathcal{I}\in [N], |\mathcal{I}|=n, n_0\leq n\leq N}\hspace{-0.1in} \sigma_{\min}\pr{\frac{1}{\sqrt{n}}\W^*_{\mathcal{I}}}\label{sigm} 
\end{align}

\begin{assumption}\label{Assum1}
(Sub-gaussian design). The samples $\x_i \in \mathbb{R}^d$ are i.i.d. with mean $0$, covariance $\mathbf{I}_d$ and are $\mathbf{I}_d$-sub-gaussian, i.e. $\mathbb{E}[e^{\mathbf{v}^\top \x_i}]\leq e^ {\norm{\mathbf{v}}^2_2/2}$ for all $\mathbf{v} \in \mathbb{R}^d$.
\end{assumption}
\begin{assumption}\label{Assum2}
(Client diversity). Let $\ssigmin$ as defined in \eqref{sigm} i.e. $\ssigmin$ is the minimum singular value of any matrix that can be obtained by taking $n$ rows of $\frac{1}{\sqrt{n}}\W^*$. Then $\ssigmin > 0$.
\end{assumption}
\vspace{-0.1in}
Assumption \ref{Assum2} implies that the optimal heads of participating clients span $\mathbb{R}^k$. This holds in most FL regimes as the number of clients is larger than the dimension of the shared representation. 
\begin{assumption}\label{Assum3}
(Client normalization). The ground-truth client specific parameters satisfy $\norm{\w_i^*}_2 = \sqrt{k}$ for all $i \in [n]$ and $\B^*$ has orthonormal columns. 
\end{assumption}
\cref{Assum3} ensures the ground-truth matrix $\W^*\B^{* \top}$ is row-wise \emph{incoherent}, i.e. its row norms have similar magnitudes. This is of vital importance since our measurement matrices are row-wise sparse and incoherence is a common requirement in  sensing problems with sparse measurements. 

\noindent\textbf{Wall Clock Time.}
To measure the speedup that our meta-algorithm enjoys we use the concept of `Wall Clock Time' described below. \texttt{SRPFL} runs in communication rounds grouped in stages. In each round, both the representation and the local heads of the clients get updated. Consider such a round $t$ in stage $r$, with nodes $\{1,2,...,n_r \}$ participating in the learning process. Here $n_r$ denotes the slowest participating node. The expected amount of time that the server has to wait for the updates to take place is $\Ex{\mathcal{T}^r_{n_r}}$. Put simply, the expected computational speed of the slowest node acts as the bottleneck for the round. Further, at the beginning and end of every round, models are exchanged between the server and the clients. This incurs an additional communication cost $\mathcal{C}$. If $\tau_{r}$ communication rounds take place at stage $r$, then the overall expected `Wall Clock Time' for \texttt{SRPFL} is $
    \Ex{T_{\textit{SRPFL}}}=\sum_{r=0}^ {\log(N/n_0)}\tau_{r}\pr{\Ex{\mathcal{T}^r_{n_r}}+\mathcal{C}}$.
Similarly, the total expected runtime for \texttt{FedRep} can be expressed in terms of the total number of rounds, $T_{FR}$, as
$\Ex{T_{\textit{FedRep}}} = T_{FR} \pr{\Ex{\mathcal{T}^r_{N}}+\mathcal{C}}$.

\noindent\textbf{Contraction Inequality.}
The  next theorem captures the progress made between two consecutive rounds of \texttt{SRPFL}. It follows that the rate of convergence to the optimal representation is exponentially fast, provided that the number of participating nodes and the batch size are sufficiently large. 
\begin{restatable}{theorem}{FirstRestatable}\label{convergence} Let Assumptions~\ref{Assum1}-\ref{Assum3} hold. Further, let the following inequalities hold for the number of participating nodes and the batch size respectively,  $n\geq n_0$ and $m ~\geq~c_0 \frac{(1+ \var^2) k^3 \kappa^4 }{E_0^2} \max \ag{\log(N), d/n_0} $, for some absolute constant $c_0$. Then \texttt{SRPFL} in the linear representation case, with stepsize $\eta\leq \frac{1}{8\ssigmax^2}$, satisfies the following contraction inequality
\begin{align}\label{contraction_inequality}
    \dist\pr{\B^{t+1},\B^*} &\leq \dist\pr{\B^{t},\B^*}\sqrt{1-a}  +\frac{a}{\sqrt{\frac{n}{n_0}\pr{1-a}}},
\end{align}
with probability $1-T\exp\pr{-90 \min \ag{d, k^2 \log(N)}}$ where $a=\frac{1}{2}\eta E_0 \ssigmin^2 \leq \frac{1}{4}$.
\end{restatable}
Here $T$ denotes the total number of communication rounds which is logarithmic w.r.t. the target accuracy~$\epsilon$. Similarly to \cite{collins2021exploiting}, 
the initial representation is computed by the Method of Moments and it holds that $\dist\pr{\B^{0}, \B^*}\leq 1-C_M$, for some constant $C_M$. It follows that $E_0$ is strictly greater than zero and hence, inequality \eqref{contraction_inequality} ensures contraction.


\noindent\textbf{Logarithmic Speedup.}
\cref{alg:3} starts with $n_0$ participating clients and follows a doubling scheme such that at stage $r$ only the fastest $2^rn_0$ nodes contribute to the learning process.  
Thus in stage $r$, inequality \eqref{contraction_inequality} can be written as :
\vspace{-1mm}
\begin{equation}
   \dist^+ \leq \dist \cdot \sqrt{1-\alpha}+ \frac{\alpha}{\sqrt{2^r\pr{1-\alpha}}} \label{con}
\end{equation}
The second term on the r.h.s. is an artifact of the noisy measurements. Specifically, using geometric series we can deduce that in the limit, contraction \eqref{con} converges to $\frac{\alpha}{\sqrt{2^r\pr{1-\alpha}}\pr{1-\sqrt{1-\alpha}}}.$ This implies that the achievable accuracy of our algorithm is lower bounded by $\frac{\alpha}{\sqrt{\frac{N}{n_0}\pr{1-\alpha}}\pr{1-\sqrt{1-\alpha}}}$, since the total number of stages is at most $\log\pr{N/n_0}$. To highlight the theoretical benefits of \texttt{SRPFL} we compare \cref{alg:3} to \texttt{FedRep} with full participation. \texttt{FedRep} can be distilled from \cref{alg:3} if one disregards the doubling scheme and instead $N$ randomly sampled nodes are chosen to participate at each round.
For fair comparison between the two methods the target accuracy should require the contribution of all $N$ nodes to be achieved i.e. it should be bounded above by $\frac{\sqrt{2}\alpha}{\sqrt{\frac{N}{n_0}\pr{1-\alpha}}\pr{1-\sqrt{1-\alpha}}}$. Taking the aforementioned points into consideration we express the target accuracy $\epsilon$ as follows
\vspace{-2mm}
\begin{align}\label{accuracy}
\vspace{-2mm}
    \epsilon = \hat{c} \frac{\alpha}{\sqrt{\frac{N}{n_0}\pr{1-\alpha}}\pr{1-\sqrt{1-\alpha}}}, \qquad \textit{with} \qquad \sqrt{2}>\hat{c}>1.
\end{align}
Intuitively, we expect \texttt{SRPFL} to vastly outperform straggler-prone \texttt{FedRep} with full participation as $\hat{c}$ approaches $\sqrt{2}$, since in this case the biggest chunk of the workload is completed before \texttt{SRPFL} utilizes the slower half of the clients. In contrast \texttt{FedRep} experiences heavy delays throughout the whole training process due to the inclusion of stragglers in all rounds. On the flip side, as $\hat{c}$ approaches $1$, the amount of rounds spent by \texttt{SRPFL} utilizing $N$ clients increases dramatically. In this case one should expect the speedup achieved in early stages by \texttt{SRPFL}, eventually to become obsolete. Theorem \ref{SpeedupTh} provides a rigorous exposition of this intuition.
\begin{restatable}{theorem}{SecondRestatable}\label{SpeedupTh}
Suppose at each stage the client's computational times are i.i.d. random variables drawn from the exponential distribution with parameter $\lambda$. Further, suppose that the expected communication cost per round is $\mathcal{C}=c\frac{1}{\lambda}$, for some constant $c$. Finally, consider target accuracy $\epsilon$ given in \eqref{accuracy}. Then, we have
$\frac{\Ex{T_{\textit{SRPFL}}}}{\Ex{T_{\textit{FedRep}}}} = \OO\pr{\frac{\log \pr{\frac{1}{\hat{c}-1}}}{\log(N)+\log\pr{\frac{1}{\hat{c}-1}}}}$.
\end{restatable}


\vspace{-2mm}
\section{Experiments}\label{exper}
\vspace{-2mm}

In our empirical study we consider the image classification problem in the following four datasets, CIFAR10, CIFAR100, EMNIST and FEMNIST. We conduct various experiments in partial and full node participation settings with different neural network architectures comparing the performance of our proposed method to other state of the art benchmarks. Further experimental results and extensive discussion can be found in \cref{appendix_section_experiment}. A summary of our findings is presented below.

\noindent{\bf Baselines}.
We include \texttt{FedRep} \cite{collins2021exploiting} and \texttt{LG-FedAvg} \cite{liang2020think}, FL algorithms that utilize a mixture of global and local models to derive personalized solutions with small global loss. The schemes that arise when those methods are coupled with our doubling scheme (denoted by \texttt{SRPFL} and \texttt{LG-SRPFL}, respectively) are also studied. \texttt{FLANP} \cite{reisizadeh2020stragglerresilient} and \texttt{FedAvg} \cite{mcmahan2017communication} present interesting benchmarks for comparison, however, since their efficiency deteriorates in data-heterogeneous settings, the performance of their fine-tuning variants \cite{wang2019federated,yu2020salvaging}, is also explored  in our experiments. 
Following a different approach \texttt{HF-MAML} \cite{fallah2020convergence} originates from the Model Agnostic Meta Learning regime and produces high quality personalized solutions in data-heterogeneous settings. 

\noindent{\bf Data allocation.}
We aim at generating data heterogeneity in a controlled manner. We randomly partition the training data points equally to $M$ clients and ensure that each client observes the same number of local classes (taking values Shard $=3,5,20$, \cref{fig:experiment_main}). Further we adjust the testing data points for each client to match the corresponding local training set. In FEMNIST dataset we follow a similar data allocation scheme to \cite{collins2021exploiting}, subsampling $15,000$ data points out of $10$ classes. Further details and different data allocations are presented in \cref{appendix_section_experiment}. 

\noindent{\bf Simulation of System Heterogeneity.}\
We consider two types of client speed configurations. \\
{\it Fixed computation speeds.} In this configuration we sample once, a value for each client from the exponential distribution with parameter $1$. The speed of the client is proportional to the reciprocal of this value and is fixed throughout the training procedure. In addition to the computational time, another critical factor in the overall running time is the communication time. Our methods suffer a fixed communication cost at each round. In \cref{fig:experiment_main} each row depicts the effects of different values of communication cost on the convergence of the algorithms under consideration (C.T. $= 0, 10, 100$).\\
    {\it Dynamic computation speeds.}
    At each round we sample a value for each client $i$ from the exponential distribution with parameter $\lambda_i$. Note that $\lambda_i$ is sampled a priori from a uniform distribution over $[1/N, 1]$ and an additional communication cost is included in the simulation of the overall running time. The experimental results remain qualitatively unchanged and are deferred to \cref{appendix_section_experiment}.
 
\noindent{\textbf{Results and Discussions.}} 
 Our experimental results in \cref{fig:experiment_main} highlight three points of significance in regimes with data and device heterogeneity: 1) Applying \texttt{SRPFL} to personalized FL solvers significantly enhance their efficiency and robustness. 2)  \texttt{SRPFL} coupled with \texttt{FedRep} exhibits superior performance compared to many state-of-the-art FL methods. 3) Our suggested method vastly outperforms the fine-tuning variant of the previously proposed \texttt{FLANP}, especially in regimes with high data heterogeneity. Furthermore, as one would expect high communication cost hurts the adaptive participation schemes disproportionately, mitigating their benefits in terms of running time.

\begin{figure*}[t]
    \centering
    \begin{tabular}{c|c| c| c }
        \includegraphics[width=.22\textwidth]{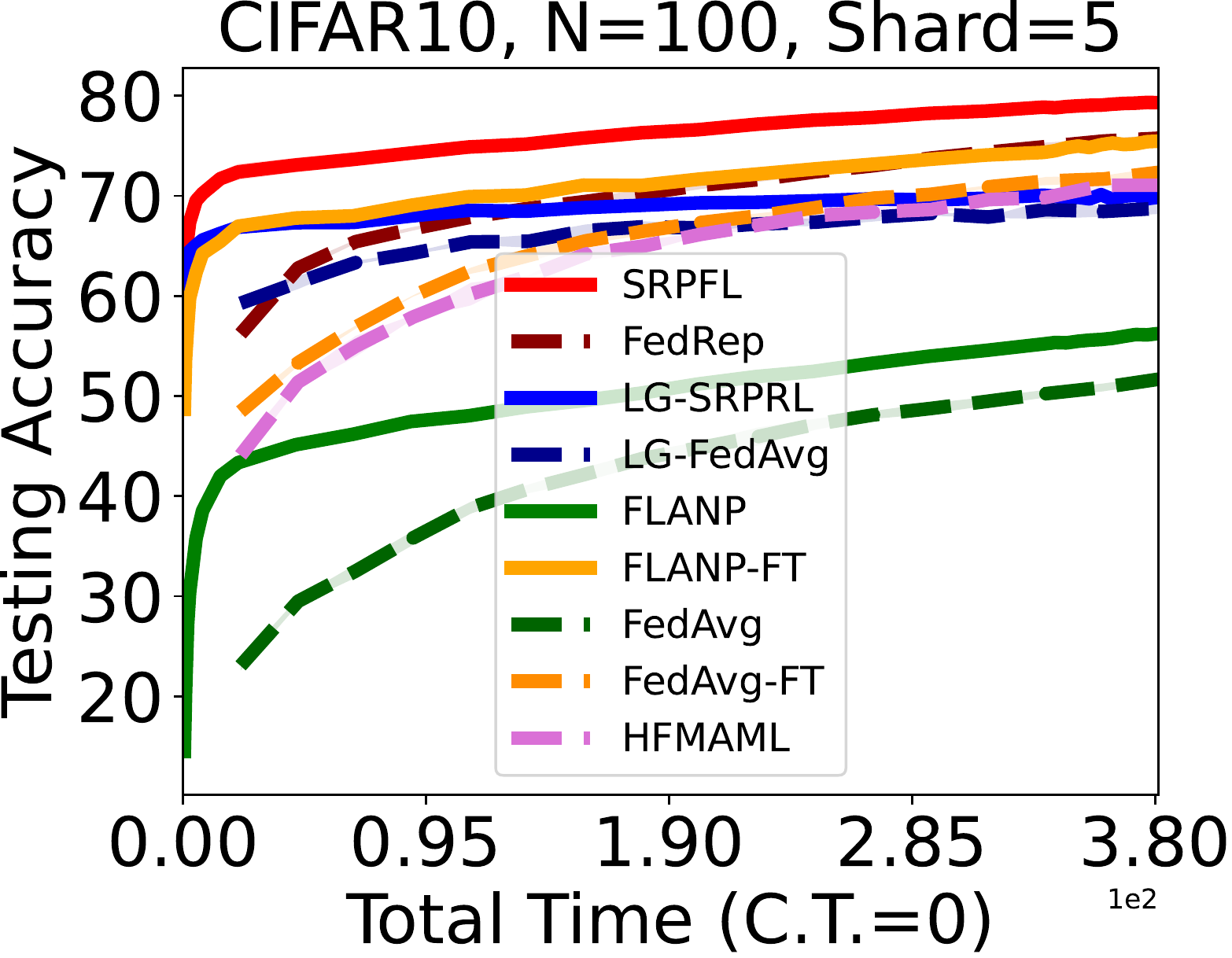} &
        \includegraphics[width=.22\textwidth]{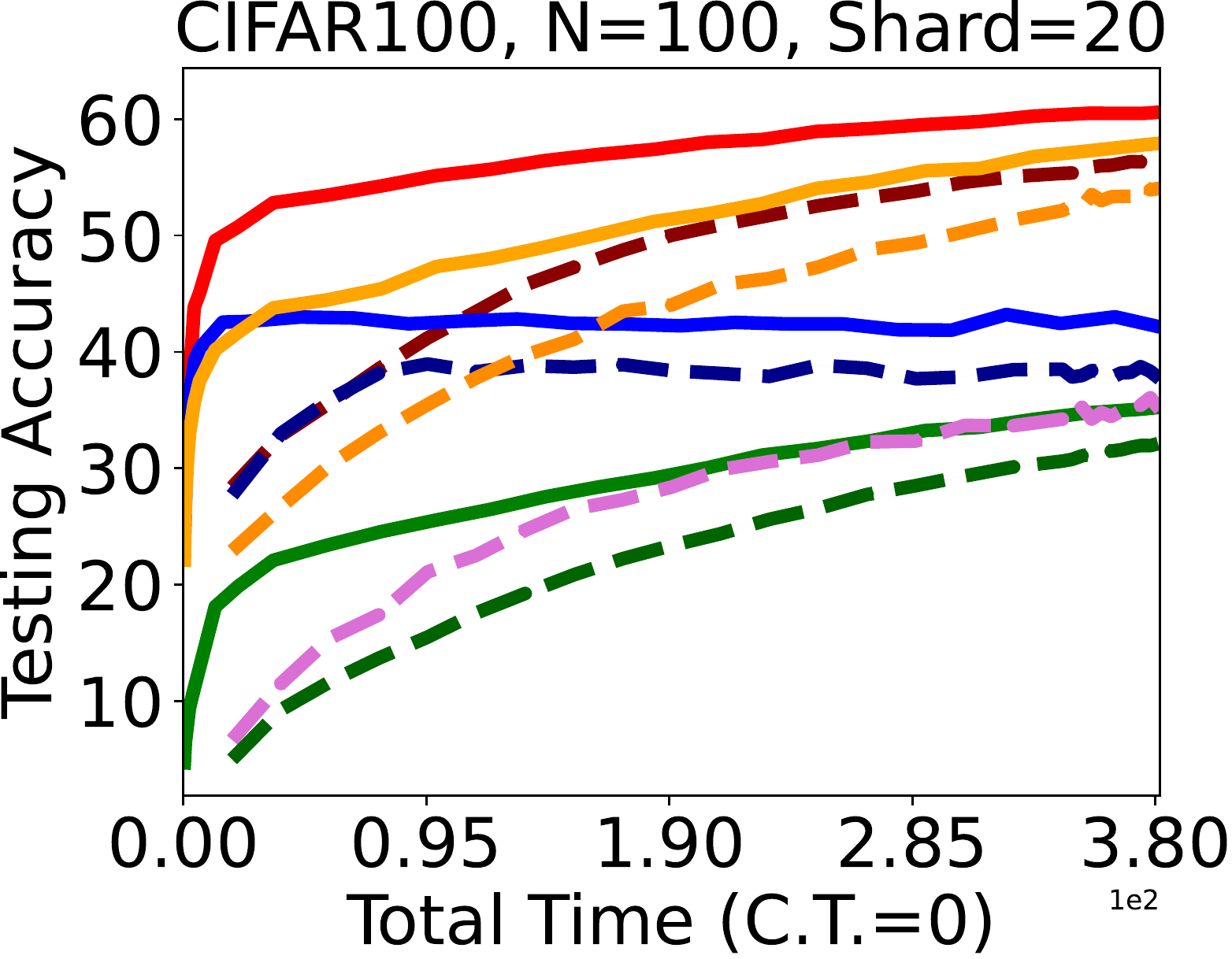} &
        \includegraphics[width=.22\textwidth]{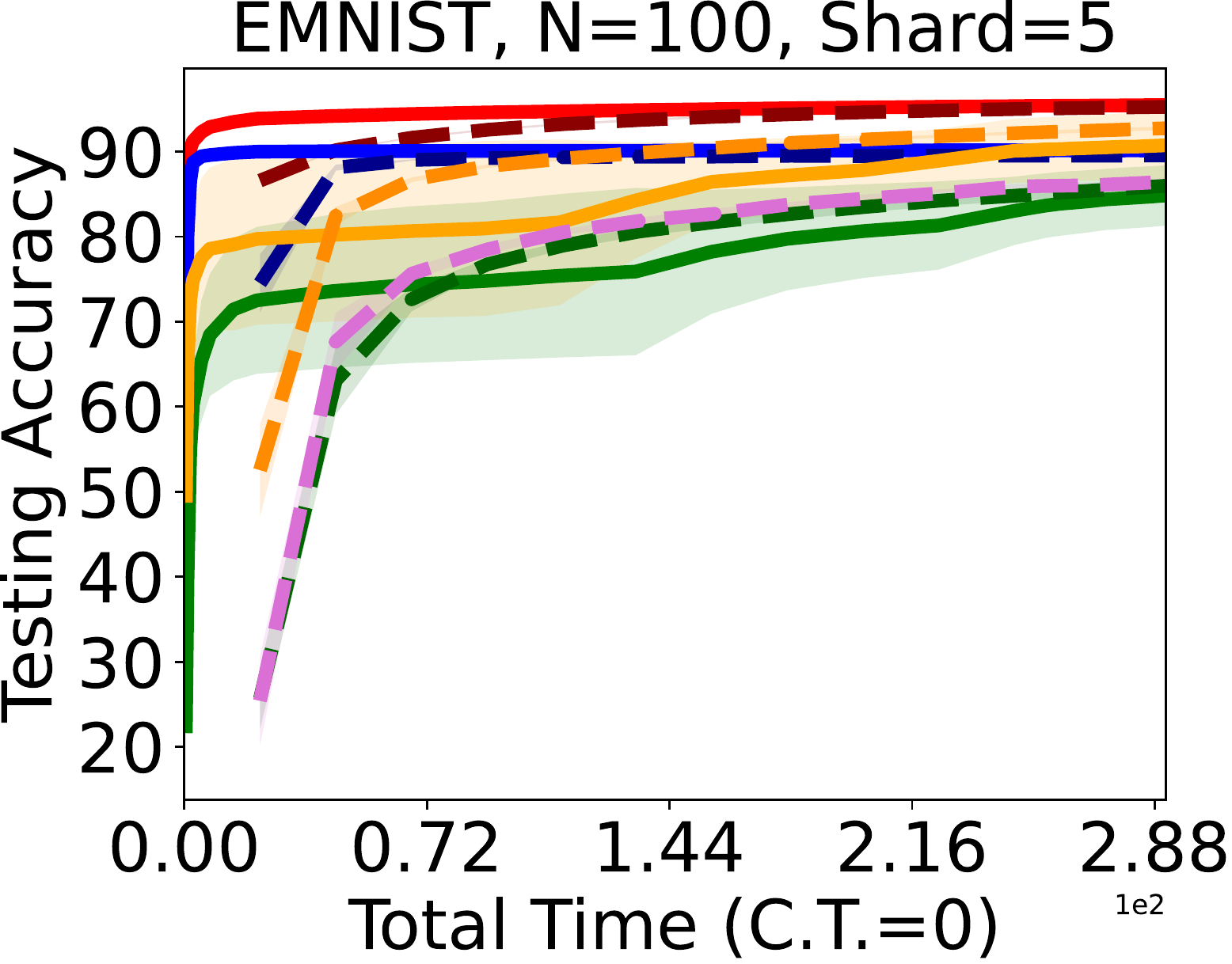} &
        \includegraphics[width=.22\textwidth]{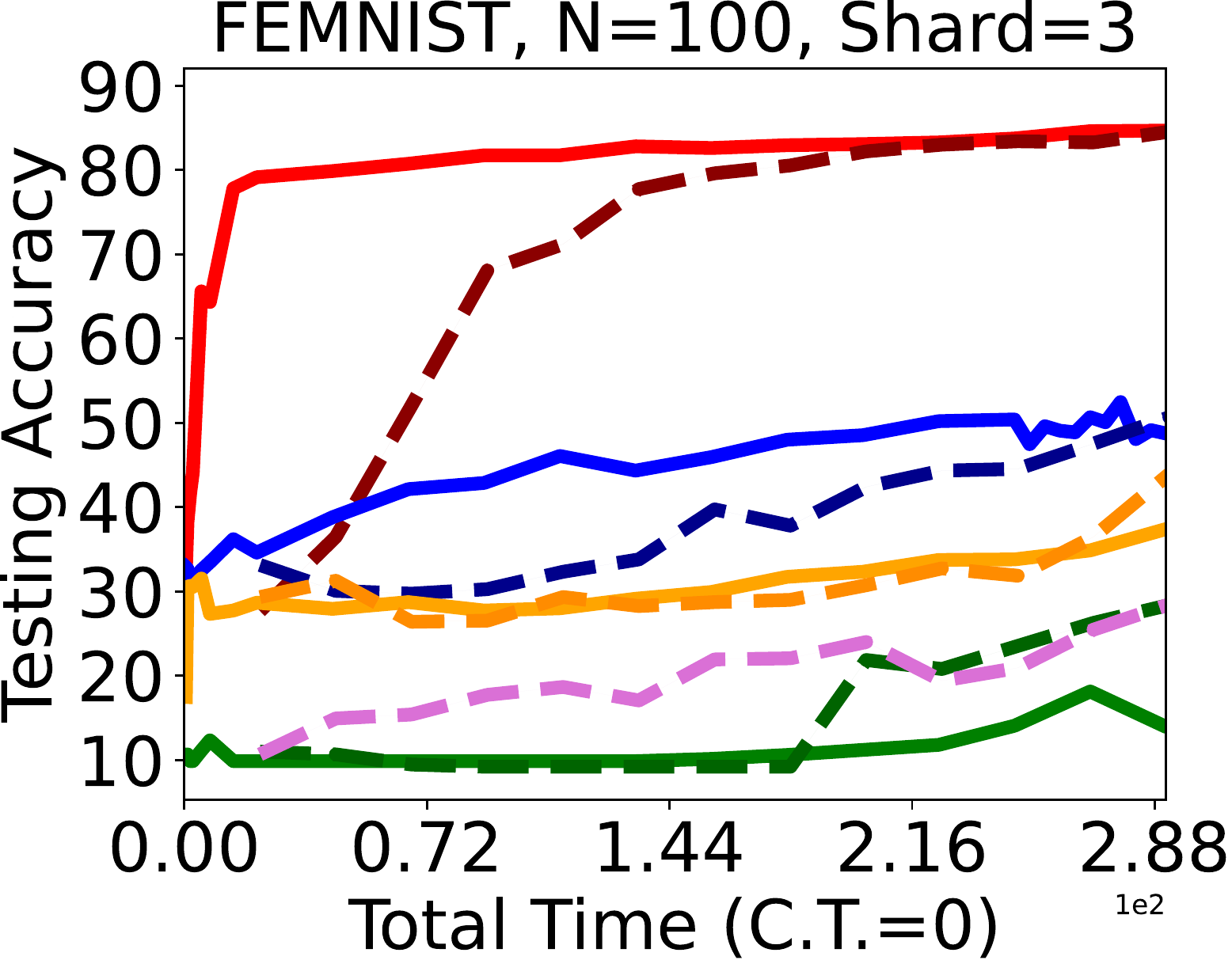} \\
        \includegraphics[width=.22\textwidth]{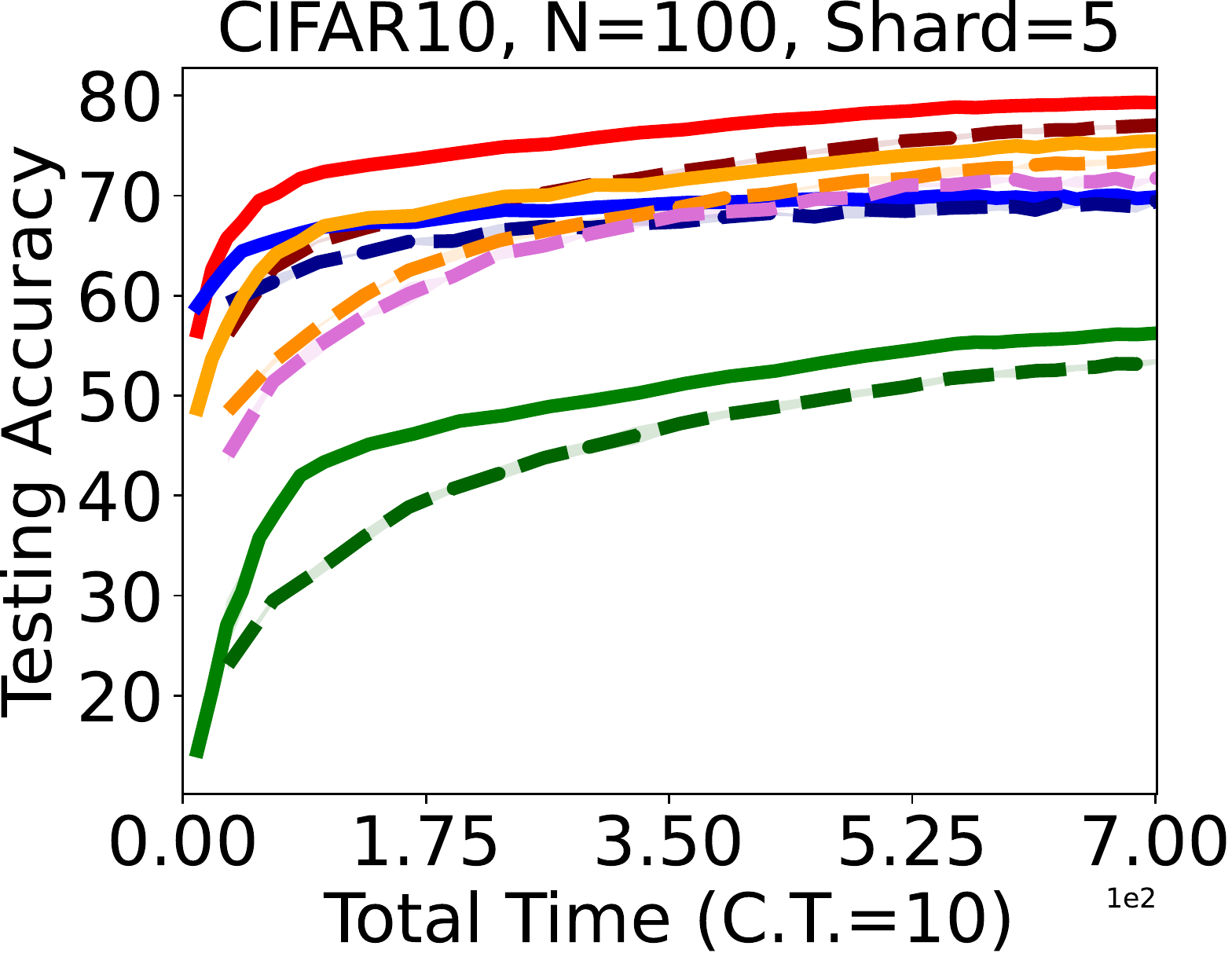}&
        \includegraphics[width=.22\textwidth]{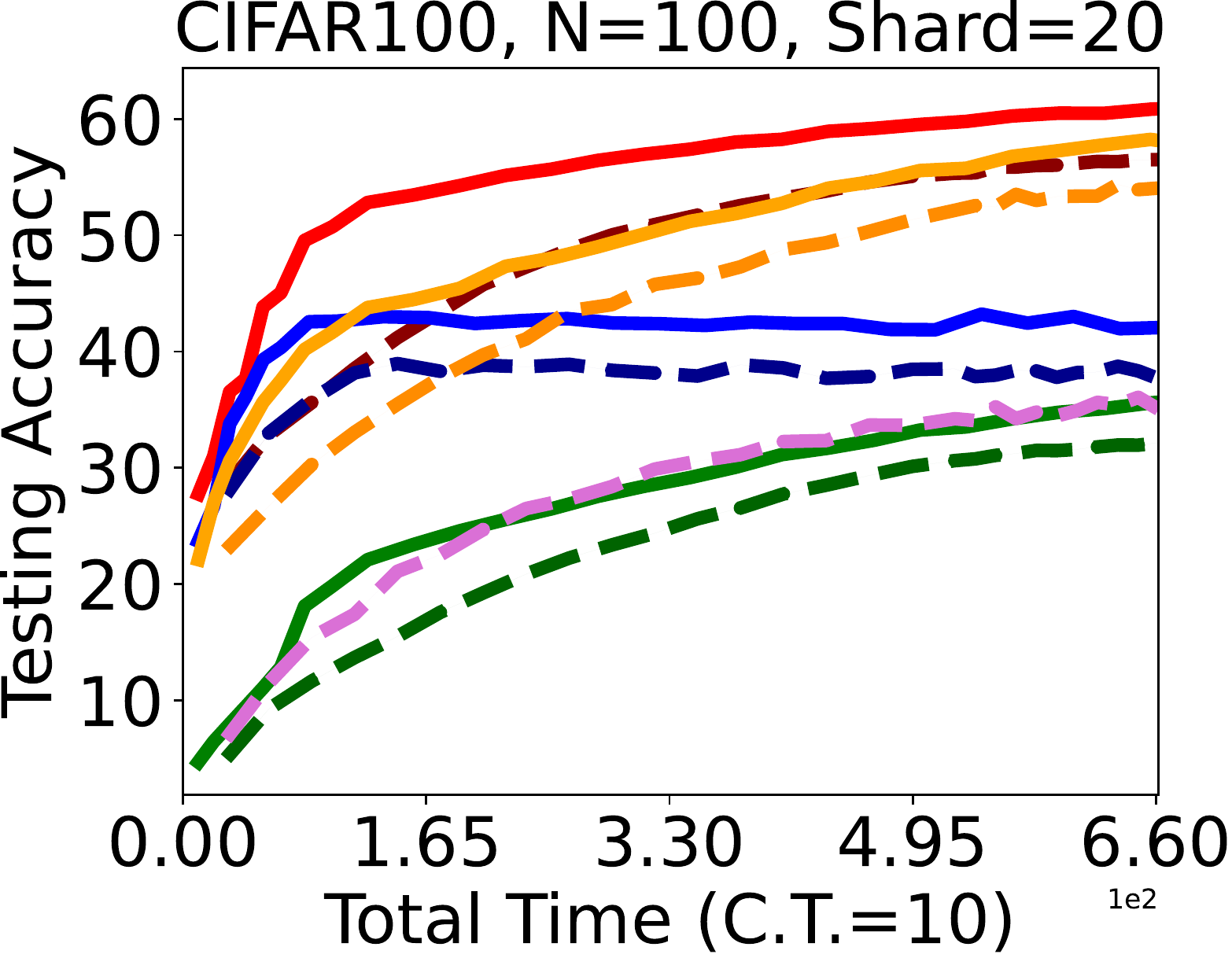} &
        \includegraphics[width=.22\textwidth]{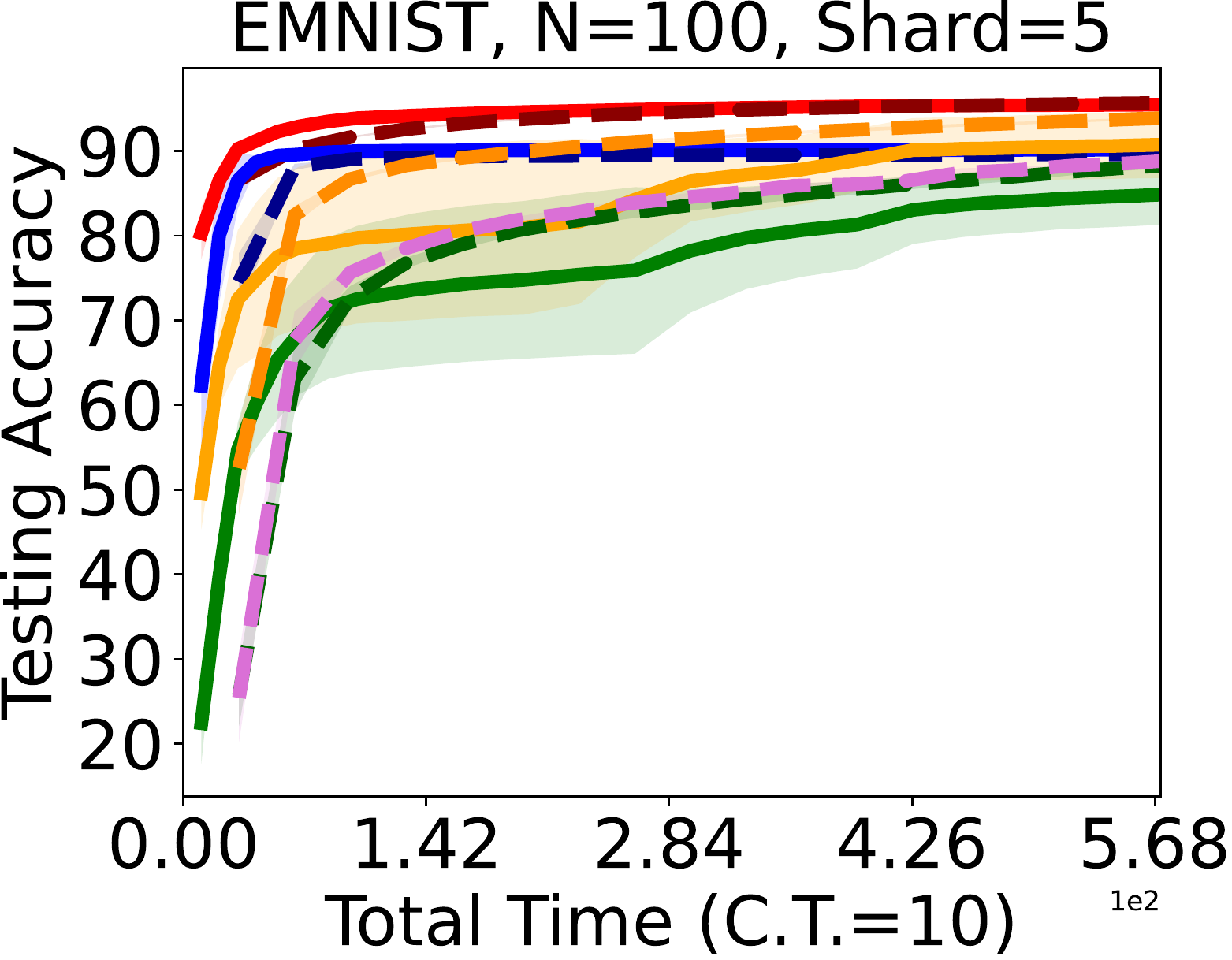} &
        \includegraphics[width=.22\textwidth]{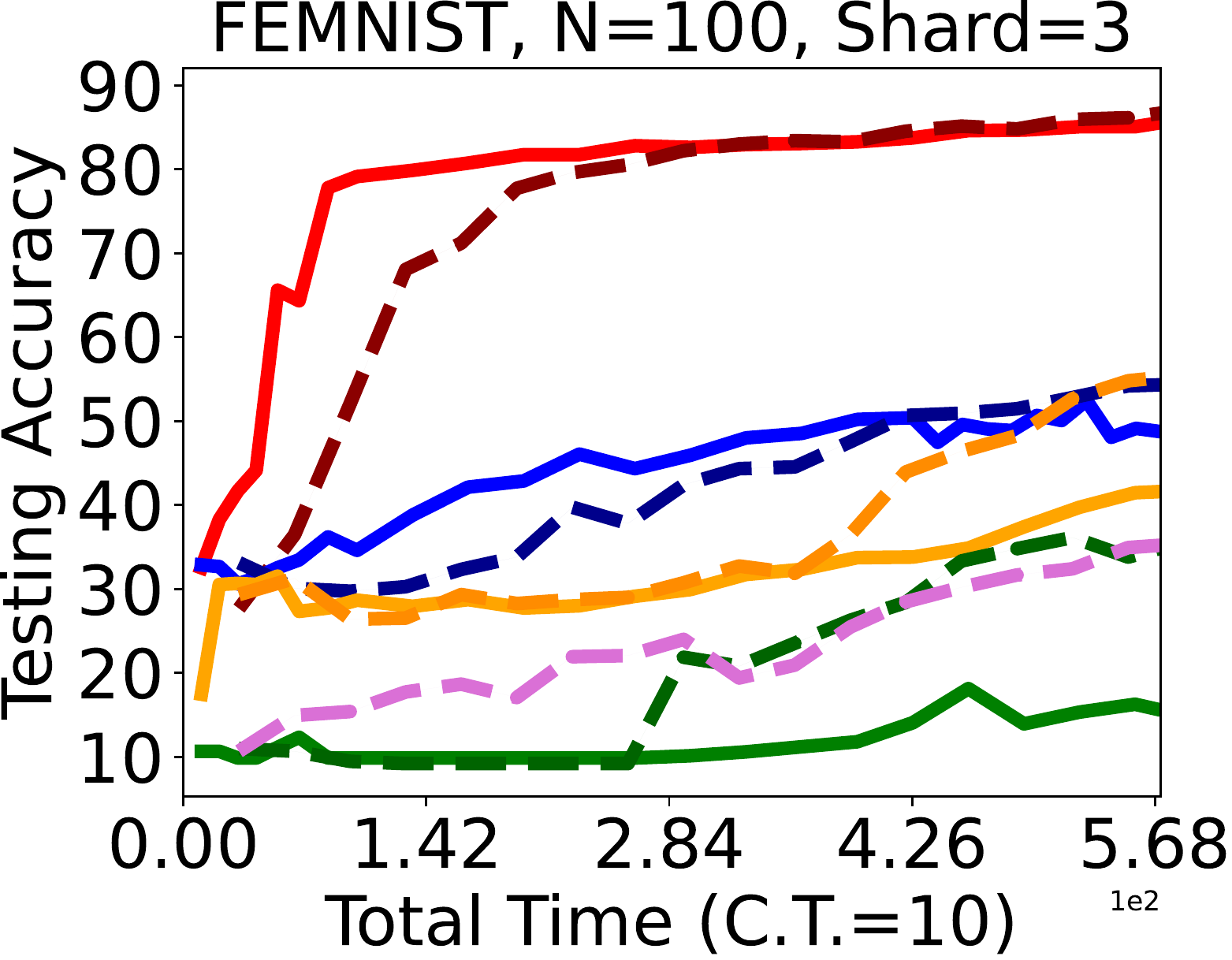} \\
        \includegraphics[width=.22\textwidth]{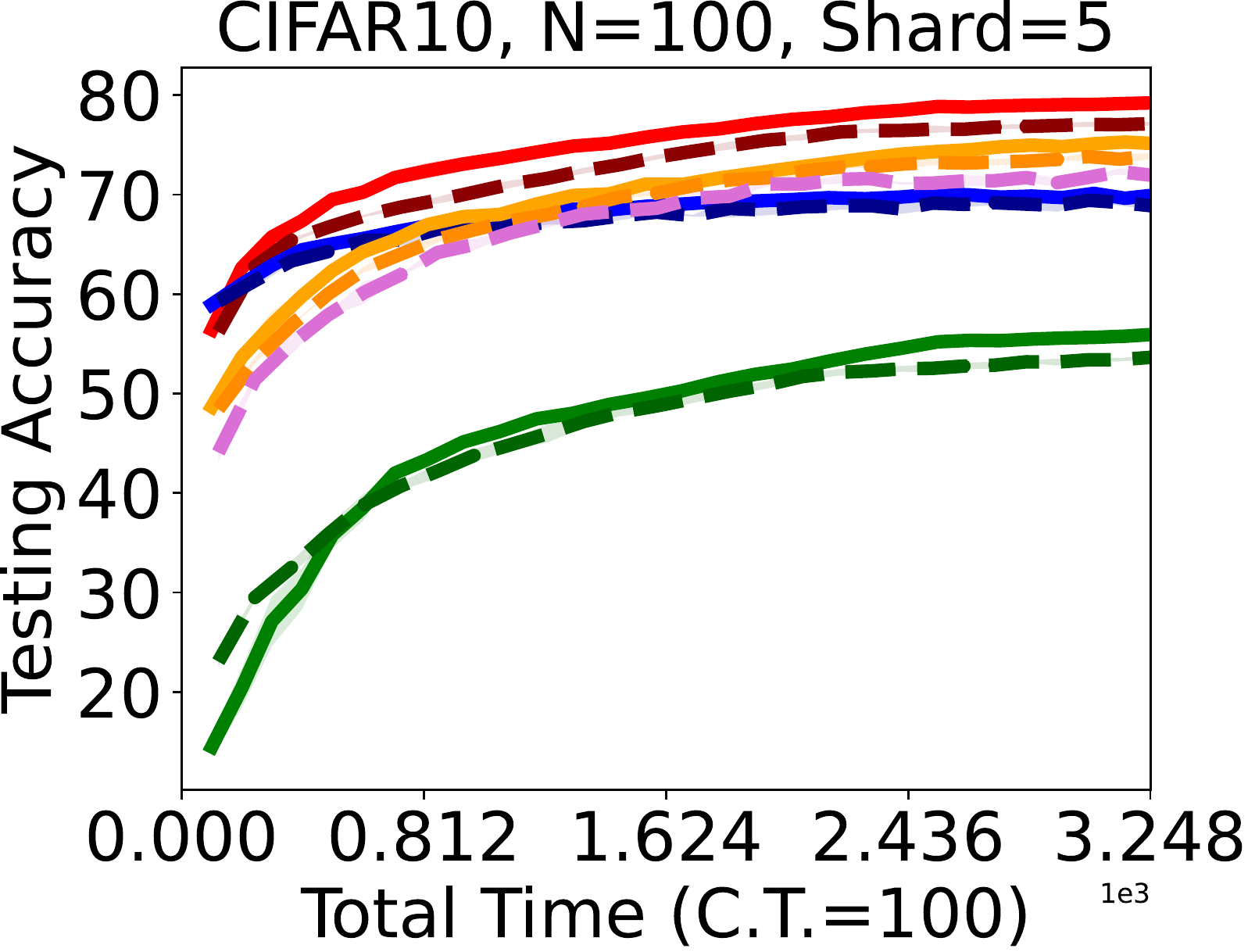} &
        \includegraphics[width=.22\textwidth]{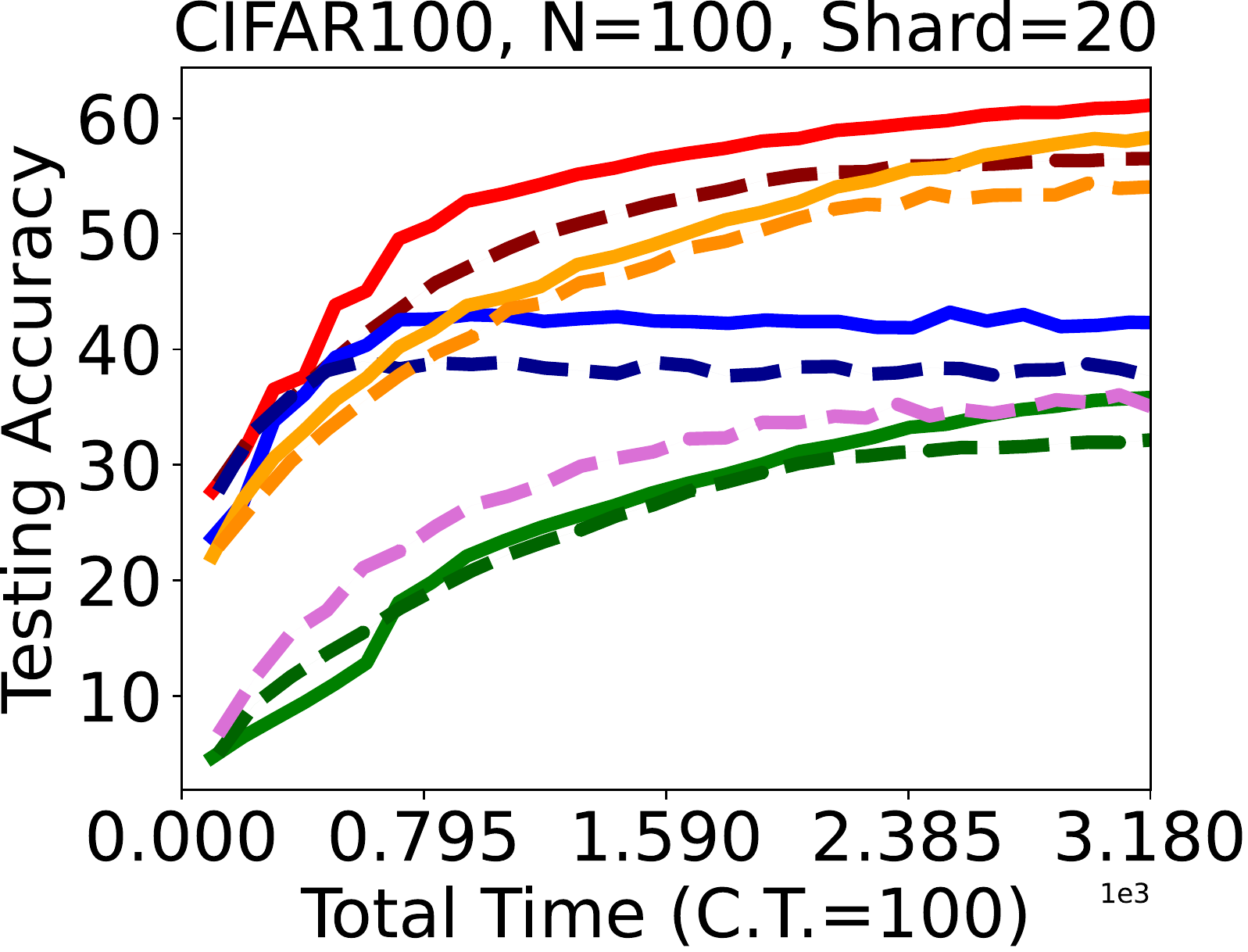} &
        \includegraphics[width=.22\textwidth]{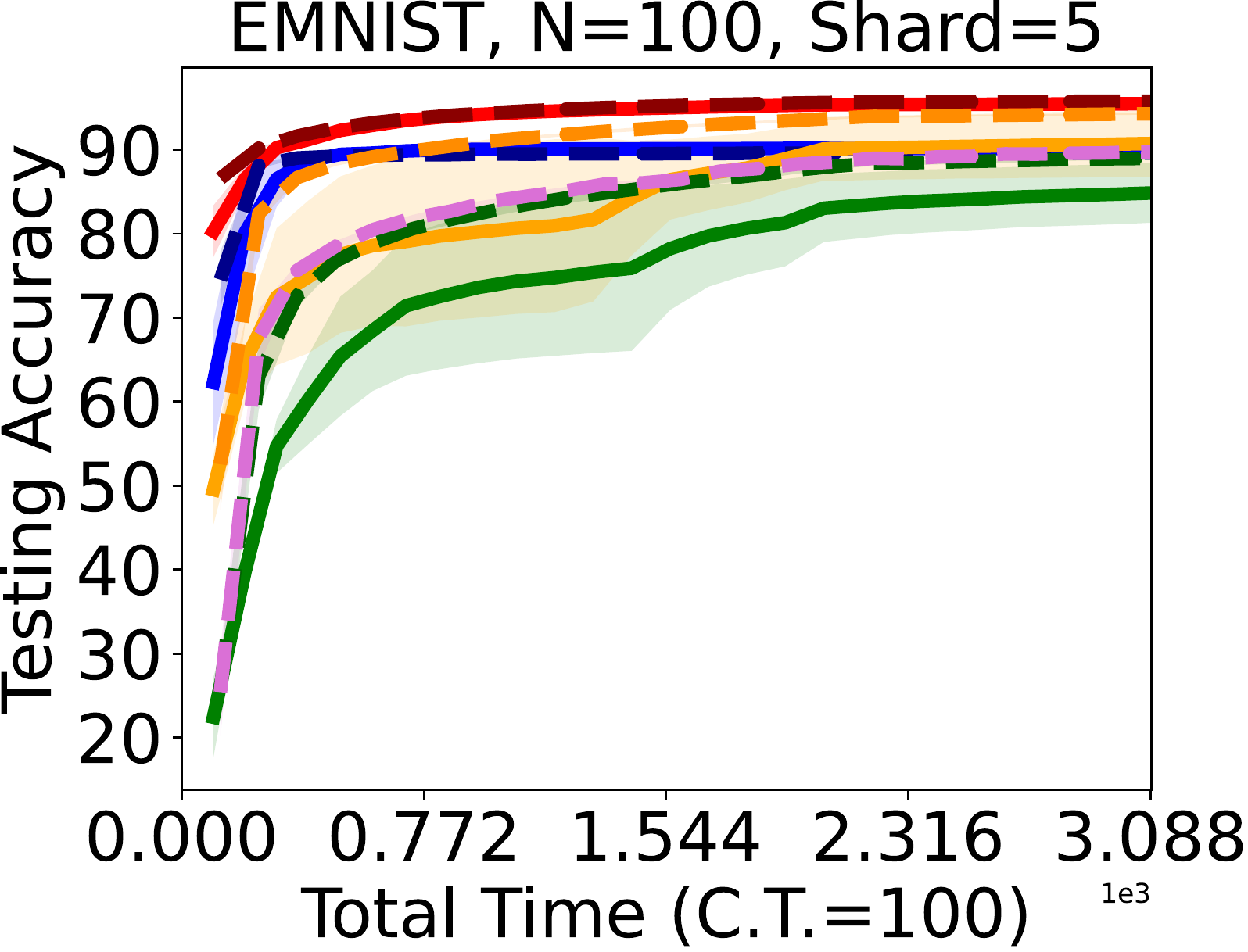} &
        \includegraphics[width=.22\textwidth]{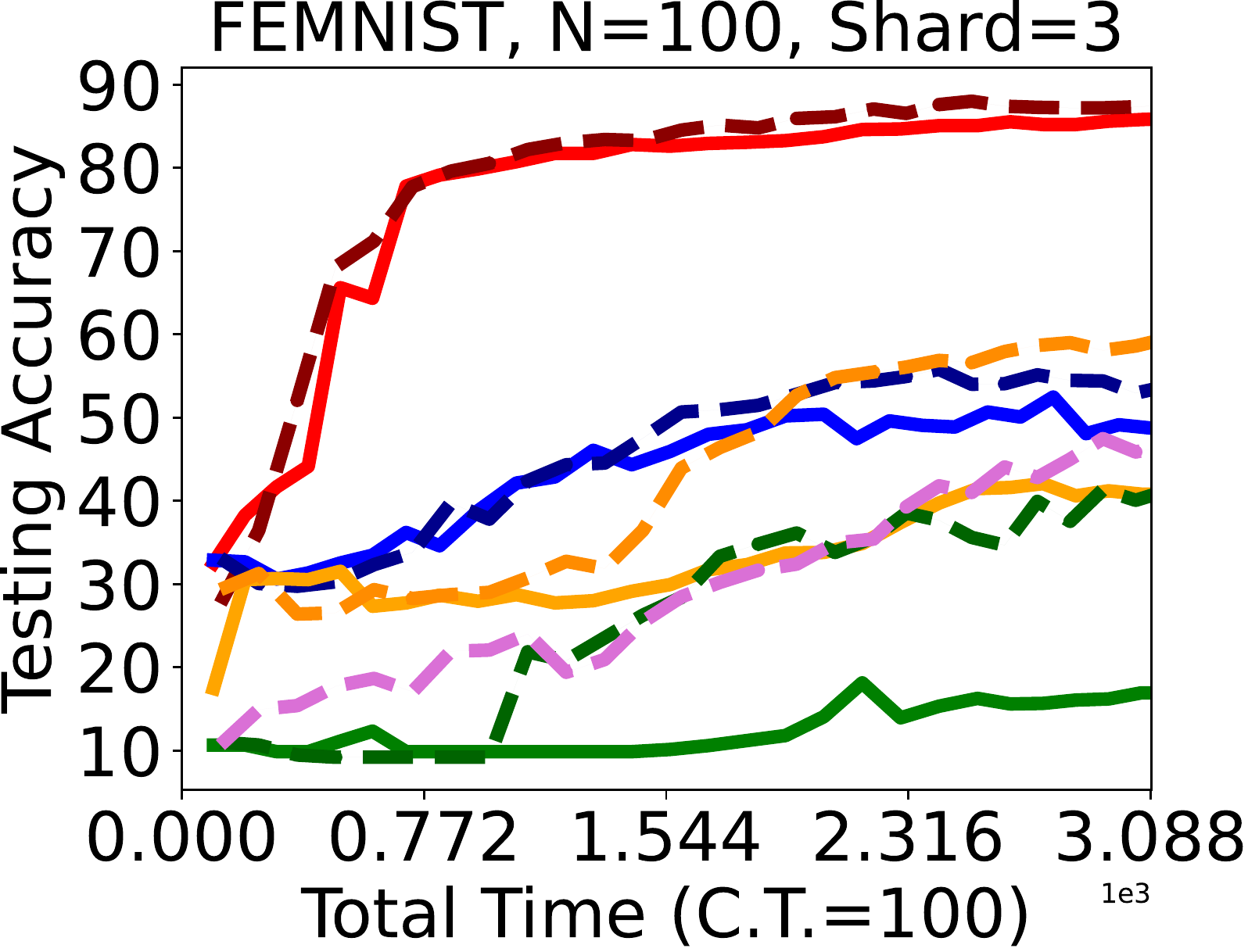} 
    \end{tabular}
    \vspace{-0mm}
    \caption{Empirical results on CIFAR10, CIFAR100, EMNIST, and FEMNIST with full participation $M=N$, and \emph{fixed computation speeds}. "Shard" stands for the number of different classes per client and C.T. denotes the communication cost per round. \texttt{SRPFL} is short for \texttt{FedRep-SRPFL}.}
    \vspace{-3mm}
    \label{fig:experiment_main}
\end{figure*}


\vspace{-2mm}
\section*{Acknowledgements}
\vspace{-2mm}
The research of Mokhtari and Tziotis is  supported in part by NSF Grants 2007668, 2019844, and 2112471, ARO Grant W911NF2110226,  the Machine Learning Lab (MLL) at UT Austin, and the Wireless Networking and Communications Group (WNCG) Industrial Affiliates Program. The research of Hassani and Shen is supported by NSF Grants 1837253, 1943064, AFOSR Grant FA9550-20-1-0111, DCIST-CRA, and the AI Institute for Learning-Enabled Optimization at Scale (TILOS). The research of Pedarsani is supported by NSF Grant 2003035. 

\bibliographystyle{ieeetr}
\bibliography{arXiv}
\appendix
\newpage
\appendix
\onecolumn
\section{Appendix}
\label{Appendix_A}
Before we dive into the analysis we provide the following useful definition.
\begin{definition} \label{eqn_A_subggausian}
For a random vector $\x \in \mathbb{R}^{d_1}$ and a fixed matrix $\A \in \mathbb{R}^{d_1\times d_2}$, the vector $\A^\top \x$ is called $\norm{\A}_2$-subgaussian if $\y^\top\A^\top\x$ is subgaussian with subgaussian norm $\OO\pr{\norm{\A_2}\norm{\y}_2}$ for all $\y \in \mathbb{R}^{d_2}$, i.e. $\Ex{\exp\pr{\y^\top \A^\top \x}}\leq \exp\pr{\norm{\y}^2_2\norm{\A}^2_2/2}$.  
\end{definition}

We study the performance of \texttt{SRFRL} with \texttt{FedRep} as the subroutine of choice. The first part of our analysis focuses on a single round $t$ and extends the analysis in \cite{collins2021exploiting}. We assume that there are $N$ clients in the network and at round $t$ a subset $\I^t$ of them participate in the learning procedure with cardinality $n\geq n_0 := \frac{2d}{\log(N) \cdot \ssigmax}$. Without loss of generality we assume that the clients are indexed from fastest to slowest thus the clients that participate in the learning process are all $i\in [n]$. Each client $i$, draws a batch of $m ~\geq~c_0 \frac{(1+ \var^2) k^3 \kappa^4 \log(N) }{E_0^2}$ fresh, i.i.d. samples at every round. We denote by $\X_i^t \in \mathbb{R}^{d \times m}$ and $\Y_i^t \in \mathbb{R}^{m}$ the matrix of samples and the labels for client $i$, such that the rows of $\X_i^t$ are samples $\{\x_i^1,...,\x_i^m\}$. By $\Z_i^t \in \mathbb{R}^{m}$ we denote the noise in the measurements of client $i$, with $z_{i,j} \sim \mathcal{N}(0,\var^2)$. Let $\hB^* \in \mathbb{R}^{d \times k}$ and $\W^* \in \mathbb{R}^{N \times k}$ stand for the optimal representation and the concatenation of optimal heads respectively. The hat denotes that a matrix is orthonormal i.e. its columns form an orthonormal set. Similarly, $\hB^t \in \mathbb{R}^{d \times k}$ and $\W^t \in  \mathbb{R}^{N \times k}$ denote the global representation and the concatenation of the heads at round $t$. $\w_i^{*}$'s and $\w_i^{t}$'s denote the optimal heads and the heads at round $t$ which constitute the rows of $\W^*$ and $\W^t$ respectively. Furthermore we define $\ssigmin:= \min_{\mathcal{I}\in [N], |\mathcal{I}|=n', n'\leq N}\sigma_{\min}\frac{1}{\sqrt{n'}}\W^*_{\mathcal{I}}$ and $\ssigmax := \max_{\mathcal{I}\in [N], |\mathcal{I}|=n', n'\leq N}\sigma_{\max}\frac{1}{\sqrt{n'}}\W^*_{\mathcal{I}}$, where $\W_\I$ is formed by taking the rows of $\W$ indexed by $\I$. That is $\ssigmax$ and $\ssigmin$ are the maximum and minimum singular values of any submatrix $\W^*_\I$ that can be obtained throughout the course of our algorithm. Notice that by assumption \ref{Assum3} each row of $\W^*$ has norm $\sqrt{k}$, so $\frac{1}{n'}$ acts as a normalization factor such that $\norm{\frac{1}{n'}\W^*_\I}_F = \sqrt{k}$. Finally we define $\kappa = \ssigmax/\ssigmin$. Since we focus on a single round the time index can be dropped for simplicity. Further, henceforth we drop the subscripts $\mathcal{I}^t$ on $\W^t$.

First we derive the update scheme our algorithm follows. Notice that the empirical objective function given in \eqref{Empirical_Loss_General} can be expressed via matrices $\X_i$ and $\Y_i$,
\begin{align}
     L_N(\B, \W)=\frac{1}{2mn}\sum_{i=1}^n\pr{\Y_i-\X_i\hB \w_i}^2
\end{align}
Further, computing the gradients we derive
\begin{align}
     \frac{1}{2mn}\sum_{i=1}^n\nabla_{\hB}\pr{\Y_i-\X_i\hB \w_i}^2 &=
     \frac{1}{mn}\sum_{i=1}^n\X_i^\top\pr{\X_i\hB \w_i-\Y_i}\w_i^\top,\\
      \frac{1}{2mn}\nabla_{\w_i}\sum_{j=1}^n\pr{\Y_j-\X_j\hB \w_j}^2 &= \frac{1}{mn}\hB^\top \X_i^\top\pr{\X_i\hB \w_i-\Y_i}
\end{align}
and since $ \pr{\frac{1}{m}\hB^\top \X_i^\top \X_i \hB}$ is invertible with high probability by Lemma \ref{LemmaInvert}, solving for the minimizer gives us
\begin{align}
   \w^+_i = \pr{\frac{1}{m}\hB^\top \X_i^\top \X_i \hB}^{-1}\frac{1}{m}\hB^\top \X_i^\top \Y_i
\end{align}
Thus our update scheme with stepsize  $\eta$ is
\begin{align}
     \forall i \in [n] \quad \w_i^+ &= \pr{\frac{1}{m}\hB^\top \X_i^\top \X_i \hB}^{-1}\frac{1}{m}\hB^\top \X_i^\top \Y_i\label{w_update} \\
    \B^+&= \hB - \frac{\eta}{mn}\sum_{i=1}^n \X_i^\top\pr{\X_i \hB \w_i^+-\Y_i}\w_i^{+\top}\label{B_update}\\
    \hB^+, \R^+ &= \textit{QR}(\B^+)\label{Bnorm_update}
\end{align}
where \textit{QR} denotes the QR decomposition and $\Y_i = \X_i \hB^*\w^*_i + \Z_i$. 
\begin{lemma}
For every client $i$ the update for $\w_i$ can be expressed as follows :
\begin{align}
\w_i^+ &=   \hB^{\top} \hB^*\w^*_i +\F_i + \G_i,\label{w+lemma}
\end{align}
where $\F_i$ and $\G_i$ are defined in equations \eqref{Feq} and \eqref{Geq}, respectively. 
\end{lemma}
\begin{proof}
Further expanding \eqref{w_update} we can write
\begin{align}
    \w_i^+ &=\pr{\frac{1}{m}\hB^\top \X_i^\top \X_i \hB}^{-1}\frac{1}{m}\hB^\top \X_i^\top \X_i \hB^*\w_i^*+\pr{\frac{1}{m}\hB^\top \X_i^\top \X_i \hB}^{-1}\frac{1}{m}\hB^\top \X_i^\top\Z_i\\
    &= \hB^{\top} \hB^*\w^*_i+ \pr{\frac{1}{m}\hB^\top \X_i^\top \X_i \hB}^{-1}\pr{\frac{1}{m}\hB^\top \X_i^\top \X_i \hB^*-\frac{1}{m}\hB^\top \X_i^\top \X_i \hB \hB^\top \hB^*}\w^*_i \nonumber \\
    &\quad + \pr{\frac{1}{m}\hB^\top \X_i^\top \X_i \hB}^{-1}\frac{1}{m}\hB^\top \X_i^\top \Z_i\\
    &= \hB^{\top} \hB^*\w^*_i +\F_i + \G_i
\end{align}
where we define 
\begin{align}
    \F_i &:= \pr{\frac{1}{m}\hB^\top \X_i^\top \X_i \hB}^{-1}\pr{\frac{1}{m}\hB^\top \X_i^\top \X_i \hB^*-\frac{1}{m}\hB^\top \X_i^\top \X_i \hB \hB^\top \hB^*}\w^*_i, \label{Feq}\\
    \G_i &:= \pr{\frac{1}{m}\hB^\top \X_i^\top \X_i \hB}^{-1}\frac{1}{m}\hB^\top \X_i^\top \Z_i\label{Geq}
\end{align}
\end{proof}
We further have the following immediate corollary.
\begin{corollary}
Let $\W^+, \F$ and $\G$ be the matrices with rows the concatenation of $\w^+_i, \F_i$ and $\G_i$, respectively. Then
\begin{align}
    \W^+ = \W^* \hB^* \hB + \F + \G \label{matrixversion}
\end{align}
\end{corollary}
Our first goal is to control the norm of $\w_i^+$. In order to achieve that we provide lemmas that bound the norms of $\F_i$ and $\G_i$ extending the analysis in \cite{collins2021exploiting} and \cite{10.1145/2488608.2488693}.
\begin{lemma} \label{LemmaInvert}
Let $\delta = c\frac{k^{3/2}\sqrt{\log(N)}}{\sqrt{m}}$ for some absolute constant $c$, then with probability at least $1-\exp{\pr{-115k^3\log(N)}}$
\begin{align}
    \forall i\in [n], \quad \sigmin\pr{\frac{1}{m}\hB^\top \X_i^\top\X_i\hB} \geq 1-\delta
\end{align}
 It follows that with the same probability
\begin{align}
    \forall i\in [n], \quad \sigmax\pr{\frac{1}{m}\hB^\top \X_i^\top\X_i\hB}^{-1} \leq \frac{1}{1-\delta} \label{sigmax}
\end{align}

\end{lemma}
\begin{proof}
First notice that we can rewrite
\begin{align}
    \frac{1}{m}\hB^\top\X_i^\top\X_i\hB = \sum_{j=1}^{m}  \frac{1}{\sqrt{m}}\hB^\top\x_i^j\pr{\frac{1}{\sqrt{m}}\hB^\top\x_i^j}^\top
\end{align}
For all $i\in [n], j\in [m]$ we define $\bv_i^j:= \frac{1}{\sqrt{m}}\hB^\top\x_i^j$ such that each $\bv_i^j$ is an i.i.d. {$\|\frac{1}{\sqrt{m}}\hB\|_2$}-subgaussian random variable (please see the definition of $\|\mathbf{A}\|_2$-subgaussian in Definition \ref{eqn_A_subggausian}) and thus by equation $(4.22)$(Theorem $4.6.1$) in \cite{vershynin_2018} we obtain the following bound for any $m\geq k, l\geq 0$
\begin{align}
    \sigmin\pr{\frac{1}{m}\hB^\top \X_i^\top\X_i\hB} \geq 1-c_1\pr{\sqrt{\frac{k}{m}}+\frac{l}{\sqrt{m}}}, 
\end{align}
with probability at least $1-\exp{\pr{-l^2}}$ and $c_1$ some absolute constant. We set $l=12k^{3/2}\log(N)\ \sqrt{k}$ and $\delta_1 = \frac{12c_1k^{3/2}\sqrt{\log(N)}}{\sqrt{m}}$ and the above bound becomes
\begin{align}
    \sigmin\pr{\frac{1}{m}\hB^\top \X_i^\top\X_i\hB} \geq 1-\delta_1,
\end{align}
with probability at least $1-\exp{\pr{-k\pr{12k\sqrt{\log(N)}-1}^2}}$

Further notice that 
\begin{align}
    \exp{\pr{-k\pr{12k\sqrt{\log(N)}-1}^2}}&= \exp{\pr{k\pr{-144k^2\log(N)+24k\log(N)-1}}} \\
    &\leq \exp{\pr{-120k^3\log(N)}}
\end{align}
Thus taking Union Bound over $i\in [n]$ we have that for all $i \in [n]$ 
\begin{align}
    \sigmin\pr{\frac{1}{m}\hB^\top \X_i^\top\X_i\hB} \geq 1-\delta_1
\end{align}
with probability at least
\begin{align}
    1- n\exp{\pr{-120k^3\log(N)}}\geq 1- \exp{\pr{-115k^3\log(N)}}
\end{align}
Choosing $c$ sufficiently large derives the statement of the lemma.
\end{proof}

\begin{lemma}\label{Hlemma}
Let $\bH_i := \pr{\frac{1}{\sqrt{m}}\hB^\top\X_i^\top} \frac{1}{\sqrt{m}}\X_i\pr{\hB \hB^\top - \bI_d}\hB^*$ and $\delta := c \frac{k^{3/2}\sqrt{\log(N)}}{\sqrt{m}}$, for an absolute constant $c$. Then with probability at least $1-\exp\pr{-115k^2\log(N)}$ we have
\begin{align}
    \sum_{i=1}^n\norm{\bH_i\w^*_i}^2_2\leq \delta^2 \norm{\W^*}^2_2 \dist^2\pr{\hB, \hB^*}\label{Hsum}
\end{align}
\end{lemma}
\begin{proof}
In order to argue about the quantity $\bH_i =\pr{\frac{1}{\sqrt{m}}\hB^\top\X_i^\top} \frac{1}{\sqrt{m}}\X_i\pr{\hB \hB^\top - \bI_d}\hB^*$ we define matrix $\U := \frac{1}{\sqrt{m}}\X_i\pr{\hB \hB^\top - \bI_d}\hB^*$ such that its $j$-th row, $\bu_j = \frac{1}{\sqrt{m}} \hB^{*\top}\pr{\hB \hB^\top - \bI_d}\x_i^j$ is subgaussian with norm at most $ \frac{1}{\sqrt{m}} \hB^{*\top}\pr{\hB \hB^\top - \bI_d}$. Similarly we define $\V = \frac{1}{\sqrt{m}}\hB^\top\X_i^\top$ such that its $j$-th row $\bv_j = \frac{1}{\sqrt{m}}\hB\x_i^j$ has norm at most $\frac{1}{\sqrt{m}}\hB$. We are now ready to use a concentration argument similar to Proposition $(4.4.5)$ in \cite{}. Let $\mathcal{S}^{k-1}$ denote the unit sphere in $k$ dimensions and $\mathcal{N}_k$ the $1/4$-net of cardinality $9^k$. From equation $(4.13)$ in \cite{vershynin_2018} we have 
\begin{align}
    \norm{\pr{\hB^\top\X_i^\top} \X_i\pr{\hB \hB^\top - \bI_d}\hB^*}_2=\norm{\U^{\top} \V}_2 &\leq 2 \max_{\p,\y\in \mathcal{N}_k}  \p^\top \pr{\sum_{j=1}^m\bu_j\bv_j^\top}\y\\
    & = 2 \max_{\p,\y\in \mathcal{N}_k} \sum_{j=1}^m \Inp{\p, \bu_j} \Inp{\bv_j, \y} \label{subg1}
\end{align}
By definition $\Inp{\p, \bu_j}$ and $\Inp{\bv_j, \y}$ are subgaussians with norms $\frac{1}{\sqrt{m}}\norm{\hB^{*\top}\pr{\hB\hB^\top - \bI_d}}_2= \frac{1}{\sqrt{m}}\dist\pr{\hB, \hB^*}$ and $\frac{1}{\sqrt{m}}\hat{\norm{\B}}_2 = \frac{1}{\sqrt{m}}$ respectively and thus for all $j\in [m]$ the product $\Inp{\p, \bu_j} \Inp{\bv_j, \y}$ is subexponential with norm at most $\frac{C'}{m}\dist\pr{\hB, \hB^*}$, for some constant $C'$. Note that 
\begin{align}
    \Ex{\Inp{\p, \bu_j}\Inp{\bv,\y}} = \p^\top \pr{\hB^{*\top}\pr{\bI_d- \hB\hB^\top}\hB}\y = 0 
\end{align}
and thus we can use Bernstein's inequality to bound the sum of $m$ zero mean subexponential random variables, for any fixed pair $\p,\y \in \mathcal{N}_k$:
\begin{align}
    \prb{\sum_{i=1}^m\Inp{\p, \bu_j} \Inp{\bv_j, \y}\geq s} &\leq \exp\pr{-c_2\min\ag{ \frac{s^2m^2}{\dist^2\pr{\hB,\hB^*}},\frac{sm}{\dist\pr{\hB, \hB^*}}}}\\
    &\leq \exp\pr{-c_2m\min\ag{ \frac{s^2}{\dist^2\pr{\hB,\hB^*}},\frac{s}{\dist\pr{\hB, \hB^*}}}}
\end{align}
for constant $c_2$. Thus taking Union Bound over all the point in the net we derive
\begin{align}
     \prb{\forall \p,\y \in \mathcal{N}_k, \quad 2\sum_{j=1}^m\Inp{\p, \bu_j} \Inp{\bv_j, \y}\geq 2s} &\leq 
     9^{2k}\exp\pr{-c_2m\min\ag{ \frac{s^2}{\dist^2\pr{\hB,\hB^*}},\frac{s}{\dist\pr{\hB, \hB^*}}}}
\end{align}
Since $m > Ck^2\log(N)$, by setting $s=\dist\pr{\hB, \hB^*}\sqrt{\frac{Ck^2\log(N)}{4m}}$ and \eqref{subg1} we obtain
\begin{align}
    &\prb{\frac{1}{m} \norm{\pr{\hB^\top\X_i^\top} \X_i\pr{\hB \hB^\top - \bI_d}\hB^*}_2\geq \dist\pr{\hB, \hB^*}\sqrt{\frac{Ck^2\log(N)}{m}} }\nonumber\\ &\hspace{3.5in}\leq 9^{2k}\exp \pr{-  c_2 m\frac{s^2}{\dist^2\pr{\hB, \hB^*}}}\\
    & \hspace{3.5in} \leq 9^{2k}\exp \pr{- C \cdot c_2 mk^2\log(N)}\\
    &\hspace{3.5in} \leq \exp \pr{-120k^2\log(N)} 
\end{align}
for sufficiently large $C$. Using Union Bound again over all participating clients we get
\begin{align}
    \prb{\forall i \in [n] \quad \norm{\bH_i}_2\leq \dist\pr{\hB, \hB^*}\sqrt{\frac{Ck^2\log(N)}{m}} }& \geq 1- n\exp \pr{-120k^2\log(N)} \\
    & \geq 1- \exp\pr{-115k^2\log(N)}\label{Hperclient}
\end{align}
The above also implies
\begin{align}
    \prb{\frac{1}{n} \sum_{i=1}^n\norm{\bH_i}^2_2\leq C \dist^2\pr{\hB, \hB^*}\frac{k^2\log(N)}{m}} & \geq 1- \exp\pr{-115k^2\log(N)}\\
    \prb{\frac{k}{n}\norm{\W^*}^2_2 \sum_{i=1}^n\norm{\bH_i}^2_2\leq C \norm{\W^*}^2_2 \dist^2\pr{\hB, \hB^*}\frac{k^3\log(N)}{m}} & \geq 1-\exp\pr{-115k^2\log(N)}\label{intermed1}
\end{align}
Finally notice that 
\begin{align}
    \sum_{i=1}^n\norm{\bH_i\w^*_i}^2_2 &\leq \sum_{i=1}^n\norm{\bH_i}_2^2k \leq \frac{\norm{\W^*}_F^2}{n}\sum_{i=1}^n \norm{\bH_i}_2^2 \leq \frac{k}{n}\norm{\W^*}_2^2\sum_{i=1}^n \norm{\bH_i}_2^2,
\end{align}
where we used Assumption \eqref{Assum3} and the fact that the rank of $\W^*$ is $k$. Combining this with \eqref{intermed1} and choosing sufficiently large $c$ we derive the result.
\end{proof}
Building on the previous lemmas we can now bound the norm of $\F_i$.
\begin{lemma}
Let $\delta := c \frac{k^{3/2} \sqrt{\log(N)}}{\sqrt{m}}$ for some absolute constant $c$ and for all $i\in [n]$ let $\F_i$ given by \eqref{Feq}. Further let matrix $\F\in\mathbb{R}^{n \times k}$ such that its rows are the concatenation of $\F_i$'s. Then with probability at least $1- \exp \pr{-110k^2\log(N)}$ we have
\begin{align}
    \forall i \in[n] \quad \norm{\F_i}_2 &\leq \frac{\delta}{1- \delta}\dist\pr{\hB, \hB^*}\norm{\w_i^*}_2,\label{Flemma}\\
    \norm{\F}_F &\leq \frac{\delta}{1-\delta} \dist\pr{\hB, \hB^*}\norm{\W^*}_2\label{matrixFlemma}
\end{align}
\end{lemma}
\begin{proof}
\begin{align}
\norm{\F_i}^2_2 &\leq \norm{\pr{\frac{1}{m}\hB^\top \X_i^\top \X_i \hB}^{-1}}_2^2 \norm{\bH_i}_2^2 \norm{ \w_i^*}_2^2\\
&\leq \frac{\delta^2}{\pr{1-\delta}^2}  \cdot \dist^2\pr{\hB, \hB^*} \norm{\w_i^*}^2_2
\end{align}
which holds for all $i \in [n]$ with probability at least $1-\exp \pr{-110k^2\log(N)}$ by using Union Bound on the failure probability of \eqref{sigmax} and \eqref{Hperclient}. 
Similarly, we have
\begin{align}
    \norm{\F}_F^2 = \sum_{i=1}^n \norm{\F_i}^2_2 &\leq \sum_{i=1}^m \norm{\pr{\frac{1}{m}\hB^\top \X_i^\top \X_i \hB}^{-1}}_2^2 \norm{\bH_i \w_i^*}_2^2\\
    &\leq \frac{1}{\pr{1-\delta}^2} \sum_{i=1}^m\norm{\bH_i \w_i^*}_2^2\\
    &\leq \frac{\delta^2}{\pr{1-\delta}^2}  \cdot \dist^2\pr{\hB, \hB^*} \norm{\W^*}^2_2
\end{align}
which holds with probability at least $1-\exp \pr{-110k^2\log(N)}$ taking Union Bound on the failure probability on  \eqref{sigmax} and \eqref{Hsum}. 
\end{proof}
We now turn our attention on deriving a bound for $\norm{\G_i}_2$. 
\begin{lemma}
Let $\delta := c \frac{k^{3/2} \sqrt{\log(N)}}{\sqrt{m}}$ for some absolute constant $c$ and for all $i\in [n]$ let $\G_i$ given by \eqref{Geq}. Further let matrix $\G\in\mathbb{R}^{n \times k}$ such that its rows are the concatenation of $\G_i$'s. Then with probability at least $1- \exp \pr{-110k^2\log(N)}$ we have
\begin{align}
    \forall i \in[n] \quad \norm{\G_i}_2 &\leq \frac{\delta}{1- \delta}\var^2,\label{Glemma}\\
    \norm{\G}_F &\leq \frac{\delta}{1-\delta} \sqrt{n}\var^2\label{matrixGlemma}
\end{align}
\end{lemma}
\begin{proof}
Notice that we can write 
\begin{align}
    \G_i = \pr{\frac{1}{m}\hB^\top \X_i^\top \X_i \hB}^{-1} \frac{1}{m}\hB^\top \X_i^\top \Z_i = \pr{\frac{1}{m}\hB^\top \X_i^\top \X_i \hB}^{-1} \frac{1}{m} \sum_{i=1}^m z_{i}^j \hB^\top \x_i^j,
\end{align}
and since $z_i^j \sim \mathcal{N}\pr{0,\var^2}$ we can conclude that for all $i,j,\hspace{0.15cm} z_{i}^j \hB^\top \x_i^j$ is an i.i.d. zero mean subexponential with norm at most $C'_2 \var^2 \hat{\norm{B}}_2= C'_2 \var^2$ for some constant $C_2$. 
Once again we denote by $\mathcal{S}^{k-1}$ the unit sphere in $k$ dimensions and by $\mathcal{N}_k$ the $1/4$-net with cardinality $9^k$. Using Bernstein's inequality and Union Bound over all the points on the net we follow the derivations from Lemma \ref{Hlemma} to get
\begin{align}
    \prb{\norm{\frac{1}{m} \sum_{i=1}^mz_{i}^j \hB^\top \x_i^j}_2\geq 2s} \leq 9^{k+1} \exp\pr{-c_3m\min \ag{\frac{s^2}{\var^4}, \frac{s}{\var^2}}}
\end{align}
Since $m>C_2 k^2 \log(N)$, by setting $s=\var^2\sqrt{\frac{C_2k^2\log(N)}{4m}}$ we derive
\begin{align}
    \prb{\norm{\frac{1}{m} \sum_{i=1}^mz_{i}^j \hB^\top \x_i^j}_2\geq \var^2\sqrt{\frac{C_2k^2\log(N)}{m}}} &\leq 9^{k+1} \exp\pr{-c_3m \frac{s^2}{\var^4}}\\
    &\leq 9^{k+1} \exp\pr{-C_2 \cdot c_3 k^2 \log(N)}\\
    &\leq \exp\pr{-115 k^2 \log(N)}
\end{align}
for sufficiently large $C_2$. Choosing $c$ large enough and taking Union Bound over all $i \in [n]$ we can obtain
\begin{align}
    \prb{\forall i \in [n] \quad \norm{\frac{1}{m} \sum_{i=1}^mz_{i}^j \hB^\top \x_i^j}_2\leq \var^2\delta} &\leq 1 - n\exp\pr{-115 k^2 \log(N)} \\
    &\leq 1- \exp\pr{-113 k^2 \log(N)} \\\label{noiseineq}
\end{align}
Finally taking Union Bound over the failure probabilities of \eqref{sigmax}  and \eqref{noiseineq} we get
\begin{align}
    \forall i \in [n] \quad \norm{\G_i}_2 \leq \norm{\pr{\frac{1}{m}\hB^\top \X_i^\top \X_i \hB}^{-1}}_2 \norm{\frac{1}{m} \sum_{i=1}^m z_{i}^j \hB^\top \x_i^j}_2 \leq \frac{\delta}{1-\delta} \var^2
\end{align}
with probability at least $1 - n\exp\pr{-110 k^2 \log(N)}$.
It follows that with the same probability
\begin{align}
    \norm{\G}_F^2 = \sum_{i=1}^n \norm{\G_i}^2_2 \leq n \pr{\frac{\delta}{1-\delta}}^2\var^4
\end{align}
\end{proof}
For all $i \in [n]$ we define $\q_i := \hB\w^+_i - \hB^*\w_i^*$. The following lemma provides upper bounds on the norms of $\w_i^+$ and $\q_i$
\begin{lemma} \label{wqlemma}
Let $\delta := c \frac{k^{3/2} \sqrt{\log(N)}}{\sqrt{m}}$ for some absolute constant $c$ and $\hd=\delta/(1-\delta)$. Then with probability at least $1-\exp\pr{-105k^2\log(N)}$ we have
\begin{align}
    \forall i \in [n] \quad \normt{\w_i^+} \leq 2\sqrt{k} +\var^2\hd
\end{align}
Further with probability at least $1-\exp\pr{-105k^2\log(N)}$ we have
\begin{align}
    \forall i \in [n] \quad \normt{\q_i} \leq 2\sqrt{k}\cdot \dist\pr{\hB, \hB^*} +\var^2\hd
\end{align}
\end{lemma}
\begin{proof}
\begin{align}
\normt{\w_i^+} & \leq   \normt{\hB^{\top}}\normt{ \hB^*}\normt{\w^*_i} +\normt{\F_i} + \normt{\G_i}\\
& \leq   \normt{\w_i^*} +\hd \normt{\w_i^*} \cdot \dist\pr{\hB, \hB^*} + \hd \var^2\\
& \leq 2\sqrt{k} + \hd \var^2
\end{align}
where the first inequality comes from \eqref{w+lemma} and the third from Assumption \ref{Assum3}. For the second inequality we take Union Bound over the failure probability of \eqref{Flemma} and \eqref{Glemma} and thus the above result holds with probability at least $1-\exp\pr{-107k^2\log(N)}$. Taking Union Bound for all $i\in [n]$ we get that with probability at least $1-\exp\pr{-105k^2\log(N)}$
\begin{align}
    \forall i \in [n] \quad \normt{\w_i^+} \leq 2\sqrt{k} +\var^2\hd
\end{align}
and the first result of the lemma follows. For the second part we have
\begin{align}
    \normt{\q_i} = \normt{\hB\w^+_i - \hB^*\w_i^*} &\leq \normt{\hB\hB^\top \hB^* \w_i^* + \hB \F_i+ \hB \G_i - \hB^* \w_i^*}\\
    &\leq \normt{\pr{\hB\hB^\top - \bI_d}\hB^*\w_i^*} + \normt{\hB \F_i}+ \normt{\hB \G_i}\\
    &  \leq  \normt{\hB_\bot \hB^*}\normt{\w^*_i} +\normt{\F_i} + \normt{\G_i}\\
    & \leq \dist\pr{\hB, \hB^*} \normt{\w_i^*} + \dist\pr{\hB, \hB^*}\hd \normt{\w_i^*} +\var^2 \hd\\
    & \leq \dist\pr{\hB, \hB^*} 2\sqrt{k} + \var^2 \hd
\end{align}
where the first inequality comes from \eqref{w+lemma}. For the forth inequality we take Union Bound over the failure probability of \eqref{Flemma} and \eqref{Glemma} and thus the above result holds with probability at least $1-\exp\pr{-107k^2\log(N)}$. Taking Union Bound for all $i\in [n]$ we get that with probability at least $1-\exp\pr{-105k^2\log(N)}$
\begin{align}
    \forall i \in [n] \quad \normt{\q_i} \leq 2\sqrt{k} \cdot \dist\pr{\hB, \hB^*} + \var^2 \hd
\end{align}
\end{proof}

\begin{lemma}
Let $\delta := c \frac{k^{3/2} \sqrt{\log(N)}}{\sqrt{m}}$ for some absolute constant $c$ and $\hd=\delta/(1-\delta)$. Then with probability at least $1-\exp\pr{-105d}-\exp\pr{-105k^2\log(N)}$ we have
\begin{align}
   \normt{\frac{1}{mn} \sum_{i=1}^n \X_i^\top \Z_i\w^{+ \top}_i} \leq c\cdot \frac{\var^2\pr{\sqrt{k}+\hd \var^2}\sqrt{d}}{\sqrt{mn}} \label{noiseterm}
\end{align}
\end{lemma}
\begin{proof}
Let $\mathcal{S}^{d-1}, \mathcal{S}^{k-1}$ denote the unit spheres in $d$ and $k$ dimensions and $\mathcal{N}_d, \mathcal{N}_k$ the $1/4$-nets of cardinality $9^d$ and $9^k$, respectively. By equation $4.13$ in \cite{vershynin_2018} we have 
\begin{align}
    \normt{\frac{1}{mn} \sum_{i=1}^n \X_i^\top \Z_i\w^{+ \top}_i} &\leq 2 \max_{\p \in \mathcal{N}_d,\y\in \mathcal{N}_k}  \p^\top \pr{\sum_{i=1}^n \frac{1}{mn}\X_i^\top \Z_i \w_i^{+\top}}\y\\
    & = \leq 2 \max_{\p \in \mathcal{N}_d,\y\in \mathcal{N}_k}  \p^\top \pr{\sum_{i=1}^n \sum_{j=1}^m \frac{z_i^j}{mn}\x_i^j \w_i^{+\top}}\y \\
    & = \leq 2 \max_{\p \in \mathcal{N}_d,\y\in \mathcal{N}_k} \sum_{i=1}^n \sum_{j=1}^m \pr{\frac{z_i^j}{mn} \Inp{\x_i^j, \p} \Inp{ \w_i^+, \y}} \label{subg2}
\end{align}
Notice that for any fixed $ \p, \y$ and $\forall i \in [n], j\in[m]$ the random variables $\frac{z_i^j}{mn}\Inp{\x_i^j, \p} \Inp{ \w_i^+, \y}$ are i.i.d. zero mean subexponentials with norm at most $\frac{C'_3\var^2\normt{\w_i^+}}{mn}$, for some absolute constant $C'_3$. Conditioning on the event \begin{align}
    \mathcal{E}_1 := \bigcap_{i=1}^n \ag{\normt{\w_i^+}\leq 2\sqrt{k}+\hd \var^2},
\end{align}
which holds with probability at least $1- \exp \pr{-105k^2\log(N)}$ by Lemma \ref{wqlemma}, we can invoke Bernstein's inequality to get \begin{align}
    \prb{ \sum_{i=1}^n\sum_{j=1}^m\frac{z_i^j}{mn}\left.\Inp{\x_i^j, \p} \Inp{ \w_i^+, \y}\geq s \right| \cE_1} &\leq \exp\pr{-c_4 mn \min \ag{\frac{s^2}{\var^4\pr{2\sqrt{k}+\var^2 \hd}^2}, \frac{s}{\var^2\pr{2\sqrt{k}+\var^2\hd}}}}
\end{align}
Since $m>\frac{d\cdot C_3}{n_0}\geq \frac{d\cdot C_3}{n}$ by setting $s= \frac{\var^2\pr{2\sqrt{k}+\hd \var^2 }\sqrt{d \cdot C_3}} {8\sqrt{mn}}$ the above quantity simplifies as follows
\begin{align}
    \prb{ \sum_{i=1}^n\sum_{j=1}^m\frac{z_i^j}{mn}\left.\Inp{\x_i^j, \p} \Inp{ \w_i^+, \y}\geq \frac{\var^2\pr{2\sqrt{k}+\hd \var^2}\sqrt{d \cdot C_3}}{\sqrt{mn}} \right| \cE_1} &\leq \exp\pr{-c_4 mn \frac{s^2}{\var^4\pr{2\sqrt{k}+\var^2 \hd}^2}}\nonumber\\
    &\leq \exp\pr{-C_3  \cdot c_4 \cdot d}\\
    &\leq \exp\pr{-110d}
\end{align}
for $C_3$ large enough.
Taking Union Bound over all points $\p, \y$ on the $\mathcal{N}_d, \mathcal{N}_k$ and using \eqref{subg2} we derive
\begin{align}
    \prb{\normt{\frac{1}{mn} \sum_{i=1}^n \X_i^\top \Z_i\w^{+ \top}_i} \left. \geq \sqrt{C_3}\frac{\var^2\pr{\sqrt{k}+\hd \var^2}\sqrt{d}}{\sqrt{mn}} \right| \cE_1} &\leq 9^{d+k} \exp\pr{-110d}\\
    &\leq \exp\pr{-105d}
\end{align}
and removing the conditioning on $\cE_1$ we get 
\begin{align}
      \prb{\normt{\frac{1}{mn} \sum_{i=1}^n \X_i^\top \Z_i\w^{+ \top}_i}  \geq \sqrt{C_3}\frac{\var^2\pr{\sqrt{k}+\hd \var^2}\sqrt{d}}{\sqrt{mn}}} &\leq \exp\pr{-105d} + \prb{\cE_1^C}\\
    &\leq \exp\pr{-105d}+ \exp\pr{-105k^2\log(N)}
\end{align}
Choosing $c$ large enough and taking the complementary event derives the result.
\end{proof}
\begin{lemma}
Let $\delta := c \frac{k^{3/2} \sqrt{\log(N)}}{\sqrt{m}}$ for some absolute constant $c$ and $\hd=\delta/(1-\delta)$. Then with probability at least $1-\exp\pr{-100d}-\exp\pr{-100k^2\log(N)}$ we have
\begin{align}
   & \normt{\frac{1}{n}\sum_{i=1}^n \pr{\frac{1}{m}\X_i^\top \X_i \pr{\hB\w_i^+ - \hB^* \w^*_i}- \pr{\hB\w_i^+ - \hB^* \w^*_i}}\w^{+\top}_i}\nonumber \\
   &\leq c\cdot \frac{\sqrt{ d}\pr{\dist \pr{\hB, \hB^*}k + \sqrt{k}\hd\var^2 + \pr{\hd\var^2}^2}}{\sqrt{mn}}\label{indicator}
\end{align}
\end{lemma}
\begin{proof}
Let us define the event 
\begin{align}
    \cE_2:= \bigcap_{i=1}^n\ag{\normt{\w_i^+} \leq 2\sqrt{k}+\hd \var^2 \quad \bigcap \quad \normt{\q_i}\leq \dist\pr{\hB, \hB^*}2\sqrt{k}+\hd \var^2}, \label{Etwotime}
\end{align}
which happens with probability at least $1- \exp\pr{-100k^2 \log(N)}$
by Union Bound and Lemma \ref{wqlemma}. For the rest of this proof we work conditioning on event $\cE_2$. Recall that $q_i:= \hB\w_i^+ - \hB^* \w^*_i$ and thus we can write
\begin{align*}
    &\frac{1}{n}\sum_{i=1}^n \pr{\frac{1}{m}\X_i^\top \X_i \pr{\hB\w_i^+ - \hB^* \w^*_i}- \pr{\hB\w_i^+ - \hB^* \w^*_i}}\w^{+\top}_i \\
    &= \frac{1}{n}\pr{\sum_{i=1}^n\X_i^\top \X_i \q_i\w_i^{+\top} - \sum_{i=1}^n\q_i\w_i^{+\top}}\\
    &= \frac{1}{n}\pr{\frac{1}{m}\sum_{i=1}^n \sum_{j=1}^m\Inp{\x_i^j, \q_i}\x_i^j\w_i^{+\top} - \sum_{i=1}^n\q_i\w_i^{+\top}}
\end{align*}
Let $\mathcal{S}^{d-1}, \mathcal{S}^{k-1}$ denote the unit spheres in $d$ and $k$ dimensions and $\mathcal{N}_d, \mathcal{N}_k$ the $1/4$-nets of cardinality $9^d$ and $9^k$, respectively. By equation $4.13$ in \cite{vershynin_2018} we have 
\begin{align}
   & \normt{\frac{1}{n}\sum_{i=1}^n \sum_{j=1}^m \frac{1}{m}\Inp{\x_i^j, \q_i}\x_i^j\w_i^{+\top} - \frac{1}{n}\sum_{i=1}^n\q_i\w_i^{+\top}} \\
    &\hspace{-0.05in}\leq 2 \hspace{-0.15in}\max_{\p \in \mathcal{N}_d,\y\in \mathcal{N}_k}  \frac{1}{n}\p^\top \pr{\sum_{i=1}^n \sum_{j=1}^m\frac{1}{m}\Inp{\x_i^j, \q_i}\x_i^j\w_i^{+\top} - \frac{1}{n}\sum_{i=1}^n\q_i\w_i^{+\top}}\y\\
    &\hspace{-0.05in} = 2 \hspace{-0.15in}\max_{\p \in \mathcal{N}_d,\y\in \mathcal{N}_k}  \frac{1}{mn} \sum_{i=1}^n \sum_{j=1}^m \hspace{-0.03in} \pr{\Inp{\x_i^j, \q_i}\Inp{\p, \x_i^j}\Inp{\w_i^+, \y} - \Inp{\p, \q_i} \Inp{\w_i^+, \y}\hspace{-0.03in}}
\end{align}
Notice that for any fixed $\p, \y$ the products $\Inp{\x_i^j, \q_i}$ are i.i.d. subgaussians with norm at most $\Tilde{c}_1\norm{\q_i}$ and $\Inp{\p, \x_i^j}$ are i.i.d. subgaussians with norm at most $
\Tilde{c}_2 \normt{\p}= \Tilde{c}_2$. Hence under the event $\cE_2$ the product $\frac{1}{mn}\Inp{\x_i^j, \q_i}\Inp{\p, \x_i^j}\Inp{\w_i^+, \y}$ are subexponentials with norm at most $\frac{C'_4}{mn}\pr{\dist\pr{\hB, \hB^*}k + \sqrt{k}\hd \var^2 + \pr{\hd \var^2}^2}$, for some constant $C'_4$. Also note that
\begin{align}
    \Ex{\Inp{\x_i^j, \q_i}\Inp{\p, \x_i^j}\Inp{\w_i^+, \y}- \Inp{\p, \q_i} \Inp{\w_i^+, \y}} = 0
\end{align}
and thus applying Bernstein's inequality we get
\begin{align}
    &\prb{\frac{1}{mn}\sum_{i=1}^n \sum_{j=1}^m\Inp{\x_i^j, \q_i}\Inp{\p, \x_i^j}\Inp{\w_i^+, \y} -  \left.\frac{1}{n} \sum_{i=1}^n \Inp{\p, \q_i} \Inp{\w_i^+, \y}\geq s\right|\cE_2}\nonumber \\ &\leq \exp\pr{-c_5 \cdot mn \min \ag{\frac{s^2}{\pr{\dist\pr{\hB, \hB^*}k+ \sqrt{k}\hd\var^2+\pr{\hd\var^2}^2}^2}, \frac{s}{\pr{\dist\pr{\hB, \hB^*}k+ \sqrt{k}\hd\var^2+\pr{\hd\var^2}^2}}}}
\end{align}
Since $m>\frac{d\cdot C_4}{n_0}\geq \frac{d \cdot C_4}{n}$ by setting $s = \frac{\sqrt{C_4 \cdot d}\pr{ \dist\pr{\hB, \hB^*}k + \sqrt{k}\hd\var^2 + \pr{\hd\var^2}^2}}{2\sqrt{mn}}$ and taking Union Bound over all $\p \in \mathcal{N}_d, \y\in \mathcal{N}_k$ we derive
\begin{align}
    &\hspace{-0.4in}\prb{\hspace{-0.05in}\normt{\frac{1}{n}\sum_{i=1}^n \pr{\frac{1}{m}\X_i^\top \X_i \pr{\hB\w_i^+ - \hB^* \w^*_i}- \pr{\hB\w_i^+ - \hB^* \w^*_i}}\w^{+\top}_i}\hspace{-0.1in} \geq \hspace{-0.05in} \left.\frac{\sqrt{C_4 \cdot d}\pr{ \dist\pr{\hB, \hB^*}k + \sqrt{k}\hd\var^2 + \pr{\hd\var^2}^2}}{\sqrt{mn}}\right| \cE_2\hspace{-0.05in}} \nonumber\\
    &\qquad  \hspace{5.2cm} \leq 9^{d+k}\exp\pr{  \frac{-c_5\cdot mn s^2}{\pr{\dist\pr{\hB, \hB^*}k+ \sqrt{k}\hd\var^2+\pr{\hd\var^2}^2}^2}}\\
    &\qquad  \hspace{5.2cm} \leq 9^{d+k}\exp\pr{-C_4 \cdot c_5 \cdot d }\\
    &\qquad  \hspace{5.2cm} \leq 9^{d+k}\exp\pr{-120d}\\
    &\qquad  \hspace{5.2cm} \leq \exp\pr{-100d}
\end{align}
choosing a  large enough constant $C_4$. Recall that $\prb{\cE_2^C} \leq \exp \pr{-100k^2\log(N)}$. Hence by removing the conditioning on $\cE_2$ we get that with probability at least $1- \exp\pr{-100d}-\exp\pr{-100k^2\log(N)}$
\begin{align}
    \normt{\frac{1}{n}\sum_{i=1}^n \pr{\frac{1}{m}\X_i^\top \X_i \pr{\hB\w_i^+ - \hB^* \w^*_i}- \pr{\hB\w_i^+ - \hB^* \w^*_i}}\w^{+\top}_i} \leq c\cdot \frac{\sqrt{ d}\pr{ \dist\pr{\hB, \hB^*}k + \sqrt{k}\hd\var^2 + \hd^2\var^4}}{\sqrt{mn}}
\end{align}
for sufficiently large $c$.
\end{proof}
Having set all the building blocks we now proceed to the proof of Theorem \ref{convergence}.
\FirstRestatable*
\begin{proof}
First let us recall the definition of $\delta := c \frac{k^{3/2} \sqrt{\log(N)}}{\sqrt{m}}$ for some absolute constant $c$ and $\hd=\delta/(1-\delta)$. Further notice that for our choice of $m$ and sufficiently large $c_0$ we have the following useful inequality
\begin{align}
    \hd = \frac{\delta}{1- \delta} \leq 2\delta \leq \frac{E_0}{20 \cdot \kappa^2 } \cdot \frac{1}{1+\var^2}  \leq \frac{1}{20}\label{delt}
\end{align}
From the update scheme of our algorithm \eqref{B_update} we have
\begin{align}
    \B^+&= \hB-\frac{\eta}{mn}\pr{\sumi \X_i^\top\X_i\hB\w_i^+\w_i^{+\top} - \sumi \X_i^\top\X_i\hB^*\w_i^*\w_i^{+\top}-\sumi \X_i^\top \Z_i \w_i^{+\top}} \label{Bupdatetwo} \\
    &= \hB - \frac{\eta}{n}\pr{\sumi\pr{\frac{1}{m}\X_i^\top \X_i \pr{\hB \w^+_i - \hB^* \w^*_i}- \pr{\hB \w^+_i - \hB^* \w^*_i}}\w_i^{+\top}} \nonumber\\
    &\quad - \frac{\eta}{n}\sumi\pr{\hB \w^+_i - \hB^* \w^*_i}\w_i^{+\top} + \frac{\eta}{n}\sumi \frac{1}{m}\X_i^\top \Z_i \w_i^{+\top}
\end{align}
where we added and subtracted terms. Multiplying both sides by $\hB_\bot^{*\top}$ we get 
\begin{align}
  \hB_\bot^{*\top}\B^+&=  \hB_\bot^{*\top}\hB - \frac{\eta}{n}\hB_\bot^{*\top}\pr{\sumi\pr{\frac{1}{m}\X_i^\top \X_i \pr{\hB \w^+_i - \hB^* \w^*_i}- \pr{\hB \w^+_i - \hB^* \w^*_i}}\w_i^{+\top}} \nonumber\\
    &\quad - \frac{\eta}{n}\sumi\pr{\hB_\bot^{*\top}\hB \w^+_i - \hB_\bot^{*\top}\hB^* \w^*_i}\w_i^{+\top} + \frac{\eta}{n}\hB_\bot^{*\top}\sumi \frac{1}{m}\X_i^\top \Z_i \w_i^{+\top}\\
    &= \hB_\bot^{*\top}\hB \pr{\bI_k - \frac{\eta}{n}\sumi\w_i^+\w_i^{+\top}} +\frac{\eta}{n}\hB_\bot^{*\top}\sumi \frac{1}{m}\X_i^\top \Z_i \w_i^{+\top}  \nonumber\\
    &\quad - \frac{\eta}{n}\hB_\bot^{*\top}\pr{\sumi\pr{\frac{1}{m}\X_i^\top \X_i \pr{\hB \w^+_i - \hB^* \w^*_i}- \pr{\hB \w^+_i - \hB^* \w^*_i}}\w_i^{+\top}}
\end{align}
where the second equality holds since $\hB_{\bot}^{*\top}\hB^*= 0$. Recall that from the \textit{QR} decomposition of $\B^+$we have $\B^+ = \hB^+ \R^+$. Hence multiplying by $\pr{\R^+}^{-1}$ and taking both sides the norm we derive
\begin{align}
    \dist\pr{\hB^+, \hB^*} &\leq \normt{\hB_\bot^{*\top}\hB \pr{\bI_k - \frac{\eta}{n}\sumi\w_i^+\w_i^{+\top}}}\normt{\pr{\R^+}^{-1}} + \normt{\frac{\eta}{n}\hB_\bot^{*\top}\sumi \frac{1}{m}\X_i^\top \Z_i \w_i^{+\top}} \normt{\pr{\R^+}^{-1}}\nonumber\\
    &\quad + \normt{\frac{\eta}{n}\hB_\bot^{*\top}\pr{\sumi\pr{\frac{1}{m}\X_i^\top \X_i \pr{\hB \w^+_i - \hB^* \w^*_i}- \pr{\hB \w^+_i - \hB^* \w^*_i}}\w_i^{+\top}}}\normt{\pr{\R^+}^{-1}} 
\end{align}
Let us define 
\begin{align}
    A_1&:=\distB \normt{\bI_k - \frac{\eta}{n}\sumi\w_i^+\w_i^{+\top}}\\
    A_2&:= \normt{\frac{\eta}{n}\hB_\bot^{*\top}\sumi \frac{1}{m}\X_i^\top \Z_i \w_i^{+\top}}\\
    A_3&:=\normt{\frac{\eta}{n}\hB_\bot^{*\top}\pr{\sumi\pr{\frac{1}{m}\X_i^\top \X_i \pr{\hB \w^+_i - \hB^* \w^*_i}- \pr{\hB \w^+_i - \hB^* \w^*_i}}\w_i^{+\top}}}
\end{align}
so that the following inequality holds
\begin{align}
   \dist\pr{\hB^+, \hB^*} &\leq \pr{A_1+A_2+A_3} \normt{\pr{\R^+}^{-1}}\label{distancesum}
\end{align}
For the rest of the proof we will work conditioning on the intersection of the events 
\begin{align}
 \cE_2&:= \bigcap_{i=1}^n\ag{\normt{\w_i^+} \leq 2\sqrt{k}+\hd \var^2 \quad \bigcap \quad \normt{\q_i}\leq \dist\pr{\hB, \hB^*}2\sqrt{k}+\hd \var^2}\label{Etwo}\\
 \cE_3 &:=\ag{\norm{\F}_F \leq  \dist\pr{\hB, \hB^*} \hd\norm{\W^*}_2\quad \bigcap \quad \norm{\G}_F \leq \hd \sqrt{n}\var^2}  \label{Ethree}\\
 \cE_4 &:= \ag{\normt{\frac{1}{mn} \sum_{i=1}^n \X_i^\top \Z_i\w^{+ \top}_i} \leq c\cdot \frac{\var^2\pr{\sqrt{k}+\hd \var^2}\sqrt{d}}{\sqrt{mn}}}\label{Efour} \\
 \cE_5  &:= \hspace{-0.02in} \ag{\hspace{-0.02in}\normt{\frac{1}{n}\sum_{i=1}^n \pr{\frac{1}{m}\X_i^\top \X_i \pr{\hB\w_i^+ - \hB^* \w^*_i}- \pr{\hB\w_i^+ - \hB^* \w^*_i}}\w^{+\top}_i} \hspace{-0.075in} \leq \hspace{-0.025in}\frac{c\sqrt{ d}\pr{ \dist\pr{\hB, \hB^*}k + \sqrt{k}\hd\var^2 + \hd^2\var^4}}{\sqrt{mn}}\hspace{-0.05in}}\label{Efive}
\end{align}
which happens with probability at least $1-\exp\pr{-90d}-\exp\pr{-90k^2\log(N)}$ by Union Bound on the failure probability of \eqref{matrixFlemma}, \eqref{matrixGlemma}, \eqref{noiseterm}, \eqref{indicator} and \eqref{Etwotime}. 

We will now provide bounds for each of the  terms of interest in \eqref{distancesum}, starting from $A_1$. Notice that by \eqref{matrixversion} we have
\begin{align}
 \lambda_{\max}\pr{\W^+ \W^+} = \normt{\W^+}^2 &= \normt{\W^*\hB^* \hB + \F+\G}^2\\
 &\leq 2\normt{\W^*}^2 + 4\normt{\F}^2+ 4\normt{\G}^2\\
 &\leq 2\normt{\W^*}^2 + \distB 4 \hd^2 \normt{\W^*}^2 + 4\hd^2 \var^4 n\\
 &\leq 4\pr{\normt{\W^*}^2 + n}\\
 & \leq 4n \pr{\ssigmax^2 +1}
\end{align}
where in the last inequality we use the fact that $\normt{\W^*}=\sqrt{n}\cdot\ssigmax$. Since $\eta < \pr{\ssigmax^2 +1}^{-1}$ the matrix $\bI_k - \frac{\eta}{n}\W^{+\top}\W^{+}$ is positive definite. Thus we have
\begin{align}
    \normt{\bI_k - \frac{\eta}{n}\W^{+\top}\W^{+}}&\leq 1-\frac{\eta}{n}\lambda_{\min}\pr{\W^{+\top}\W^+}\\
    &\leq 1-\frac{\eta}{n} \lambda_{\min}\pr{\pr{\W^*\hB^* \hB + \F+\G}^\top \pr{\W^*\hB^* \hB + \F+\G}}\\
    &\leq 1- \frac{\eta}{n}\pr{\sigmin^2\pr{\W^*\hB^*\hB}- \sigmin^2(\F)- \sigmin^2(\G)}\nonumber\\
    &\quad +\frac{2\eta}{n}\pr{\sigmax\pr{\F^\top \W^*\hB^{*\top}\hB} + \sigmax\hspace{-0.04in}\pr{\F^\top\G}+ \sigmax\hspace{-0.04in}\pr{\G^\top\W^*\hB^{*\top}\hB}\hspace{-0.02in}}\\
    &\leq 1- \frac{\eta}{n}\pr{\sigmin^2 \pr{\W^*}\sigmin^2\pr{\hB^{*\top}\hB}+2\normt{\hB^{*\top}\hB}\pr{\sigmax\pr{\F^\top \W^*}+\sigmax\pr{\G^\top \W^*}}}\nonumber\\
    &\quad + \frac{2\eta}{n}\sigmax(\F)\sigmax(\G)\\
    &\leq 1- \eta \cdot \ssigmin^2 \cdot \sigmin^2\pr{\hB^{*\top}\hB} + \frac{2\eta}{n}\pr{\normt{\F}+\normt{\G}}\normt{\W^*} + \frac{2\eta}{n}\pr{\normt{\F}\cdot \normt{\G}}
\end{align}
where we used that the norms of $\hB^*$ and $\hB$ are $1$ since the matrices are orthonormal and $\ssigmin~\leq~\sigmin(\W^*)$. Recall that we operate under $\cE_3$ and thus we can further write 
\begin{align}
    \normt{\bI_k - \frac{\eta}{n}\W^{+\top}\W^{+}}&\leq 1- \eta \cdot \ssigmin^2 \cdot \sigmin^2\pr{\hB^{*\top}\hB} + \frac{2\eta}{n}\pr{\distB \hd \normt{\W^*} + \sqrt{n} \hd \var^2}\normt{\W^*} \nonumber\\
    & \quad +\frac{2\eta}{n}\pr{\distB \hd^2 \var^2\sqrt{n}\normt{\W^*}}\\
    & \leq 1- \eta \cdot \ssigmin^2 \cdot \sigmin^2\pr{\hB^{*\top}\hB} +2\eta\pr{\hd \frac{\normt{\W^*}^2}{n} + \hd \var^2 \frac{\normt{\W^*}}{\sqrt{n}}+\hd^2 \var^4\frac{\norm{\W^*}}{\sqrt{n}}}\\
       & \leq 1- \eta \cdot \ssigmin^2 \cdot \sigmin^2\pr{\hB^{*\top}\hB} +3\eta\pr{\frac{E_0 \ssigmin^2}{20 \ssigmax^2} \cdot \ssigmax^2 } \\
       &\leq 1- \eta \cdot \ssigmin^2 \cdot \sigmin^2\pr{\hB^{*\top}\hB} +\frac{1}{6}\eta E_0 \ssigmin^2
\end{align}
where we upper bound $\distB$ by $1$, $\hd \leq  \frac{E_0}{20 \cdot \kappa^2 } \cdot \frac{1}{1+\var^2}$ and use Assumption \ref{Assum2} in the third inequality. Further by the definition of $E_0:=1-\dist^2\pr{\hB^0, \hB^*}\leq \sigmin^2\pr{\hB^{*\top}, \hB}$ we have
\begin{align}
  \normt{\bI_k - \frac{\eta}{n}\W^{+\top}\W^{+}}&\leq 1 - \eta E_0 \ssigmin^2   + \frac{1}{6}\eta E_0 \ssigmin^2
\end{align}
and it follows immediately that
\begin{align}
    A_1 \leq \distB \pr{1 - \eta E_0 \ssigmin^2   +\frac{1}{6}\eta E_0 \ssigmin^2}\label{Aone}
\end{align}
Further since we operate under $\cE_4$ \eqref{Efour} and $\normt{\hB^*_\bot}=1$ we have 
\begin{align}
    A_2 \leq \eta c\var^2\pr{\sqrt{k}+1}\frac{\sqrt{d}}{\sqrt{mn}} \label{Atwo}
\end{align}
and since we operate under $\cE_5$ \eqref{Efive} we obtain
\begin{align}
    A_3 \leq \eta c\pr{\distB k+ \sqrt{k}+1} \frac{\sqrt{d}}{\sqrt{mn}}\label{Athree}
\end{align}
Combining \eqref{distancesum}, \eqref{Aone}, \eqref{Atwo} and  \eqref{Athree} we get
\begin{align}
     \dist\pr{\hB^+, \hB^*} &\leq \distB \pr{1 - \frac{5}{6}\eta E_0 \ssigmin^2   +  \eta c k\frac{\sqrt{d}}{\sqrt{mn}}} \cdot \normt{\pr{\R^+}^{-1}} \nonumber\\
     &\quad + \eta c\pr{\sqrt{k}+1}\pr{\var^2 +1}\frac{\sqrt{d}}{\sqrt{mn}} \cdot \normt{\pr{\R^+}^{-1}}\label{distance}
\end{align}
The last part of the proof focuses on bounding $\normt{\pr{\R^+}^{-1}}$.

Let us define
\begin{align}
    \bS&:= \sumi \frac{1}{m} \X_i^{\top}\X_i\pr{\hB\w_i^+- \hB^* \w_i^{*}}\w_i^{+\top}\\
    \E&:= \sumi \frac{1}{m}\X_i^\top \Z_i \w_i^{+\top}
\end{align}
and hence \eqref{Bupdatetwo} takes the form 
 \begin{align}
     \B^+ = \hB - \frac{\eta}{n}\bS + \frac{\eta}{n}\E
 \end{align}
and also
\begin{align}
    \B^{+\top} \B^+ &= \hB^\top \hB - \frac{\eta}{n}\pr{\hB^\top \bS + \bS^\top \hB}+ \frac{\eta}{n}\pr{\hB^\top \E + \E^\top \hB}+ \frac{\eta^2}{n^2}\bS^\top \bS - \frac{\eta^2}{n^2} \pr{\E^\top\bS+ \bS^\top \E} + \frac{\eta^2}{n^2}\E^\top \E\\
    &= \bI_k - \frac{\eta}{n}\pr{\hB^\top \bS + \bS^\top \hB}+ \frac{\eta}{n}\pr{\hB^\top \E + \E^\top \hB} - \frac{\eta^2}{n^2} \pr{\E^\top\bS+ \bS^\top \E} + \frac{\eta^2}{n^2}\E^\top \E +\frac{\eta^2}{n^2}\bS^\top \bS
\end{align} 
By Weyl's inequality and since $\R^{+\top}\R^{+}= \hB^{+\top}\hB^+$  we derive
\begin{align}
    \sigmin^2\pr{\R^+}\geq 1- \frac{\eta}{n}\lambda_{\max}\pr{\hB^\top \bS + \bS^\top \hB} - \frac{\eta}{n}\lambda_{\max}\pr{\hB^\top \E + \E^\top \hB} - \frac{\eta^2}{\n^2}\lambda_{\max}\pr{\E^\top\bS+ \bS^\top \E}
\end{align}
Let us further define 
\begin{align}
    R_1&:= \frac{\eta}{n} \lambda_{\max}\pr{\hB^\top \bS + \bS^\top \hB}\\
    R_2&:= \frac{\eta^2}{\n^2}\lambda_{\max}\pr{\E^\top\bS+ \bS^\top \E}\\
    R_3&:= \frac{\eta}{n}\lambda_{\max}\pr{\hB^\top \E + \E^\top \hB}
\end{align}
So that we can succinctly rewrite the above inequality as follows
\begin{align}
 \sigmin^2\pr{\R^+}\geq 1- R_1 - R_2 - R_3\label{Rinequality}   
\end{align}
We work to bound separately each of the three terms.
\begin{align}
    R_1 &= \frac{2\eta}{n} \maxp  \p^\top \hB^\top \bS \p\\
    &= \maxp  \frac{2\eta}{n}\p^\top \hB^\top \br{\pr{\sumi\pr{\frac{1}{m}\X_i^\top \X_i \pr{\hB \w^+_i - \hB^* \w^*_i}- \pr{\hB \w^+_i - \hB^* \w^*_i}}\w_i^{+\top}}}\p \nonumber \\
    &\quad + \maxp \frac{2\eta}{n}\p^\top \hB^\top \br{\sumi \pr{\hB \w^+_i - \hB^* \w^*_i}\w_i^{+\top}}\p
\end{align}
and since we operate under $\cE_5$ \eqref{Efive} the above simplifies to
\begin{align}
 R_1 & \leq 2\eta \normt{\hB} c \pr{\distB k +  \sqrt{k}+1}\frac{\sqrt{d}}{\sqrt{mn}}    +\maxp \frac{2\eta}{n}\p^\top \hB^\top \br{\sumi \pr{\hB \w^+_i - \hB^* \w^*_i}\w_i^{+\top}}\p\\
 &\leq 3\eta c \pr{\distB k +  \sqrt{k}}\frac{\sqrt{d}}{\sqrt{mn}}    +\maxp \frac{2\eta}{n}\p^\top \hB^\top \br{\sumi \pr{\hB \w^+_i - \hB^* \w^*_i}\w_i^{+\top}}\p\label{R_one}
\end{align}
We focus on the second term and using \eqref{w+lemma} we get
\begin{align}
  \frac{2\eta}{n}\p^\top \hB^\top \br{\sumi \pr{\hB \w^+_i - \hB^* \w^*_i}\w_i^{+\top}}\p &= \frac{2\eta}{n}\cdot \tr{\sumi \pr{\hB \w^+_i - \hB^* \w^*_i}\w_i^{+\top}\p \p^\top \hB^\top} \\
  &= \frac{2\eta}{n}\cdot \tr{\sumi \pr{\hB \w^+_i - \hB^* \w^*_i}\pr{\hB^\top \hB^*\w_i^{*} +\F_i+\G_i }^{\top}\p \p^\top \hB^\top}\label{ineq}
\end{align}
We bound each term separately and to this end we define 
\begin{align}
    T_1 &:= \frac{2\eta}{n}\cdot \tr{\sumi \pr{\hB \w^+_i - \hB^* \w^*_i}\w_i^{*\top}\hB^{*\top}\hB\p \p^\top \hB^\top}\\
    T_2 &:= \frac{2\eta}{n}\cdot \tr{\sumi \pr{\hB \w^+_i - \hB^* \w^*_i}\F_i^{\top}\p \p^\top \hB^\top}\\
    T_3 &:=\frac{2\eta}{n}\cdot \tr{\sumi \pr{\hB \w^+_i - \hB^* \w^*_i}\G_i^{\top}\p \p^\top \hB^\top}
\end{align}
such that \eqref{ineq} can be expressed as
\begin{align}
    \frac{2\eta}{n}\p^\top \hB^\top \br{\sumi \pr{\hB \w^+_i - \hB^* \w^*_i}\w_i^{+\top}}\p &= T_1+ T_2 +T_3 \label{R_one_sec_term}
\end{align}
Further expanding $T_1$ we have
\begin{align}
    T_1 &=\frac{2\eta}{n} \tr{\sumi \pr{\hB\hB^\top\hB^* \w_i^* \w_i^{*\top} + \hB \F_i \w_i^{*\top} + \hB \G_i\w_i^{*\top}- \hB^* \w_i^* \w_i^{*\top}}\hB^{*\top}\hB\p \p^\top \hB^\top}\\
    &= \frac{2\eta}{n} \tr{\hB^\top\pr{\hB\hB^\top - \bI_d}\sumi\pr{\hB^{*\top}\w_i^* \w_i^{*\top}} \hB^{*\top}\hB\p \p^\top } + \frac{2\eta}{n} \tr{\sumi \pr{\hB \F_i \w_i^{*\top}}\hB^{*\top}\hB\p \p^\top \hB^\top} \nonumber \\
    &\quad + \frac{2\eta}{n} \tr{\sumi \pr{\hB \G_i \w_i^{*\top}}\hB^{*\top}\hB\p \p^\top \hB^\top} \\
    &= \frac{2\eta}{n}\tr{\hB^\top \hB \pr{\F^\top + \G^\top}\W^* \hB^{*\top}\hB\p \p^\top}\\
    &\leq \frac{2\eta}{n}\pr{\normf{\F}+\normf{\G}}\normt{\W^*}
\end{align}
where in the first equality we expand $w^{+}_i$ via \eqref{w+lemma} and in the third equality we use that $\hB^\top\hspace{-0.05in} \pr{\hB\hB^\top\hspace{-0.05in}- \bI_d}=~0 $ and $\F^\top = \sum\limits_{i=1}^n \F_i \w_i^{*\top}$, $\G^\top = \sum\limits_{i=1}^n \G_i \w_i^{*\top}$. The inequality is obtained by noticing that the norms of the orthonormal $\hB, \hB^*$ is one and also $\normt{\p \p^\top}\leq 1$. Conditioning on $\cE_3$  \eqref{Ethree} we can further simplify as follows
\begin{align}
    T_1 &\leq \frac{2\eta}{n} \pr{\distB \hd \normt{\W^*}^2 + \hd \var^2 \sqrt{n}\normt{\W}}\\
    &\leq \frac{1}{10} \eta E_0 \ssigmin^2  \label{Tone}
\end{align}
We now turn our attention to $T_2$
\begin{align}
    T_2 &= \frac{2\eta}{n}\cdot \tr{\sumi \pr{\hB \hB^\top \hB^* \w^*_i + \hB \F_i + \hB \G_i - \hB^* \w^*_i}\F_i^{\top}\p \p^\top \hB^\top}\\
    &= \frac{2\eta}{n} \tr{\hB^\top \pr{\hB \hB^\top - \bI_d} \sumi\pr{ \hB^* \w_i^*} \F_i^\top \p \p^\top} + \frac{2\eta}{n} \tr{\hB^\top \hB \sumi \F_i \F_i^\top \p \p^\top }  
    + \frac{2\eta}{n} \tr{\hB^\top \hB \sumi \G_i \F_i^\top \p \p^\top } \\
    &= \frac{2\eta}{n} \tr{\hB^\top \hB \pr{\F^\top \F + \G^\top \F}\p \p^\top}\\\
    &\leq \frac{2\eta}{n}\pr{\normf{\F}^2 +\normf{\G}\normf{\F}}
\end{align}
where in the third equality we used that $\hB^\top \pr{\hB \hB^\top - \bI_d}=0$ and in the forth that the norms of the orthonormal matrices is $1$ as well as the norm of $\p\p^\top$. Following the same calculations for $T_3$ we get
\begin{align}
    T_3 \leq \frac{2\eta}{n}\pr{\normf{\G}^2 +\normf{\G}\normf{\F}}
\end{align}
and thus summing the two terms we get the following
\begin{align}
    T_2 + T_3 \leq \frac{2\eta}{n}\pr{\normf{\F} + \normf{\G}}^2
\end{align}
Again conditioning on $\cE_3$  \eqref{Ethree} we derive
\begin{align}
    T_2 + T_3 &\leq \frac{2\eta}{n} 
    \pr{\hd \pr{\normt{\W^*} + \sqrt{n} \var^2}}^2
    \leq 2\eta \hd^2 \pr{\ssigmax^2 + \var^4 }
    \leq \frac{1}{10}\eta E_0 \ssigmin^2 \label{Ttwothree}
\end{align}
Hence combining \eqref{R_one}, \eqref{Tone} and \eqref{Ttwothree} we get a bound for $R_1$
\begin{align}
    R_1 &\leq 3\eta c \pr{k +  \sqrt{k}}\frac{\sqrt{d}}{\sqrt{mn}} + \frac{1}{5}\eta E_0 \ssigmin^2\\
    &\leq 6\eta\cdot c\cdot k \frac{\sqrt{d}}{\sqrt{mn}} + \frac{1}{5}\eta E_0 \ssigmin^2\label{OneR}
\end{align}
We work in similar fashion to derive the bound on $R_2$
\begin{align}
    R_2 &= \frac{\eta^2}{n^2}\lambda_{\max}\pr{\E^\top\bS+ \bS^\top \E}  \\
    &\leq \frac{2\eta^2}{n^2} \maxp\p^\top \br{ \sumi \pr{\frac{1}{m}\X_i^\top \X_i\pr{\hB \w_i^+ - \hB^* \w^*_i}-\pr{\hB \w_i^+ - \hB^* \w^*_i}}\w_i^{+\top}} \sumi\pr{\frac{1}{m}\X_i^\top \Z_i \w_i^{+ \top}}\p \nonumber\\
    &\quad +\frac{2\eta^2}{n^2} \maxp \p^\top \br{ \sumi \pr{\hB \w_i^+ - \hB^* \w^*_i}\w_i^{+\top}} \sumi\pr{\frac{1}{m}\X_i^\top \Z_i \w_i^{+ \top}}\p\\
    &\leq 2\eta^2 \normt{\frac{1}{n}\sumi \pr{\frac{1}{m}\X_i^\top \X_i\pr{\hB \w_i^+ - \hB^* \w^*_i}-\pr{\hB \w_i^+ - \hB^* \w^*_i}}\w_i^{+\top}} \cdot \normt{\frac{1}{mn} \sumi\pr{\X_i^\top \Z_i \w_i^{+ \top}}}\nonumber \\
    &\quad  +\frac{2\eta^2}{n} \normt{\sumi \pr{\hB \w_i^+ -\hB^*\w^*_i}\w_i^{+\top}} \cdot \normt{\frac{1}{mn} \sumi\pr{\X_i^\top \Z_i \w_i^{+ \top}}}
\end{align}
Since we work conditioning on the event $\cE_4 \bigcap \cE_5$ we further derive
\begin{align}
    R_2 & \leq 2\eta^2\pr{ c \frac{\sqrt{ d}\pr{ \dist\pr{\hB, \hB^*}k + \sqrt{k}\hd\var^2 + \pr{\hd\var^2}^2}}{\sqrt{mn}}} \cdot  \pr{c \frac{\var^2\pr{\sqrt{k}+\hd \var^2}\sqrt{d}}{\sqrt{mn}}} \nonumber \\ 
   & \quad +\frac{2\eta^2}{n} \normt{\sumi \pr{\hB \w_i^+ -\hB^*\w^*_i}\w_i^{+\top}} \cdot \pr{c \frac{\var^2\pr{\sqrt{k}+\hd \var^2}\sqrt{d}}{\sqrt{mn}}}\\
   &\leq 3\eta^2 \pr{ck\frac{\sqrt{d}}{\sqrt{mn}}}\pr{c\sqrt{k}\var^2 \frac{\sqrt{d}}{\sqrt{mn}}} +\frac{2\eta^2}{n} \sumi\normt{ \q_i}\normt{\w_i^+}\pr{c\sqrt{k}\var^2 \frac{\sqrt{d}}{\sqrt{mn}}}
\end{align}
And since we also condition on $\cE_2$ \eqref{Etwo} we finally get
\begin{align}
    R_2 &\leq 3\eta^2 c^2 k^{\frac{3}{2}}\var^2\frac{d}{mn} + \frac{2 \eta^2}{n} \sumi \pr{2 \sqrt{k} + \hd \var^2}^2 \pr{c\sqrt{k}\var^2 \frac{\sqrt{d}}{\sqrt{mn}}}\\
    &\leq 3\eta^2 c^2 k^{\frac{3}{2}}\var^2\frac{d}{mn} + 9\eta^2 \pr{ck^{\frac{3}{2}}\var^2 \frac{\sqrt{d}}{\sqrt{mn}}} \label{Rtwo}
\end{align}
The last term we need to bound is $R_3$
\begin{align}
     R_3 &= \frac{\eta}{n}\lambda_{\max}\pr{\E^\top\hB+ \hB^\top \E}  \\
    &= \frac{2\eta^2}{\n^2} \maxp \p^\top \hB^\top \E \p\\
    &\leq \frac{2\eta}{n}\normt{\hB} \normt{\sumi \frac{1}{m} \X_i^\top \Z_i \w_i^{+\top}}\\
    & \leq 3 \eta c\sqrt{k}\var^2\frac{\sqrt{d}}{\sqrt{mn}}\label{Rthree}
\end{align}
Combining \eqref{Rinequality} with \eqref{OneR}, \eqref{Rtwo} and \eqref{Rthree} we derive
\begin{align}
    \sigmin^2\pr{\R^+}&\geq 1- 6\eta\cdot c\cdot k \frac{\sqrt{d}}{\sqrt{mn}} - \frac{1}{5}\eta E_0 \ssigmin^2 - 3\eta^2 c^2 k^{\frac{3}{2}}\var^2\frac{d}{mn} - 9\eta^2 \pr{ck^{\frac{3}{2}}\var^2 \frac{\sqrt{d}}{\sqrt{mn}}} - 3 \eta c\sqrt{k}\var^2\frac{\sqrt{d}}{\sqrt{mn}}\nonumber \\
    &\geq 1- 14 \eta c k^{\frac{3}{2}} \var^2 \frac{\sqrt{d}}{\sqrt{mn}}- 3\eta^2 c^2 k^{\frac{3}{2}}\var^2\frac{d}{mn} - \frac{1}{5}\eta E_0 \ssigmin^2\\
    & \geq 1- 15 \eta c k^{\frac{3}{2}} \var^2 \frac{\sqrt{d}}{\sqrt{mn}}- \frac{1}{5}\eta E_0 \ssigmin^2\label{R}
\end{align}
where the last inequality holds since $\sqrt{mn}\geq c\sqrt{d}$. We can now combine \eqref{distance} and \eqref{R} to obtain the contraction inequality  
\vspace{-0.05in}
    \begin{align}
     \dist\pr{\hB^+, \hB^*} &\leq \distB \pr{1 - \frac{5}{6}\eta E_0 \ssigmin^2   +  \eta c k\frac{\sqrt{d}}{\sqrt{mn}}} \cdot \normt{\pr{1- 15 \eta c k^{\frac{3}{2}} \var^2 \frac{\sqrt{d}}{\sqrt{mn}}- \frac{1}{5}\eta E_0 \ssigmin^2}^{-\frac{1}{2}}} \nonumber\\
     &\quad + \eta c\pr{\sqrt{k}+1}\pr{\var^2 +1}\frac{\sqrt{d}}{\sqrt{mn}} \normt{\pr{1- 15 \eta c k^{\frac{3}{2}} \var^2 \frac{\sqrt{d}}{\sqrt{mn}}- \frac{1}{5}\eta E_0 \ssigmin^2}^{-\frac{1}{2}}}
\end{align}
We divide and multiply by $n_0$ and using our bounds on $m$ and $n$ this expression further simplifies,
\vspace{-0.05in}
\begin{align}
    \dist\pr{\hB^+, \hB^*} &\leq \distB \pr{1 - \frac{5}{6}\eta E_0 \ssigmin^2   +  \eta c k\frac{\sqrt{d}}{\sqrt{mn_0 \cdot \frac{n}{n_0}}}}  \pr{1- 15 \eta c k^{\frac{3}{2}} \var^2 \frac{\sqrt{d}}{\sqrt{mn_0 \cdot \frac{n}{n_0}}}- \frac{1}{5}\eta E_0 \ssigmin^2}^{-\frac{1}{2}} \nonumber\\
     &\quad + \eta c\pr{\sqrt{k}+1}\pr{\var^2 +1}\frac{\sqrt{d}}{\sqrt{mn_0 \cdot \frac{n}{n_0}}}  \pr{1- 15 \eta c k^{\frac{3}{2}} \var^2 \frac{\sqrt{d}}{\sqrt{mn_0 \cdot \frac{n}{n_0}}}- \frac{1}{5}\eta E_0 \ssigmin^2}^{-\frac{1}{2}}\\
     & \leq \distB \pr{1-\frac{1}{2}\eta E_0 \ssigmin^2} \pr{1-\frac{1}{2}\eta E_0 \ssigmin^2}^{-\frac{1}{2}} \hspace{-0.1in}+ \pr{\sqrt{\frac{n_0}{4n}}\eta E_0 \ssigmin^2}\pr{1-\frac{1}{2}\eta E_0 \ssigmin^2}^{-\frac{1}{2}}\\
     & \leq \distB \sqrt{\pr{1-\frac{1}{2}\eta E_0 \ssigmin^2}} + \frac{\pr{\frac{1}{2}\eta E_0 \ssigmin^2}}{\sqrt{\frac{n}{n_0}\pr{1-\frac{1}{2}\eta E_0 \ssigmin^2}}} \label{finalcont}
\end{align}
where in the second inequality we used that for our choices of $m$ and $n$ the following inequality holds $15 \eta c k^\frac{3}{2}(1+\sigma^2) \frac{\sqrt{d}}{\sqrt{mn_0} } \leq \frac{1}{10}\eta E_0 \ssigmin^2$.
Taking Union Bound over the total number of iterations $T$ we derive the result.
\end{proof}
\begin{corollary}\label{corollaryone}
Recall that our algorithm starts at stage $0$ with $n_0$ participating clients and doubles the number of participating clients at every subsequent stage. Thus, by slightly abusing notation, we can reformulate the contraction inequality of \cref{convergence} at stage $r$ as follows
\begin{align}
\dist^+& \leq \dist \sqrt{1-a}+\frac{a}{\sqrt{2^r(1-a)}} \quad \textit{with} \quad a \leq \frac{1}{4}
\end{align} 
\end{corollary}

\section{Appendix}
In the second part of our analysis we compute the expected `Wall Clock Time' of our proposed method and compare it to corresponding the `Wall Clock Time' of straggler-prone \texttt{FedRep}. We prove that when the computational speeds are drawn from the exponential distribution with parameter $\lambda$ and the communication cost is given by $\mC = c\frac{1}{\lambda}$, (for some constant $c$), then \texttt{SRPFL} guarantees a logarithmic speedup.  
Recall that in \cref{corollaryone} we get the following simplified version of the contraction inequality
\begin{align}\label{conttwo}
\dist^+& \leq \dist \sqrt{1-a}+\frac{a}{\sqrt{2^r(1-a)}}\quad \textit{with} \quad a \leq \frac{1}{4}. 
\end{align}
For the rest of this section w.l.g.\ we assume that the clients are re-indexed at every stage so that the expected computation times maintain a decreasing ordering i.e. $\forall r \quad \Ex{\mT^r_{1}}\leq \Ex{\mT^r_{2 }}\leq ..., \leq \Ex{\mT^r_{N }}$. For simplicity henceforth we drop the stage index $r$.
Notice that the decreasing ordering of the computation times in combination with \eqref{conttwo} imply that \texttt{SRPFL} initially benefits by including only few fast nodes in the training procedure. However, as the distance diminishes the improvement guaranteed by the contraction inequality becomes negligible and thus our method benefits by including slower nodes, thus decreasing the second term of the r.h.s. of \eqref{conttwo}.

Let us denote by $X_i$ the maximum distance for which the optimal number of participating nodes (for \texttt{SRPFL})  is $n_0 \cdot 2^i$. This definition immediately implies that $X_0=+ \infty$. 
To compute each $X_i$ we turn our attention on measuring the progress per unit of time achieved by \texttt{SRPFL}, when $2^r \cdot n_0$ nodes are utilized. This ratio at stage $r$ can be expressed as
\begin{align} 
\frac{\dist^+ - \dist \sqrt{1-a}-\frac{a}{\sqrt{2^r(1-a)}}}{{\Ex{\mT_{n_02^r}}+\mC}}\label{ratio1}.
\end{align}
Notice that by \eqref{conttwo} the nominator captures the progress per round while the algorithm incurs $\Ex{\mT_{n_02^r}}$ computation and $\mC$ communication cost.
Similarly the ration when $2^{r+1}\cdot n_0$ nodes are used is given by 
\begin{align}
\frac{\dist^+ - \dist \sqrt{1-a}-\frac{a}{\sqrt{2^{r+1}(1-a)}}}{{\Ex{\mT_{n_02^{r+1}}}+\mC}}\label{ratio2}.
\end{align}
Based on the above inequalities we can now compute the optimal doubling points (in terms of distance) and thus the values of $X_i$'s. Subsequently, we compute the number of iterations \texttt{SRPFL} spends in every stage.
\begin{lemma}
For all $i$ let $X_i$ denote the maximum distance for which the optimal number of nodes for \texttt{SRPFL} is $n_0 \cdot 2^i$. Then the following holds
\begin{align}
    \forall i>0  \quad X_i&= \frac{a}{\sqrt{2^r(1-a)}1-\sqrt{1-a})} \pr{1+\frac{(\TT{r}+\mC)\pr{1-\frac{1}{\sqrt{2}}}}{\TT{r+1}-\TT{r}}}\label{Xbound}\\
    X_0&= + \infty \nonumber
\end{align}
Further \texttt{SRPFL} spends at each stage $r$ at most $t^r$ communication rounds such that
\begin{align}
      t^r\geq \frac{2\log\pr{\frac{\sqrt{2}\pr{\TT{r+1}-\TT{r}}}{\TT{r}-\TT{r-1}}}}
    {\log(\frac{1}{1-a})}.
\end{align} 
\end{lemma}
\begin{proof}
For each stage $r$ let us compute the point where the transitioning between $2^r \cdot \n_0$ and $2^{r+1} \cdot \n_0$ occurs. That is the distance at which \texttt{SRPFL} benefits by doubling the number of participation nodes to $2^{r+1}$. Thus equating the two ratios in \eqref{ratio1} and \eqref{ratio2} and solving for $X_{r+1}$ we get
\begin{align}
\frac{X_{r+1}-X_{r+1}\sqrt{1-a}-\frac{a}{\sqrt{2^r(1-a)}}}{\Ex{\mT_{n_02^r}}+\mC}&=
\frac{X_{r+1}-X_{r+1}\sqrt{1-a}-\frac{a}{\sqrt{2^{r+1}(1-a)}}}{\Ex{\mT_{n_02^{r+1}}}+\mC}\\
 X_{r+1}\pr{\Ex{\mT_{n_02^{r+1}}}-\Ex{\mT_{n_02^r}}}&=
\frac{a}{\sqrt{2^r(1-a)}(1-\sqrt{1-a})}\pr{\Ex{\mT_{n_02^{r+1}}}-\frac{1}{\sqrt{2}}\Ex{\mT_{n_02^r}}+\mC\pr{1-\frac{1}{\sqrt{2}}}}\\
 X_{r+1}&=
\frac{a}{\sqrt{2^r(1-a)}(1-\sqrt{1-a})} \pr{1+\frac{(\Ex{\mT_{n_02^r}}+\mC)\pr{1-\frac{1}{\sqrt{2}}}}{\Ex{\mT_{n_0 2^{r+1}}}-\Ex{\mT_{n_02^{r}}}}}
\end{align}
Let us now compute the number of rounds $t^r$ (henceforth denoted by $t$) required in stage $r$. That is the minimum number of iterations that \texttt{SRPFL} needs to decrease the distance from $X_r$ to $X_{r+1}$ using only the $n_02^r$ fastest participating nodes. Thus, starting off at $X_r$ and following \eqref{conttwo} for $t$ rounds we have
\begin{align}
    X_r^t \leq X_r(\sqrt{1-a})^t+\sum_{i=0}^{t-1}\frac{a}{\sqrt{2^r(1-a)}}(\sqrt{1-a})^i
\end{align}
As stated above we want to find the minimum number of rounds such that we reach the next doubling point i.e. we want $t$ large enough such that
\begin{align}
    X_{r+1}&\geq X_r(\sqrt{1-a})^t+\sum_{i=0}^{t-1}\frac{a}{\sqrt{2^r(1-a)}}(\sqrt{1-a})^i\\
    &\geq X_r(\sqrt{1-a})^t+\frac{a}{\sqrt{2^r(1-a)}}\cdot\frac{1-\sqrt{1-a}^t}{1-\sqrt{1-a}}
\end{align}
where in the last inequality we use geometric series properties. We proceed to solve for $t$ by rearranging and using \eqref{Xbound} and the fact that $X_r > \frac{a}{\sqrt{2^r(1-a)}(1-\sqrt{1-a})}$ ,
\\
\begin{align}
    (\sqrt{1-a})^t& \leq  \frac{\sqrt{2^r(1-a)}(1-\sqrt{1-a})X_{r+1}-a}{\sqrt{2^r(1-a)}(1-\sqrt{1-a})X_r-a}\\
        & \leq \frac{\frac{\pr{\Ex{\mT_{n_02^r}}+\mC}\pr{1-\frac{1}{\sqrt{2}}}}{\Ex{\mT_{n_02^{r+1}}}-\Ex{\mT_{n_02^{r}}}}}
    {\sqrt{2}\pr{\frac{(\Ex{\mT_{n_02^{r-1}}}+\mC)\pr{1-\frac{1}{\sqrt{2}}}}{\Ex{\mT_{n_02^{r}}}-\Ex{\mT_{n_02^{r-1}}}}+1-\frac{1}{\sqrt{2}}}}\\
    & \leq
    \frac{\Ex{\mT_{n_02^{r}}}-\Ex{\mT_{n_02^{r-1}}}}{\Ex{\mT_{n_02^{r+1}}}-\Ex{\mT_{n_02^{r}}}}\\
\end{align}
taking the logarithm on both sides we derive the required amount of rounds
\begin{align}\label{timebound}
    t\geq \frac{2\log\pr{\frac{\sqrt{2}\pr{\Ex{\mT_{n_02^{r+1}}}-\Ex{\mT_{n_02^{r}}}}}{\Ex{\mT_{n_02^{r}}}-\Ex{\mT_{n_02^{r-1}}}}}}
    {\log(\frac{1}{1-a})}\\
\end{align}

\end{proof}
The following lemmas compute the `Wall Clock Time' that \texttt{SRPFL} and \texttt{FedRep} require in order to achieve target accuracy $\epsilon$. As discussed in \cref{Theoretical_Results} for fair comparison we consider accuracy of the form 
\begin{align}\label{targetacc}
    \epsilon = \hat{c} \frac{\alpha}{\sqrt{\frac{N}{n_0}\pr{1-\alpha}}\pr{1-\sqrt{1-\alpha}}}, \qquad \textit{with} \qquad \sqrt{2}>\hat{c}>1.
\end{align}
When $\hat{c}$ takes values close to $\sqrt{2}$ we expect \texttt{SRPFL} to vastly outperform \texttt{FedRep} and as $\hat{c}$ takes values close to $1$ the performance gap diminishes.
\begin{lemma}\label{WCTLemma1}
Suppose at each stage the client's computational times are i.i.d. random variables drawn from the exponential distribution with parameter $\lambda$. Further, suppose that the expected communication cost per round is $\mathcal{C}=c\frac{1}{\lambda}$, for some constant $c$. Finally, consider target accuracy $\epsilon$ given in \eqref{accuracy}. Then the expected `Wall Clock Time' for \texttt{SRPFL} is upper bounded as follows 
\begin{align}
   \Ex{T_{\textit{SRPFL}}} \leq \log N \pr{\frac{6(c+1)+4\log (\frac{1}{\hat{c}-1})}{\log (\frac{1}{1-a})}}\frac{1}{\lambda}
\end{align}
\end{lemma}
\begin{proof}
First we upper bound the expected cost suffered by our method until the distance between the current representation and the optimal representation becomes smaller than $X_{\log (\nicefrac{N}{n_0})}$, i.e. the cost corresponding to the first $\log\pr{\frac{N}{2n_0}}$ stages of \texttt{SRPFL} (denoted by $\Ex{T_{\textit{SRPFL}}}$) .

    \begin{align}
        \Ex{T^1_{\textit{SRPFL}}}  &= \sum_{i=1}^{\log (\frac{N}{2n_0})} t^i \pr{\TT{i} + \mC}\\
        &\leq \sum_{i=1}^{\log (\frac{N}{2n_0})}2\pr{\Ex{\mT_{n_02^i}} + \mC}\cdot \frac{\log\pr{\frac{\sqrt{2}\pr{\Ex{\mT_{n_02^{i+1}}}-\Ex{\mT_{n_02^{i}}}}}{\Ex{\mT_{n_02^{i}}}-\Ex{\mT_{n_02^{i-1}}}}}}
    {\log(\frac{1}{1-a})}\\
    &\leq 
    \sum_{i=1}^{\log (\frac{N}{4n_0})}2\pr{\Ex{\mT_{n_0 \cdot 2^i}}+\mC}\frac{\log\pr{\frac{\sqrt{2}(\Ex{\mT_{\frac{N}{2}}}-\Ex{\mT_{\frac{N}{4}}})}{\Ex{\mT_{\frac{N}{4}}}-\Ex{\mT_{\frac{N}{8}}}}}}
    {\log(\frac{1}{1-a})}\nonumber \\
    &\quad +
    2(\Ex{\mT_{\nicefrac{N}{2}}}+\mC)\frac{\log\pr{\frac{\sqrt{2}(\Ex{\mT_{N}}-\Ex{\mT_{\frac{N}{2}}})}{\Ex{\mT_{\frac{N}{2}}}-\Ex{\mT_{\frac{N}{4}}}}}}
    {\log(\frac{1}{1-a})},\label{midway}
    \end{align}
where we used \ref{timebound}.
Since the computational times of the clients come from the exponential distribution it is straightforward to derive the following bounds
\begin{align}
    &\Ex{\mT_N}-\Ex{\mT_{\nicefrac{N}{2}}}=\frac{1}{\lambda}\sum_{i=1}^{\nicefrac{N}{2}}\frac{1}{i}
    \leq \frac{1}{\lambda}\pr{\ln(\nicefrac{N}{2})+1} \leq \frac{1}{\lambda}\log (N)\\
     &\Ex{\mT_{\nicefrac{N}{2}}}-\Ex{\mT_{\nicefrac{N}{4}}}=\frac{1}{\lambda}\pr{\frac{1}{\nicefrac{N}{2}+1}+\frac{1}{\nicefrac{N}{2}+2}+...+\frac{1}{\nicefrac{3N}{4}}}\geq \frac{1}{\lambda}\cdot\frac{N}{4}\cdot\frac{4}{3N}= \frac{1}{3\lambda}\\
     &\Ex{\mT_{\nicefrac{N}{2}}}-\Ex{\mT_{\nicefrac{N}{4}}} \leq \frac{1}{\lambda}\cdot\frac{N}{4}\cdot\frac{2}{N}= \frac{1}{2\lambda}\\
      &\Ex{\mT_{\nicefrac{N}{4}}}-\Ex{\mT_{\nicefrac{N}{8}}}=
      \frac{1}{\lambda}\pr{\frac{1}{\nicefrac{3N}{4}+1}+\frac{1}{\nicefrac{3N}{4}+2}+...+\frac{1}{\nicefrac{7N}{8}}}\geq\frac{1}{\lambda}\cdot\frac{N}{8}\cdot\frac{4}{3N}= \frac{1}{6\lambda}
\end{align}
Making use of the above bounds the expression in \ref{midway} further simplifies to
\begin{align}
    &\leq 
    2\sum_{i=1}^{\log \pr{\frac{N}{4n_0}}}
    \br{\pr{\Ex{\mT_{n_0 \cdot 2^i}}+\mC}\frac{\log (3\sqrt{2})}{\log (\frac{1}{1-a})}}
    +
    2\pr{\Ex{\mT_{\nicefrac{N}{2}}}+ \mC} \frac{\log (3\sqrt{2}\log(N))}{\log (\frac{1}{1-a})}\\
    &\leq
    \frac{5}{\log (\frac{1}{1-a})}\pr{\sum_{i=1}^{\log \pr{\frac{N}{4n_0}}}\Ex{\mT_{n_0 \cdot 2^i}}+\log\pr{\frac{N}{2n_0}}\cdot \mC}
    +
    2\frac{\log (3\sqrt{2}\log(N))}{\log (\frac{1}{1-a})}(\Ex{\mT_{\nicefrac{N}{2}}}+\mC)\label{beforesubstitution}
\end{align}
Further, notice that
\begin{align}\label{boundhalf}
 \Ex{\mT_{\nicefrac{N}{2}}} = \frac{1}{\lambda}\sum\limits^{\nicefrac{N}{2}}_{i=1} \frac{1}{\nicefrac{N}{2}+i}\leq \frac{1}{\lambda},
\end{align}
and similarly 
$\Ex{\mT_{\nicefrac{N}{4}}} \leq \frac{1}{\lambda}\cdot \frac{1}{3}, \quad \Ex{\mT_{\nicefrac{N}{8}}} \leq \frac{1}{\lambda}\cdot \frac{1}{7}, \quad \Ex{\mT_{\nicefrac{N}{16}}} \leq \frac{1}{\lambda}\cdot \frac{1}{15}$ and so on. Thus,
\begin{align}\label{boundsum}
    \sum_{i=1}^{\log \pr{\frac{N}{4n_0}}}\Ex{\mT_{n_0 \cdot 2^i}} = \Ex{2n_0} + \Ex{4n_0}+...+ \Ex{\nicefrac{N}{4}} \leq \frac{1}{\lambda}\sum^\infty_{i=1}\frac{1}{2^i}\leq \frac{1}{\lambda}
\end{align}
Combining the bounds from \ref{boundhalf} and \ref{boundsum} and substituting $\mC = c\frac{1}{\lambda}$ in expression \ref{beforesubstitution} we derive the following bound
\begin{align}
    \Ex{T^1_{\textit{SRPFL}}}  &\leq \frac{5}{\log\pr{\frac{1}{1-a}}}\pr{c\log(\nicefrac{N}{2n_0})+1}\frac{1}{\lambda} + \frac{2\log(3\sqrt{2}\log(N))}{\log\pr{\frac{1}{1-a}}(c+1)}\frac{1}{\lambda}\\
    &\leq \log(\nicefrac{N}{n_0}) \frac{6(c+1)}{\log\pr{\frac{1}{1-a}}}\cdot \frac{1}{\lambda}\label{costpart1}
\end{align}
Having derived an upper bound on the cost suffered by \texttt{SRPFL} on the first $\log\pr{\frac{N}{2n_0}}$ stages we now turn our attention on bounding the cost incurred from $X_{\log (\frac{N}{n_0})}$ until the target accuracy $\epsilon$ is achieved (denoted by $\Ex{T^2_{\textit{SRPFL}}}$ ). Recall that 
\begin{align}
    X_{\log ( \nicefrac{N}{n_0})}=\frac{a}{\sqrt{\frac{N}{2n_0}(1-a)}(1-\sqrt{1-a})}\pr{1+\frac{(\Ex{\mT_{\nicefrac{N}{2}}}+\mC)(1-\frac{1}{\sqrt{2}})}{\Ex{\mT_{N}}-\Ex{\mT_{\nicefrac{N}{2}}}}}\label{breakpoint}
\end{align}
and further during the last stage of \texttt{SRPFL}, $N$ clients are utilized deriving the following form in the contractions inequality from \eqref{conttwo}
\begin{align}\label{contfinal}
\dist^+& \leq \dist \sqrt{1-a}+\frac{a}{\sqrt{\frac{N}{n_0}(1-a)}}\quad \textit{with} \quad a \leq \frac{1}{4}. 
\end{align}
We first compute the number of rounds required in this second phase of the algorithm. Starting with distance $ X_{\log ( \nicefrac{N}{n_0})}$ and following the contraction in \eqref{contfinal} for $t$ rounds, we derive current distance at most
\begin{align}
    &X_{\log \pr{\frac{N}{n_0}}}\cdot(\sqrt{1-a})^t + \sum_{i=0}^{t-1}\frac{a}{\sqrt{\frac{N}{n_0}(1-a)}}(\sqrt{1-a})^i\\
    &\quad=
     X_{\log \pr{\frac{N}{n_0}}}\cdot(\sqrt{1-a})^t + \frac{a}{\sqrt{\frac{N}{n_0}(1-a)}}\frac{1-\sqrt{1-a}^t}{1-\sqrt{1-a}}\\
     &\quad=
     \frac{a}{\sqrt{\frac{N}{n_0}(1-a)(1-\sqrt{1-a})}}\pr{\sqrt{2}\pr{1+\frac{(\Ex{\mT_{\nicefrac{N}{2}}}+\mC)(1-\frac{1}{\sqrt{2}})}{\Ex{\mT_N}-\Ex{\mT_{\nicefrac{N}{2}}}}}(\sqrt{1-a}^t)+(1-\sqrt{1-a}^t)}\label{midwayn2}
\end{align}
where in the first equality we use geometric series properties and in the second we substitute according to \eqref{breakpoint}. Using the fact that $\sqrt{2}(\Ex{\mT\nicefrac{N}{2}}+\mC)(1-\frac{1}{\sqrt{2}})\leq \frac{1}{\lambda}(c+1)(\sqrt{2}-1)$ and $\Ex{\mT_N}-\Ex{\mT{\nicefrac{N}{2}}}\leq \frac{1}{\lambda}\log N$ expression \eqref{midwayn2} is upper bounded by
\begin{align}
     &\leq
     \frac{a}{\sqrt{\frac{N}{n_0}(1-a)(1-\sqrt{1-a})}}\pr{\pr{\sqrt{2}+\frac{(c+1)(\sqrt{2}-1)}{\log N}}(\sqrt{1-a}^t)+(1-\sqrt{1-a}^t)}\\
     &\leq \frac{a}{\sqrt{\frac{N}{n_0}(1-a)}(1-\sqrt{1-a})}+\frac{a}{\sqrt{\frac{N}{n_0}(1-a)}(1-\sqrt{1-a})}\cdot\sqrt{1-a}^t
\end{align}
The above implies that the number of rounds in this second phase is going to be the smallest $t$ that guarantees that the target accuracy has been achieved that is 
\begin{align}
    \epsilon \geq \frac{a}{\sqrt{\frac{N}{n_0}(1-a)}(1-\sqrt{1-a})}+\frac{a}{\sqrt{\frac{N}{n_0}(1-a)}(1-\sqrt{1-a})}\cdot\sqrt{1-a}^t
\end{align}
Further recall that from \eqref{targetacc} the accuracy can be expressed in terms of 
\begin{align}
    \epsilon = \hat{c} \frac{\alpha}{\sqrt{\frac{N}{n_0}\pr{1-\alpha}}\pr{1-\sqrt{1-\alpha}}}, \qquad \textit{with} \qquad \sqrt{2}>\hat{c}>1.
\end{align}
Combining the above and solving for $t$ we derive the required number of rounds for the second phase
\begin{align}
    t \geq \frac{2\log (\frac{1}{\hat{c}-1})}{\log (\frac{1}{1-a})}
\end{align}
The expected cost during phase $2$ can be computed as follows
\begin{align}
    \Ex{T^2_{\textit{SRPFL}}}  \leq &(\Ex{\mT_N}+\mC)\pr{\frac{2\log (\frac{1}{\hat{c}-1})}{\log (\frac{1}{1-a})}+1}\\
    \leq &(\ln (N)+1+c)\frac{1}{\lambda}\pr{\frac{2\log (\frac{1}{\hat{c}-1})}{\log (\frac{1}{1-a})}+1}\\
     \leq &4\log (N)\pr{\frac{\log (\frac{1}{\hat{c}-1})}{\log (\frac{1}{1-a}}}\frac{1}{\lambda}\\
     \leq &\log (N) \cdot\frac{4}{\log (\frac{1}{1-a})}\log \pr{\frac{1}{\hat{c}-1}}\frac{1}{\lambda}
\end{align}
Summing the two quantities of interest we can derive the promised upper bound on the expected `Wall Clock Time' of \texttt{SRPFL}. 
\begin{align}
    \Ex{T_{\textit{SRPFL}}} &= \Ex{T^1_{\textit{SRPFL}}} + \Ex{T^2_{\textit{SRPFL}}}\\
   &\leq  \log(\nicefrac{N}{n_0}) \frac{6(c+1)}{\log\pr{\frac{1}{1-a}}}\cdot \frac{1}{\lambda}+\log (N) \cdot\frac{4}{\log (\frac{1}{1-a})}\log \pr{\frac{1}{\hat{c}-1}}\frac{1}{\lambda}\\
    &\leq\log N \pr{\frac{6(c+1)+4\log (\frac{1}{\hat{c}-1})}{\log (\frac{1}{1-a})}}\frac{1}{\lambda}
\end{align}
\end{proof}
Having computed an upper bound on the expected `Wall Clock Time' of \texttt{SRPFL} we proceed to compute an lower bound on the expected `Wall Clock Time' of \texttt{FedRep}.
\begin{lemma}\label{WCTLemma2}
Suppose at each stage the client's computational times are i.i.d. random variables drawn from the exponential distribution with parameter $\lambda$. Further, suppose that the expected communication cost per round is $\mathcal{C}=c\frac{1}{\lambda}$, for some constant $c$. Finally, consider target accuracy $\epsilon$ given in \eqref{accuracy}. Then the expected `Wall Clock Time' for \texttt{FedRep} is lower bounded as follows 
\begin{align}
   \Ex{T_{\textit{FedRep}}} \geq \log N \pr{\frac{\log N +2\log \pr{\frac{1}{\hat{c}-1}}}{\log\pr{\frac{1}{1-a}}}} \frac{1}{\lambda}
\end{align}
\end{lemma}
\begin{proof}
First we compute the number of rounds required by \texttt{FedRep} to achieve the target accuracy. Recall that \texttt{FedRep} utilizes $N$ clients at each round deriving the following form of the contractions inequality from \eqref{conttwo}
\begin{align}\label{contfinalFedRep}
\dist^+& \leq \dist \sqrt{1-a}+\frac{a}{\sqrt{\frac{N}{n_0}(1-a)}}\quad \textit{with} \quad a \leq \frac{1}{4}. 
\end{align}
Starting with distance equal $1$ and following the contraction in \eqref{contfinalFedRep} for $t$ rounds, we derive current distance at most
\begin{align}
    (\sqrt{1-a})^t + \sum_{i=0}^{t-1}\frac{a}{\sqrt{\frac{N}{n_0}(1-a)}}(\sqrt{1-a})^i=
    (\sqrt{1-a})^t + \frac{a}{\sqrt{\frac{N}{n_0}(1-a)}}\frac{1-\sqrt{1-a}^t}{1-\sqrt{1-a}},
\end{align}
using the properties of geometric series. Further recall that from \eqref{targetacc} the accuracy can be expressed as 
\begin{align}
    \epsilon = \hat{c} \frac{\alpha}{\sqrt{\frac{N}{n_0}\pr{1-\alpha}}\pr{1-\sqrt{1-\alpha}}}, \qquad \textit{with} \qquad \sqrt{2}>\hat{c}>1.
\end{align}
The above imply that the number of rounds is going to be the smallest $t$ that guarantees that the target accuracy has been achieved that is 
\begin{align}
    \hat{c} \frac{\alpha}{\sqrt{\frac{N}{n_0}\pr{1-\alpha}}\pr{1-\sqrt{1-\alpha}}} \geq (\sqrt{1-a})^t + \frac{a}{\sqrt{\frac{N}{n_0}(1-a)}}\frac{1-\sqrt{1-a}^t}{1-\sqrt{1-a}}
\end{align}
We use the fact that $\sqrt{\frac{N}{n_0}(1-a)}(1-\sqrt{1-a})-a >0$ for $a\leq \nicefrac{1}{4}$ and all reasonable values of $N$ to rearrange and solve for $t$. Thus, we derive
\begin{align}
  t \geq \frac{2\log \pr{\frac{1}{\hat{c}-1}}}{\log\pr{\frac{1}{1-a}}} + \frac{\log N}{\log\pr{\frac{1}{1-a}}}  
\end{align}
Multiplying the number of rounds with a lower bound on the expected cost incurred per round, results in the desired lower bound on the expected `Wall Clock Time' suffered by \texttt{FedRep}:
\begin{align}
    \Ex{T_{\textit{FedRep}}} &\geq  \pr{\Ex{\mT_N}+\mC}\pr{\frac{2\log \pr{\frac{1}{\hat{c}-1}}}{\log\pr{\frac{1}{1-a}}}}\\
    &\geq \log N \pr{\frac{\log N +2\log \pr{\frac{1}{\hat{c}-1}}}{\log\pr{\frac{1}{1-a}}}} \frac{1}{\lambda}
\end{align}
\end{proof}
Combining the results of \cref{WCTLemma1} and \cref{WCTLemma2} we obtain \cref{SpeedupTh}.
\SecondRestatable*
\begin{proof}
\begin{align}
\frac{\Ex{T_{\textit{SRPFL}}}}{\Ex{T_{\textit{FedRep}}}} \leq \frac{6(c+1)+4\log (\frac{1}{\hat{c}-1})}{\log N +2\log \pr{\frac{1}{\hat{c}-1}}} = \OO\pr{\frac{\log \pr{\frac{1}{\hat{c}-1}}}{\log(N)+\log\pr{\frac{1}{\hat{c}-1}}}}
\end{align}
\end{proof}
\begin{remark}
The initialization scheme in \cref{alg:3} guarantees that $\dist{\pr{\B^0,\B^*}}\leq 1-c$, with probability at least $1-\OO\pr{(mn)^{-100}}$, effectively without increasing the overall sample complexity. The formal statement and proof is identical to Theorem $3$ in \cite{collins2021exploiting} and is omitted.  
\end{remark}
\section{More on Experiments} \label{appendix_section_experiment}
\paragraph{Hyperparameters and choice of models.}
We set the hyperparameters mainly following the previous work \cite{collins2021exploiting}.
Specifically, the learning rate is set to $0.01$ and the batch size is fixed to $10$ in all of our experiments.
The number of local epochs is set to $1$ in experiments on CIFAR10 with $N=100$ and is set to $5$ on other datasets.\\
In terms of the choice of the neural network model, for CIFAR10, we use LeNet-5 including two convolution layers with $(64, 64)$ channels and three fully connected layers where the numbers of hidden neurons are $(120, 64)$.
The same structure is used for CIFAR100, but the numbers of channels in the convolution layers are increased to $(64, 128)$ and the numbers of hidden neurons are increased to $(256, 128)$.
Additionally, a dropout layer with parameter $0.6$ is added after the first two fully connected layers, which improves the testing accuracy.
For EMNIST and FEMNIST, we use an MLP including three hidden layers with the number of hidden neurons being $(512, 256, 64)$.\\
For \texttt{SRFRL}, we need to split the neural network model into two parts, the customized head $h_i$ and the common representation $\phi$.
In our experiments, we take the customized head to be the last hidden layer and the rest of the parameters are treated as the common representation.
Note that LG-FedAvg and LG-FLANP have a different head/representation split scheme and that the head is globally shared across all clients while a local version of the representation is maintained on every client.
For all included datasets, i.e. CIFAR10, CIFAR100, EMNIST, and FEMNIST, the common head include the last two fully connected layers and the rest of the layers are treated as the representation part.
\vspace{-0.1in}
\paragraph{Datasets.} We include four datasets in our empirical study: CIFAR10 which consists of 10 classes and a total number of 50,000 training data points, CIFAR100 which consists of 100 classes and the same amount of data points as CIFAR10, and EMNIST (balanced) which consists of 47 classes and 131,600 training data points.
Note that in Figure \ref{fig:experiment_main}, we use the first 10 classes from EMNIST as did in \cite{collins2021exploiting}. 
For FEMNIST, we use the same setting as \cite{collins2021exploiting} except that we allocate each client $150$ data points without using the log-normal distribution as \cite{collins2021exploiting}.
During our training procedure, we perform the data augmentation operations of standard random crop and horizontal flip on the first two datasets and use the input as it is for the last one.
\vspace{-0.1in}
\paragraph{Results under the configuration of random computation speed.}
We present results under the second type of system heterogeneity, random computation speed, in Figure \ref{fig:cifar-full-random}. The details about this type of system heterogeneity are described in the experiment section of the main body.
\begin{figure*}
    \centering
    \begin{tabular}{c@{} c@{} c@{} c}
        \includegraphics[width=.24\textwidth]{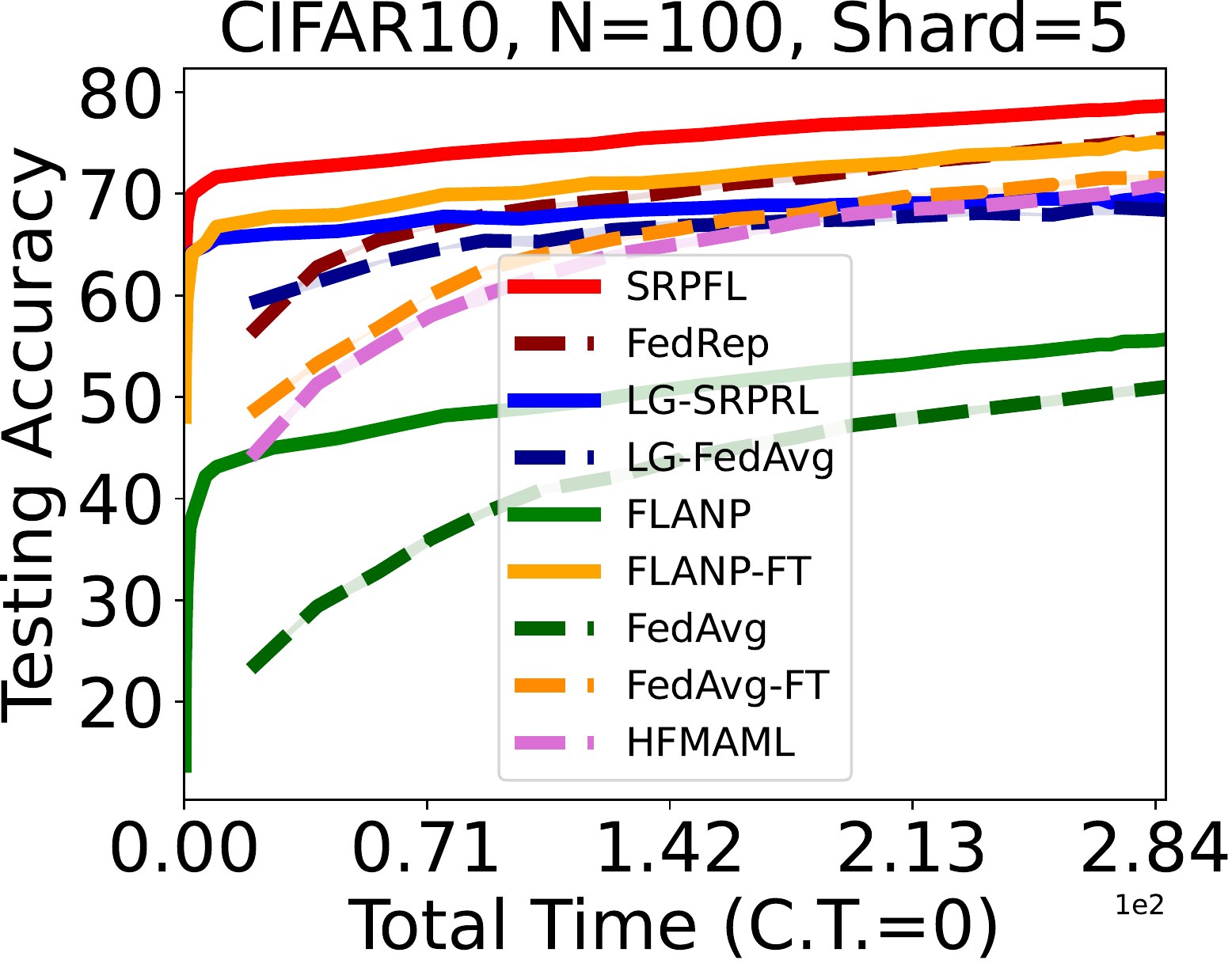} & \includegraphics[width=.24\textwidth]{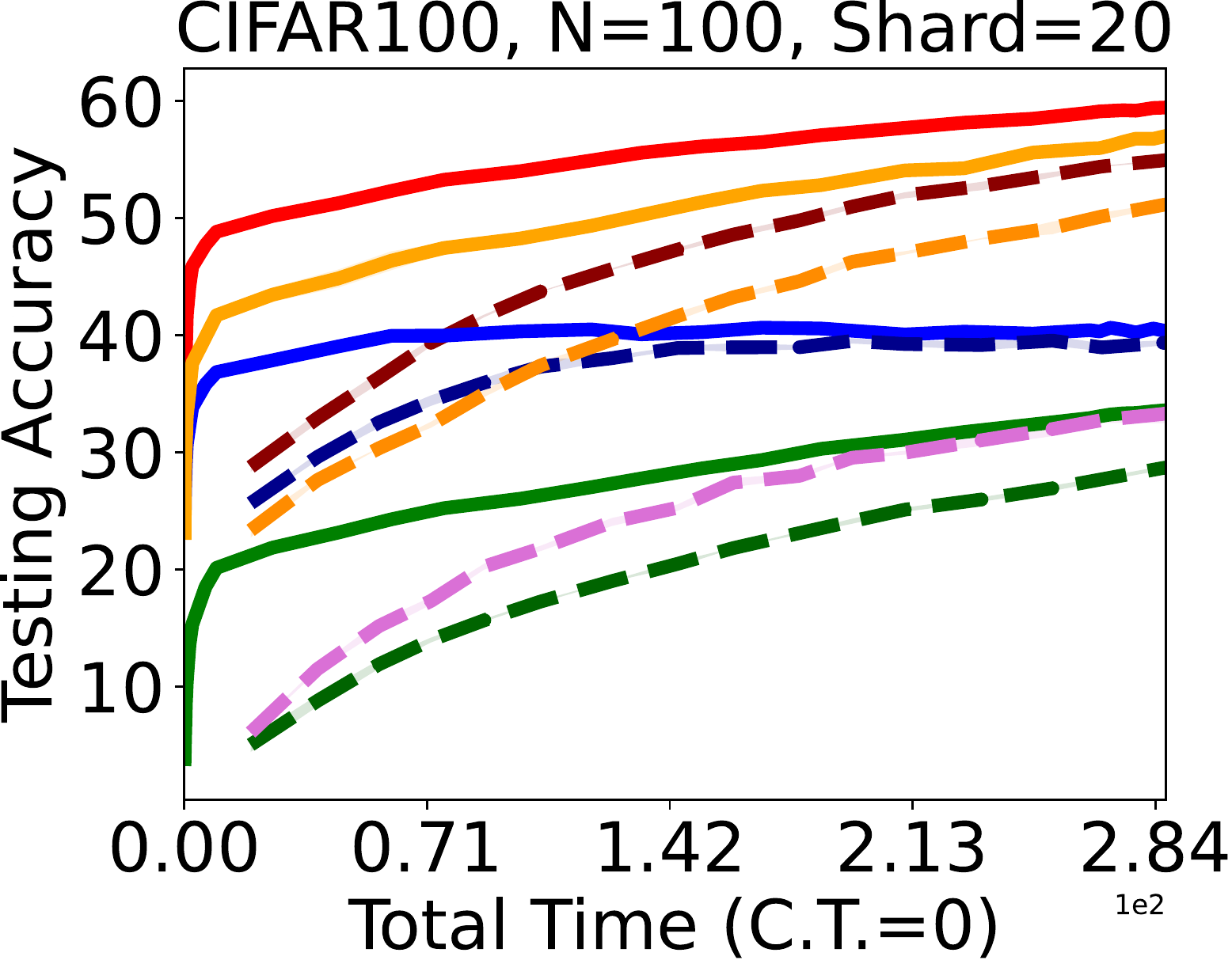} & \includegraphics[width=.24\textwidth]{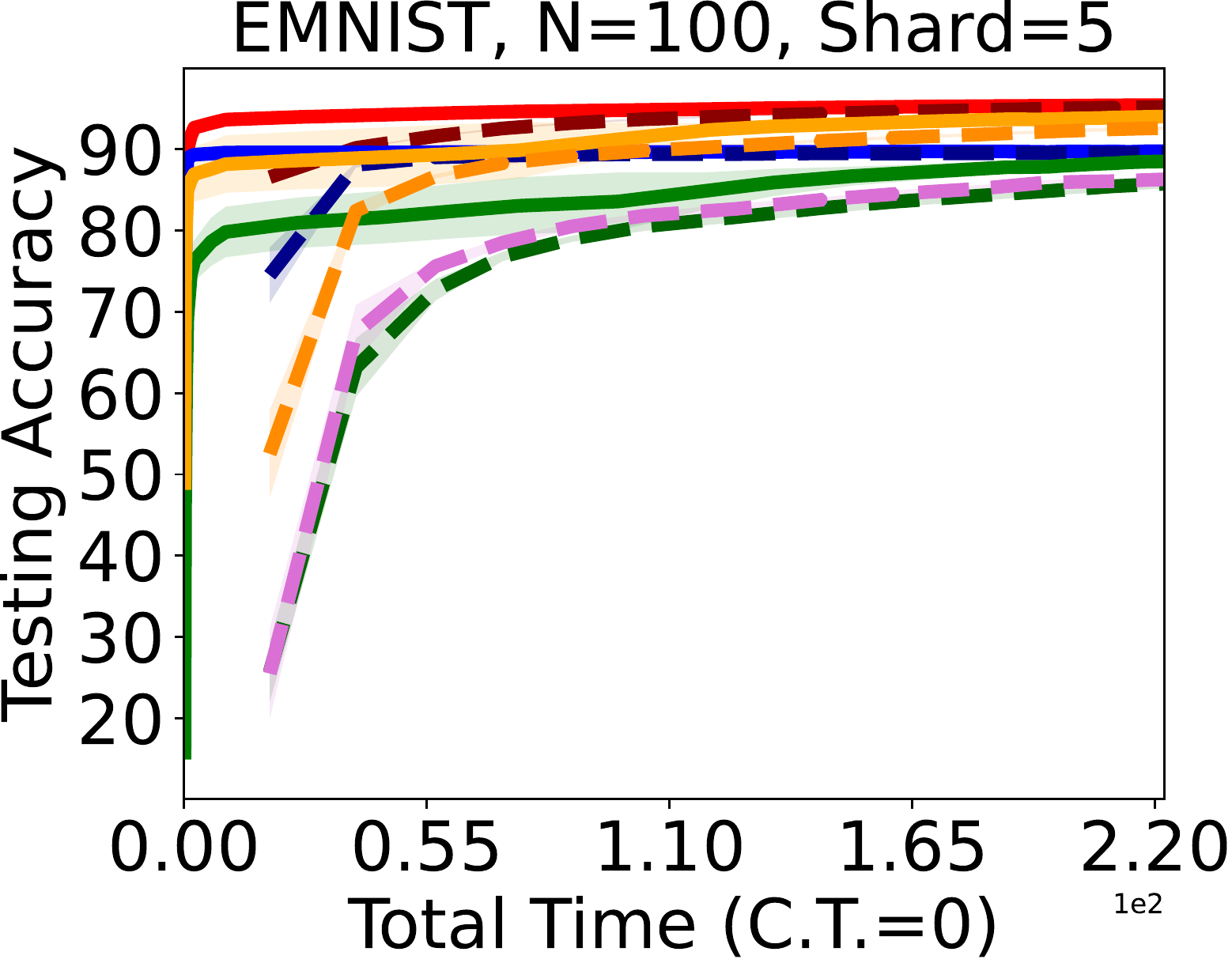} & \includegraphics[width=.24\textwidth]{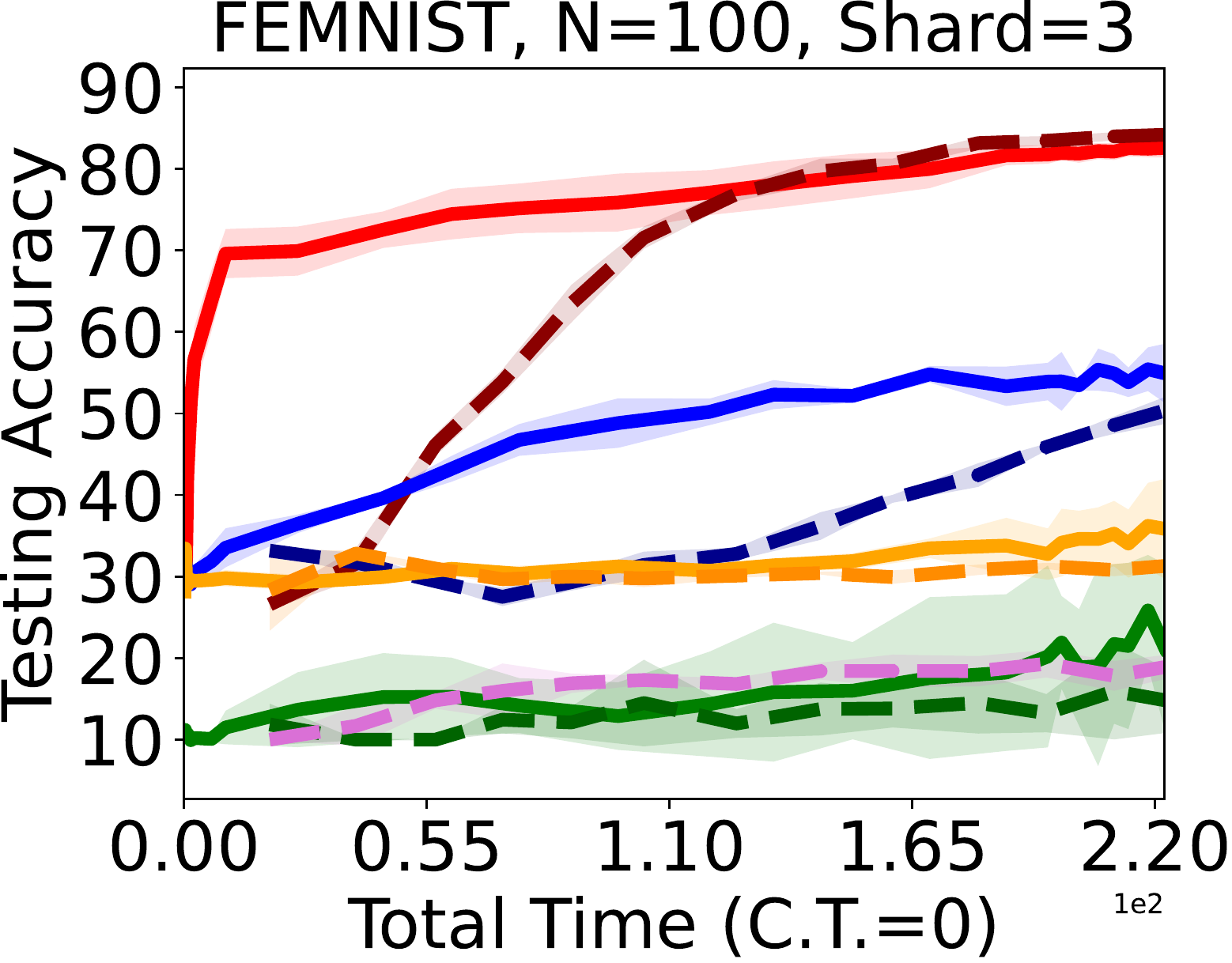}\\
        \includegraphics[width=.24\textwidth]{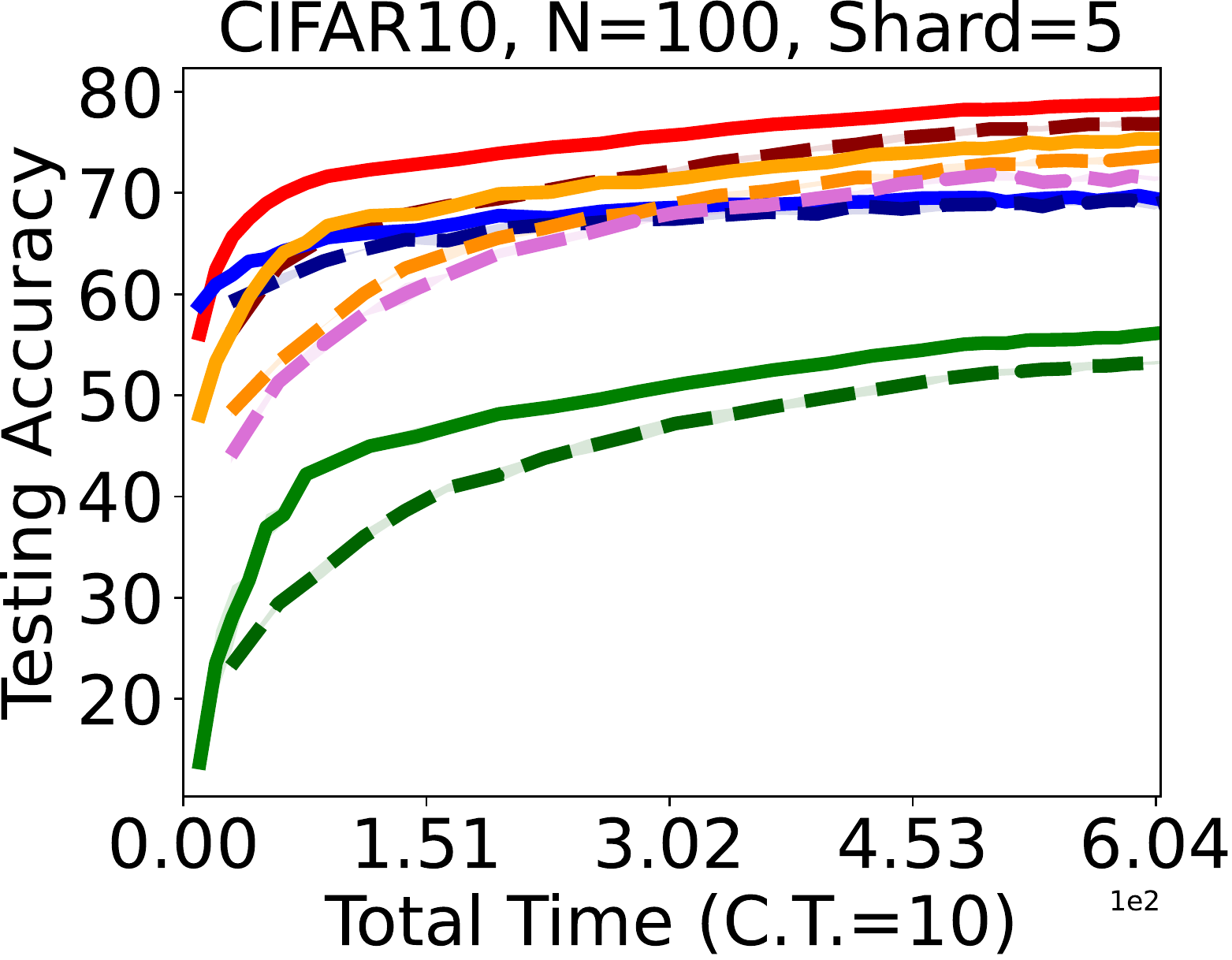} & \includegraphics[width=.24\textwidth]{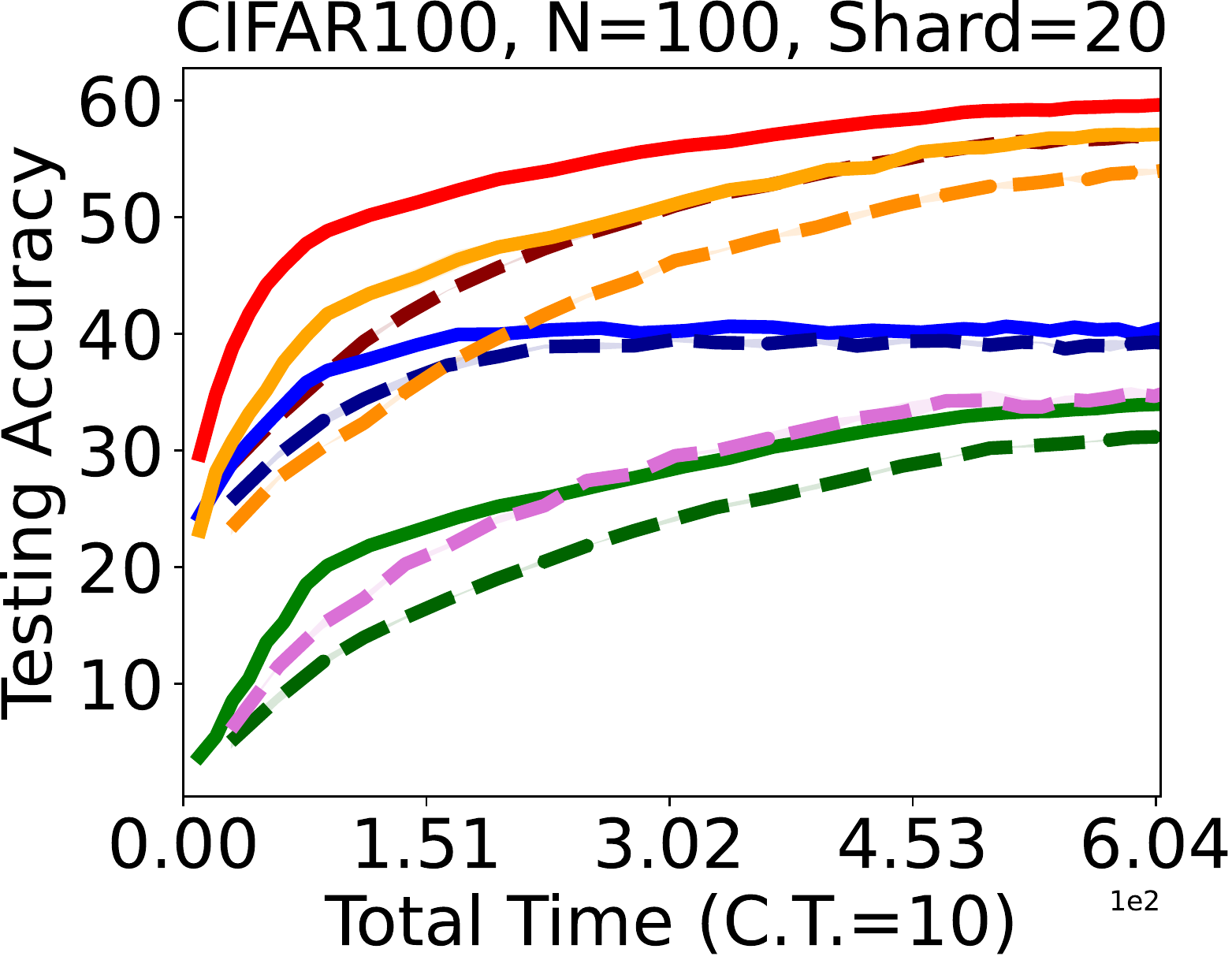} & 
        \includegraphics[width=.24\textwidth]{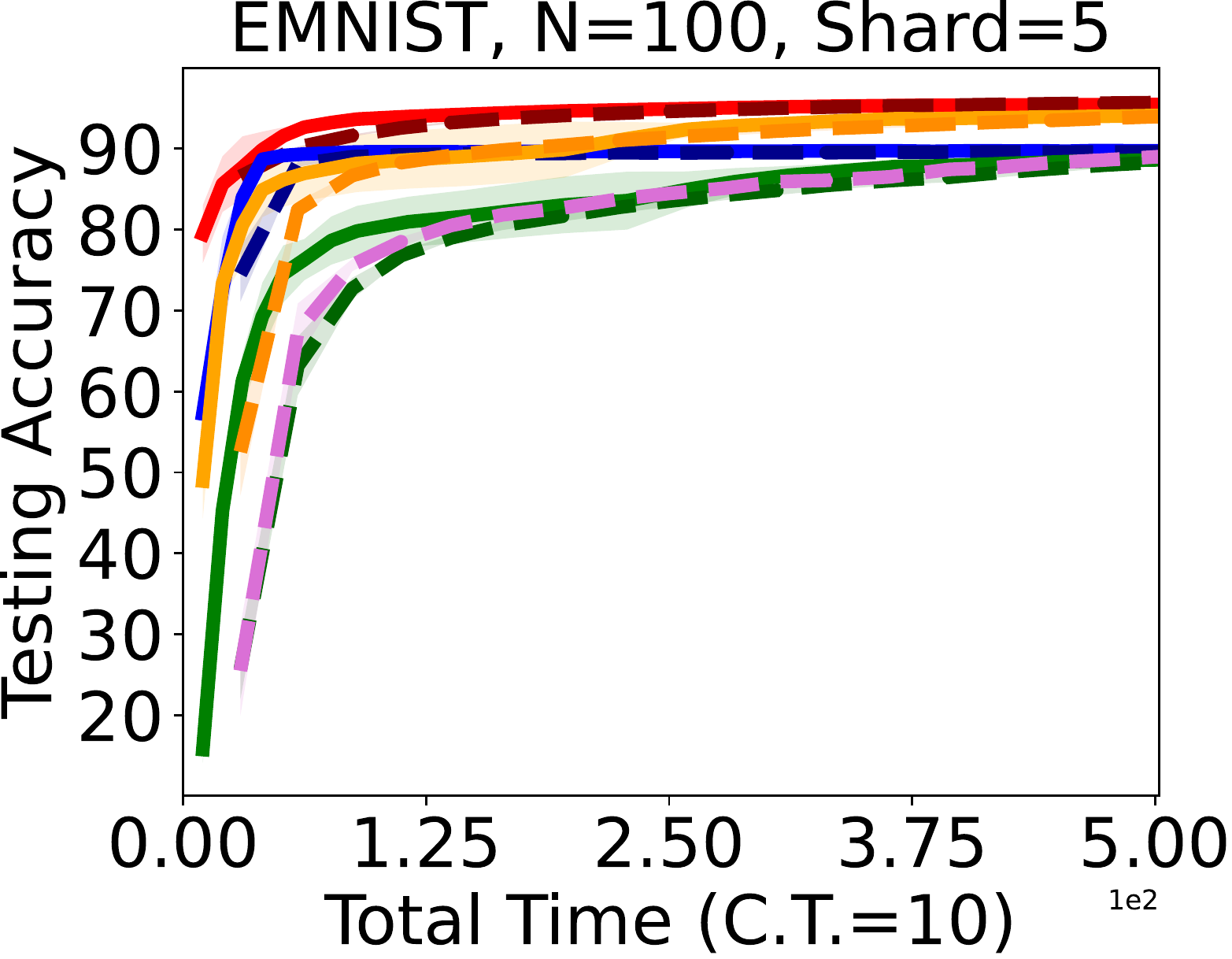}& \includegraphics[width=.24\textwidth]{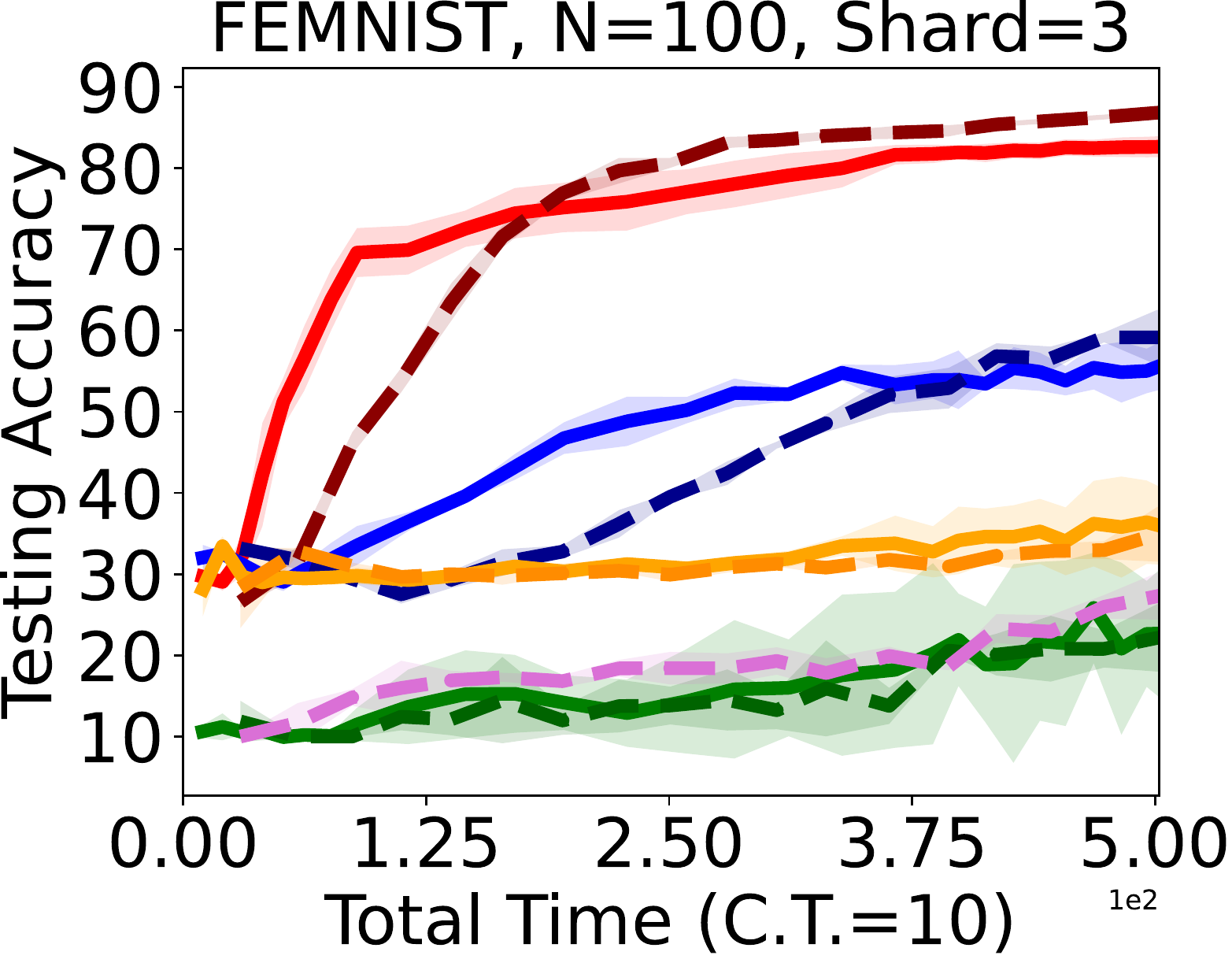} \\ 
        \includegraphics[width=.24\textwidth]{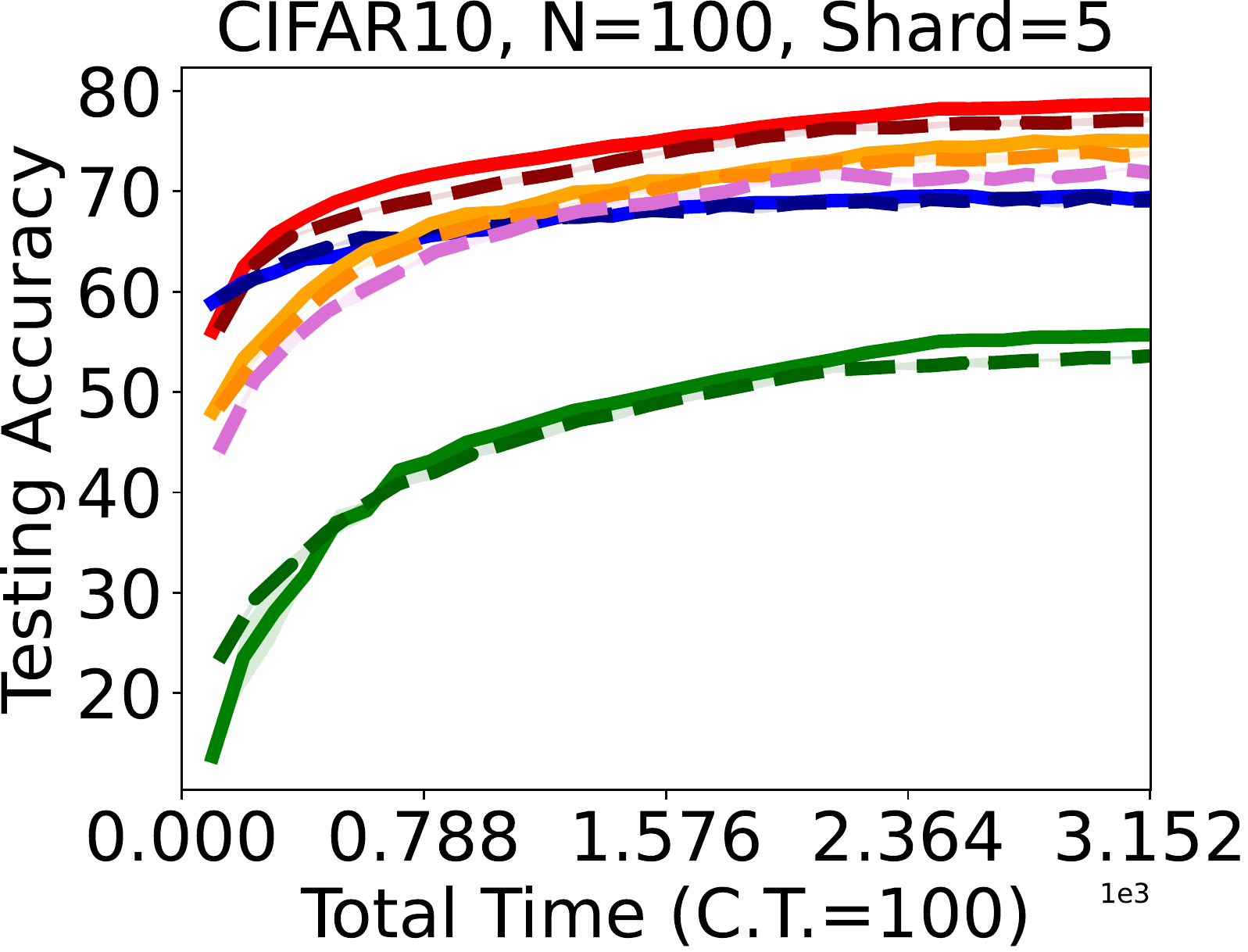} & \includegraphics[width=.24\textwidth]{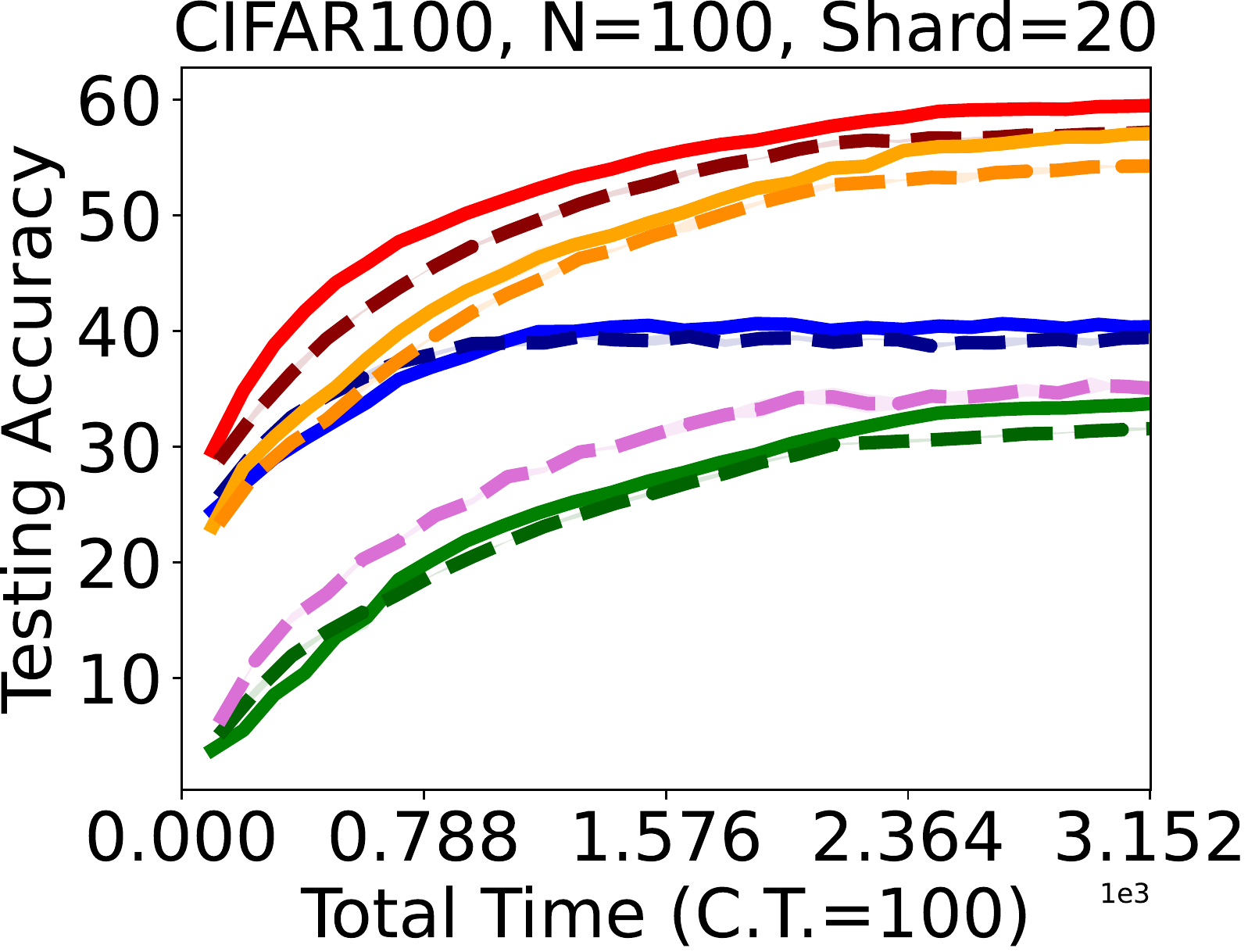} &
        \includegraphics[width=.24\textwidth]{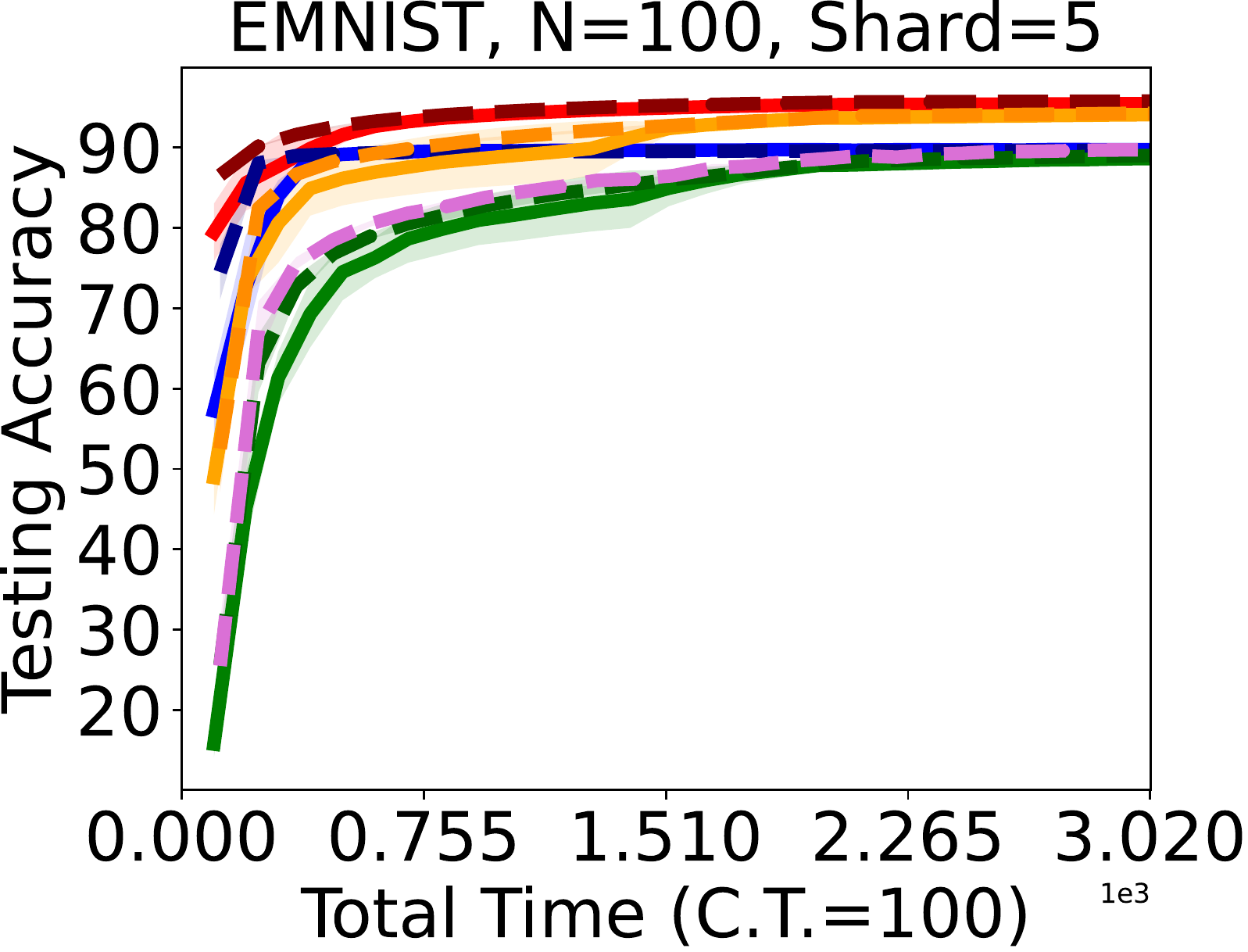}& \includegraphics[width=.24\textwidth]{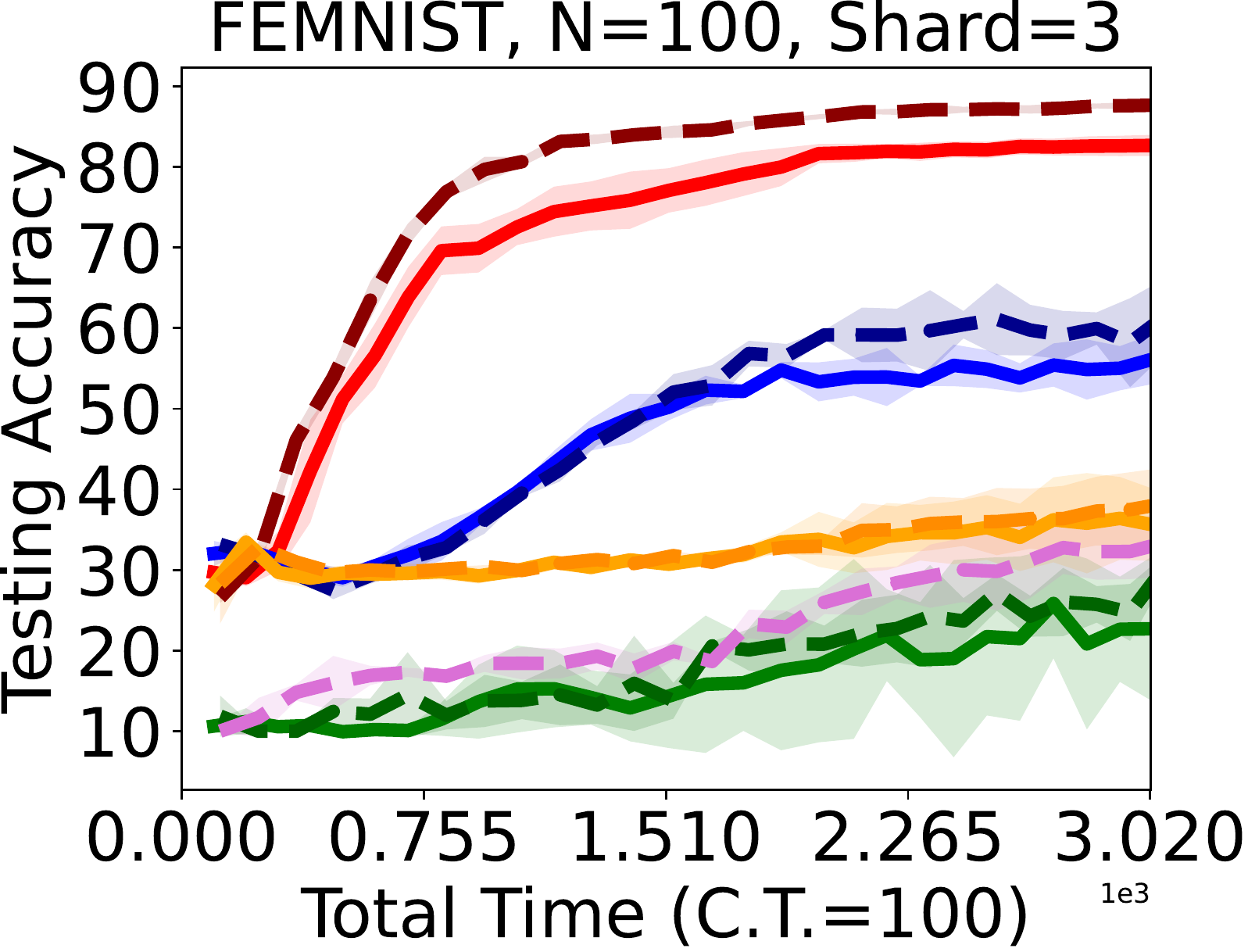}
    \end{tabular}
    \caption{Random computation speed setting. All the active nodes participate, i.e. $N = M$.}
    \label{fig:cifar-full-random}
\end{figure*}
\vspace{-0.1in}
\paragraph{Results under the configuration of partial participation.}
We present results when only a fraction ($20\%$ in the current experiment) of the agents participate in each global round, i.e. $N = M/5$, in Figures \ref{fig:cifar-partial-fixed} and \ref{fig:cifar-partial-random} for two types of system heterogeneity settings respectively.
\begin{figure*}
    \centering
    \begin{tabular}{c@{} c@{} c@{} c}
        \includegraphics[width=.24\textwidth]{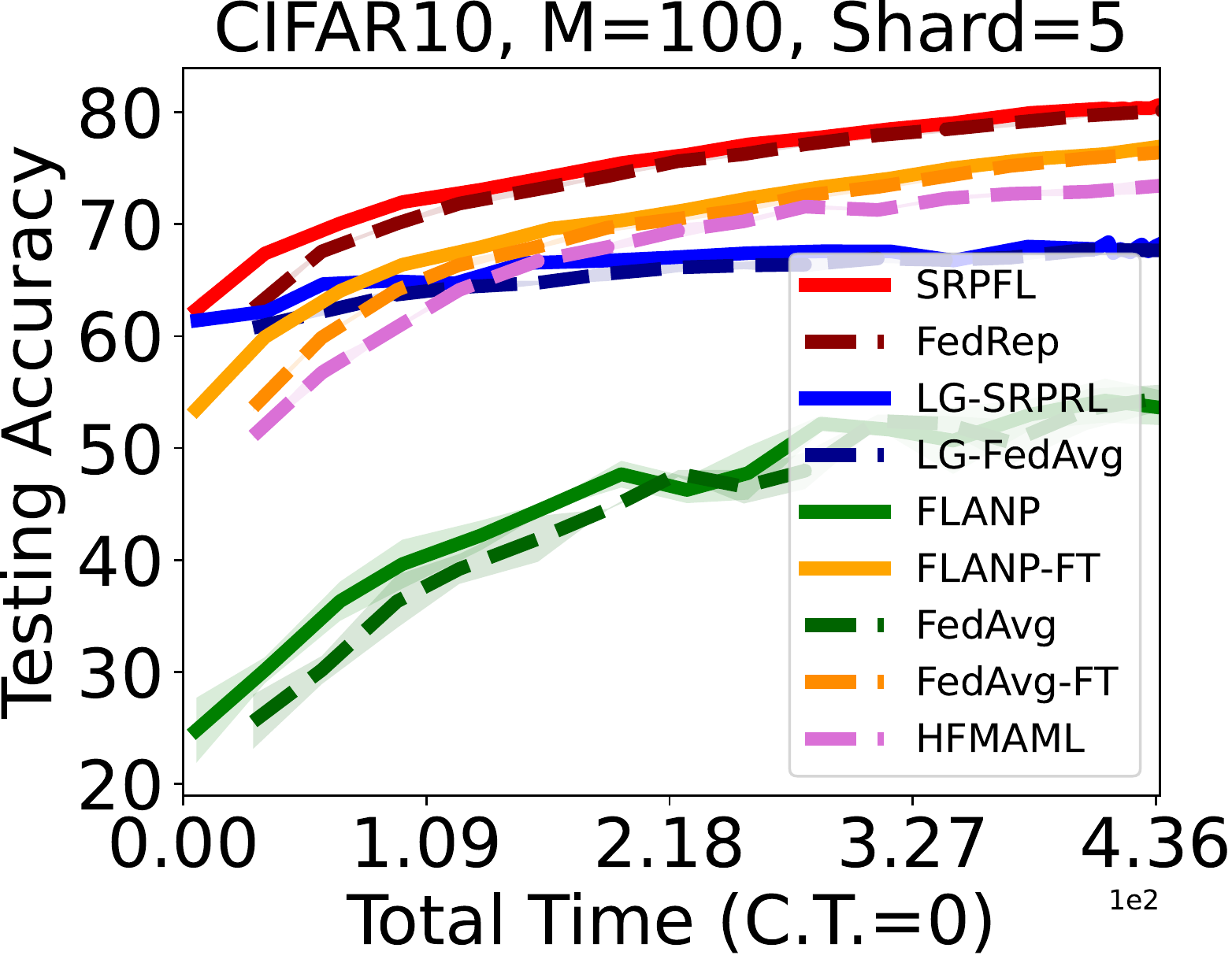} & \includegraphics[width=.24\textwidth]{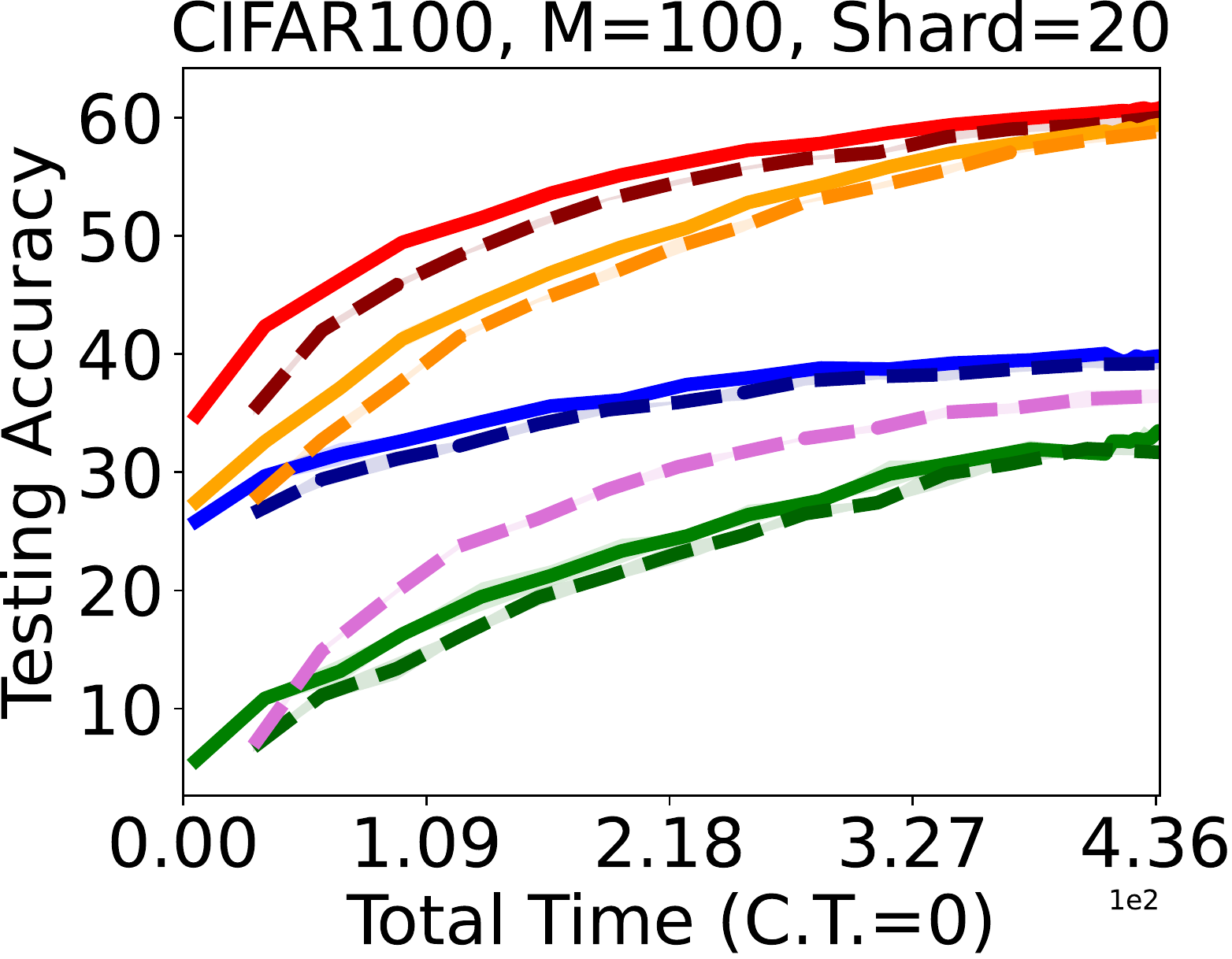} & \includegraphics[width=.24\textwidth]{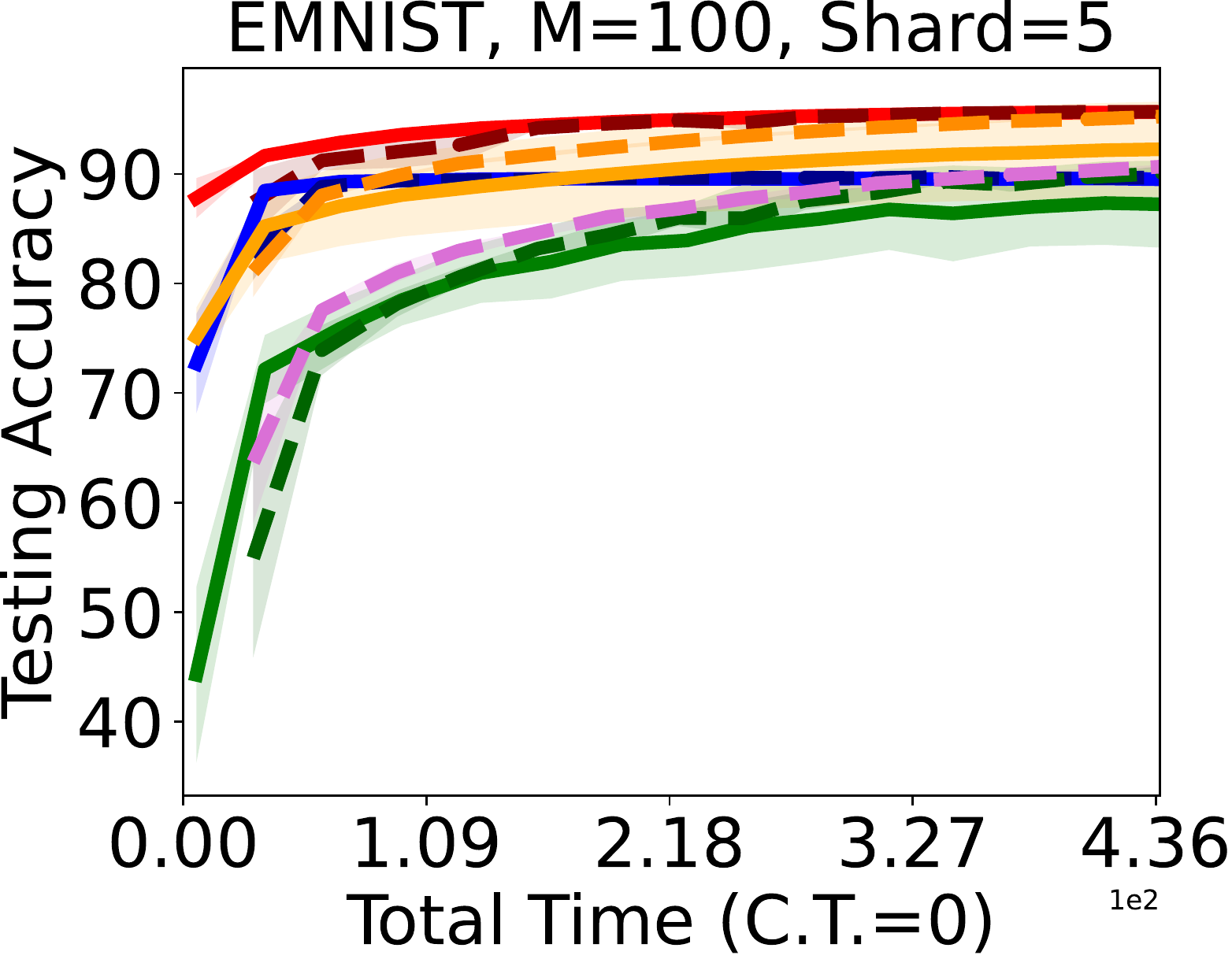} & \includegraphics[width=.24\textwidth]{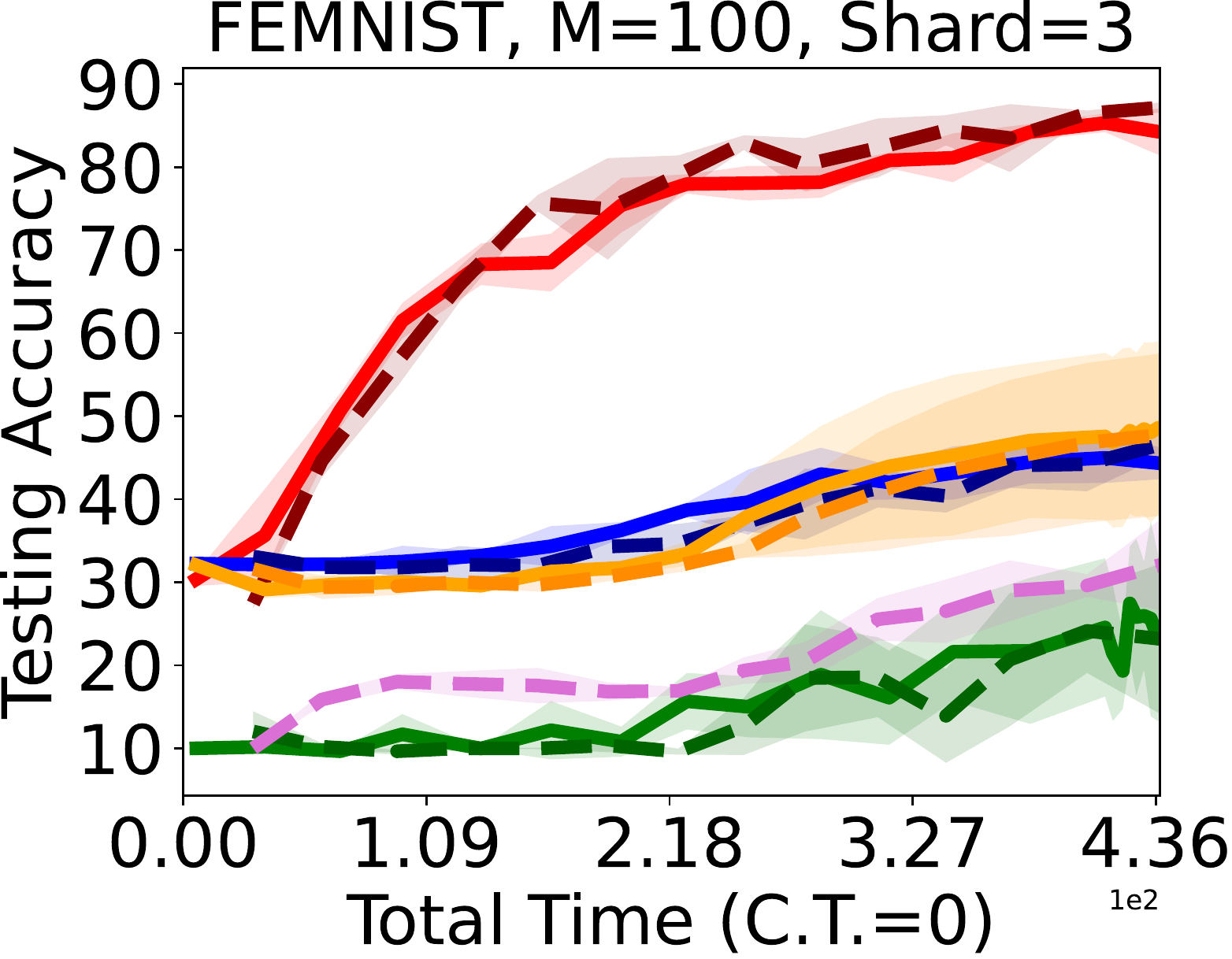}\\
        \includegraphics[width=.24\textwidth]{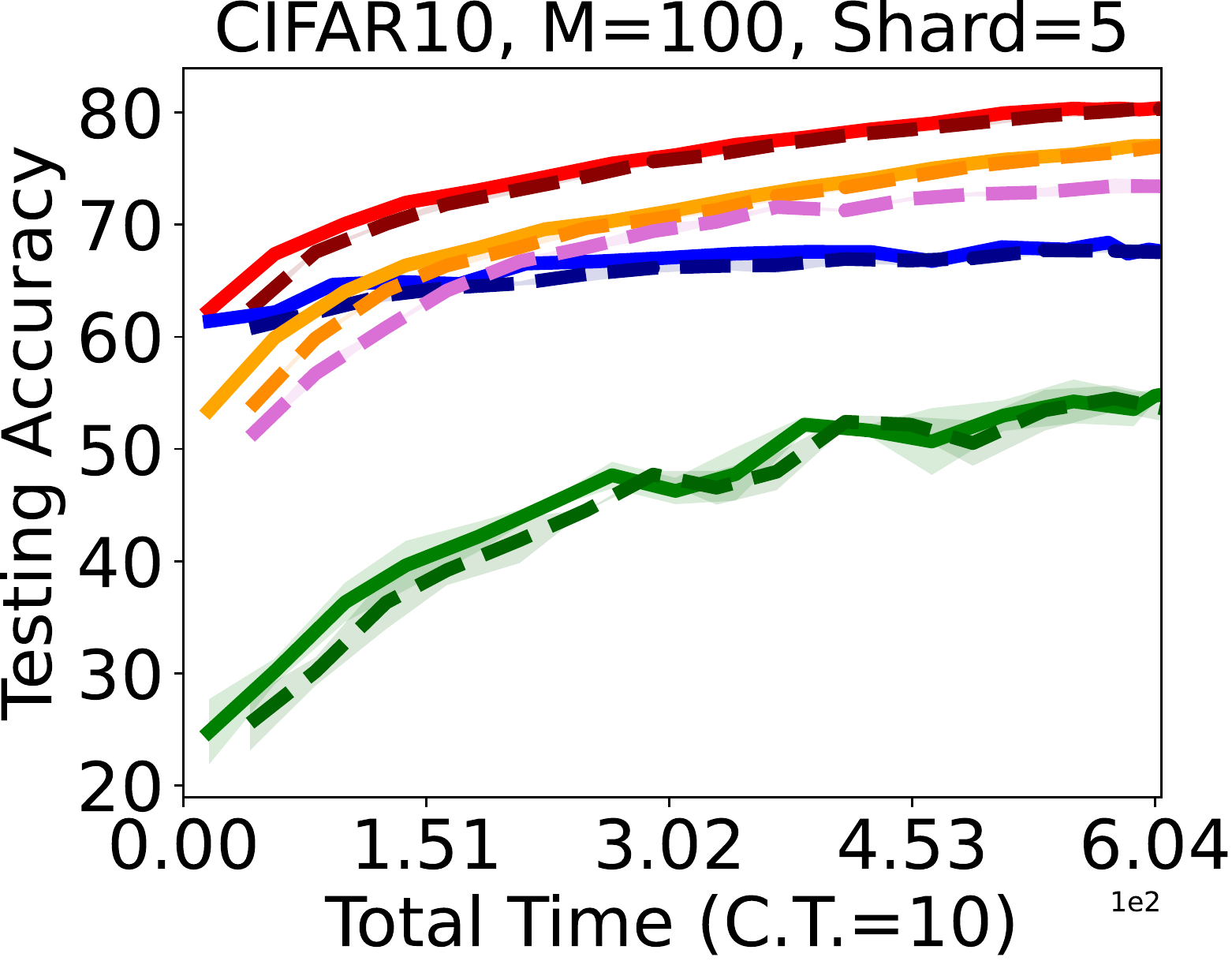} & \includegraphics[width=.24\textwidth]{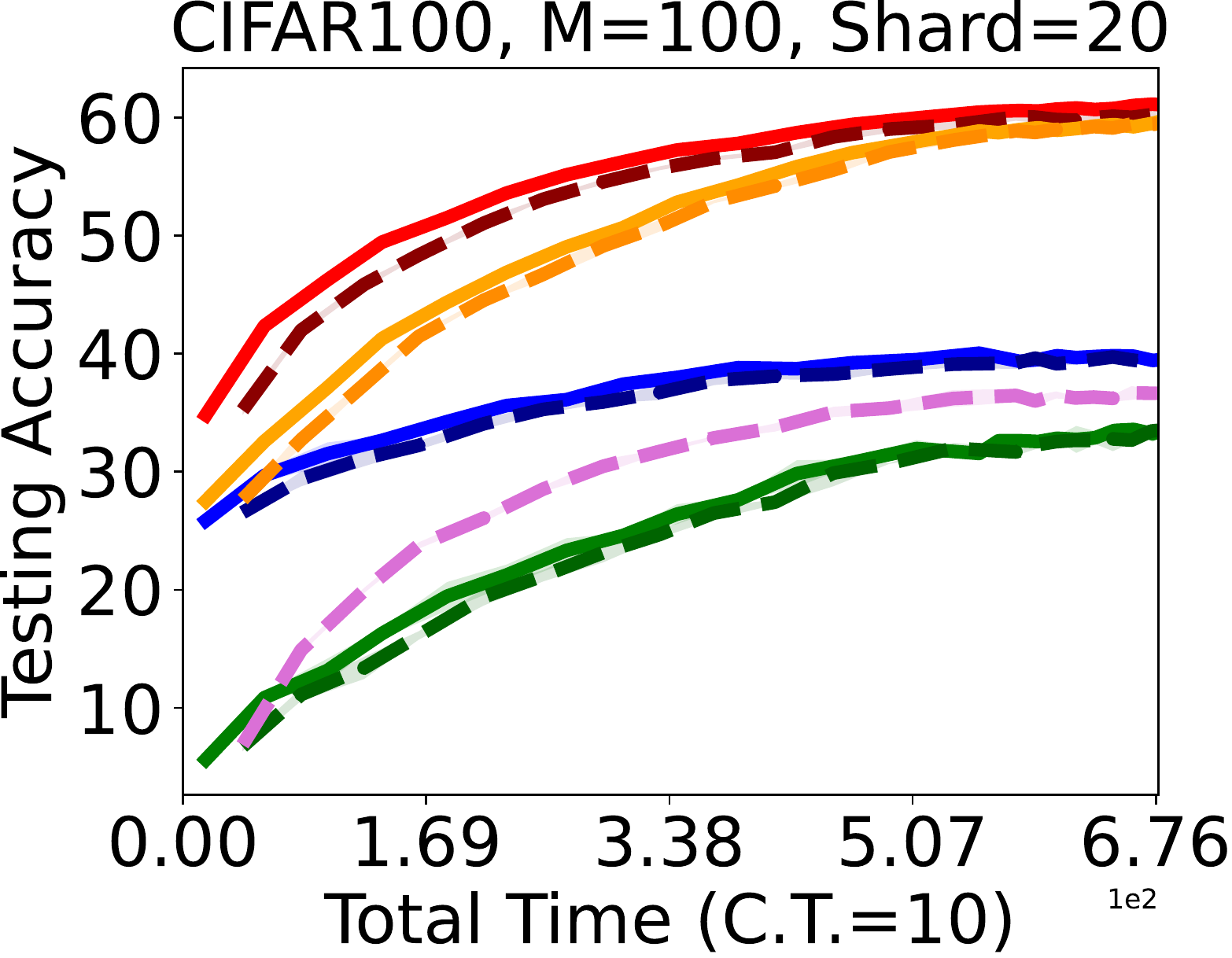} & \includegraphics[width=.24\textwidth]{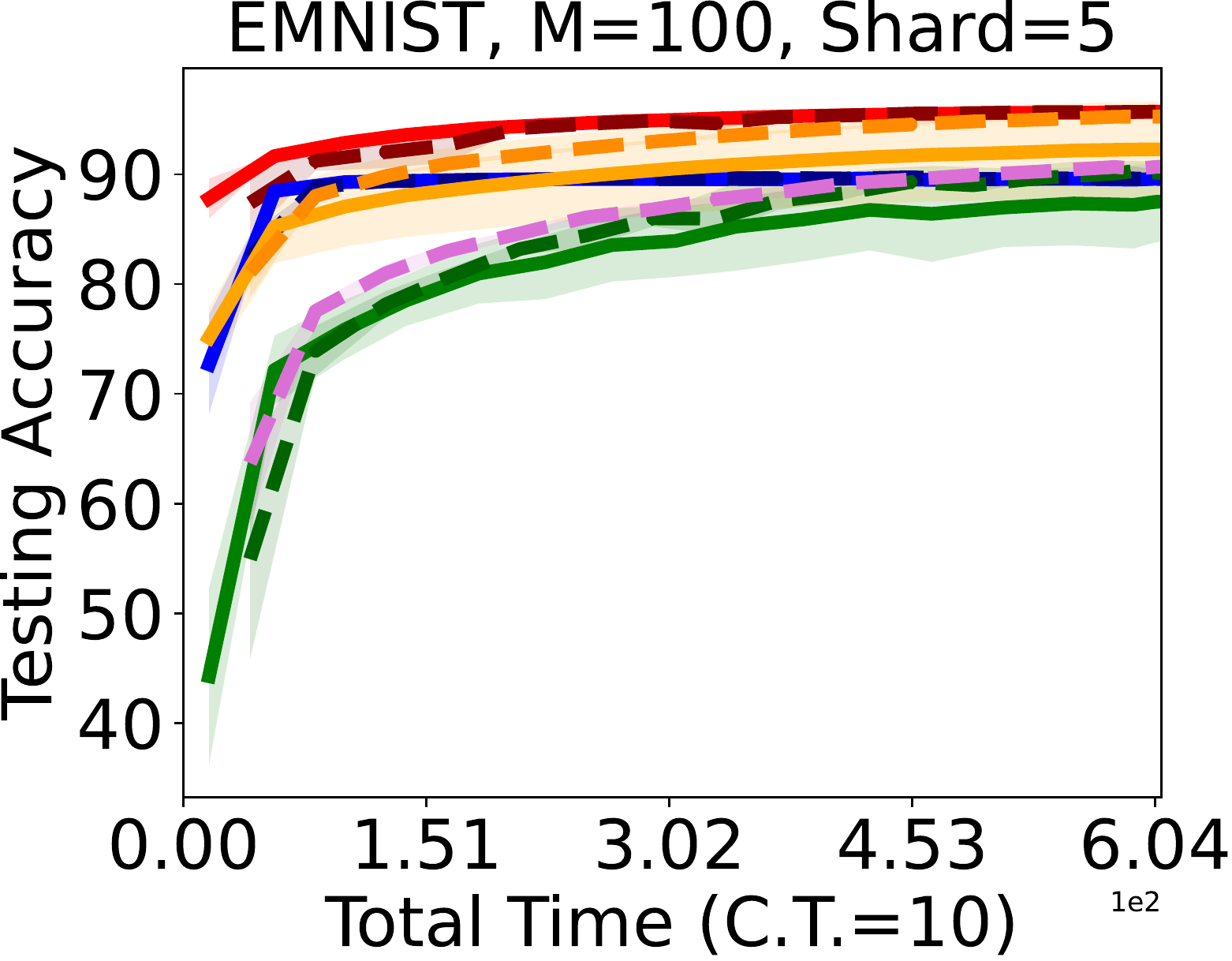}
        & 
        \includegraphics[width=.24\textwidth]{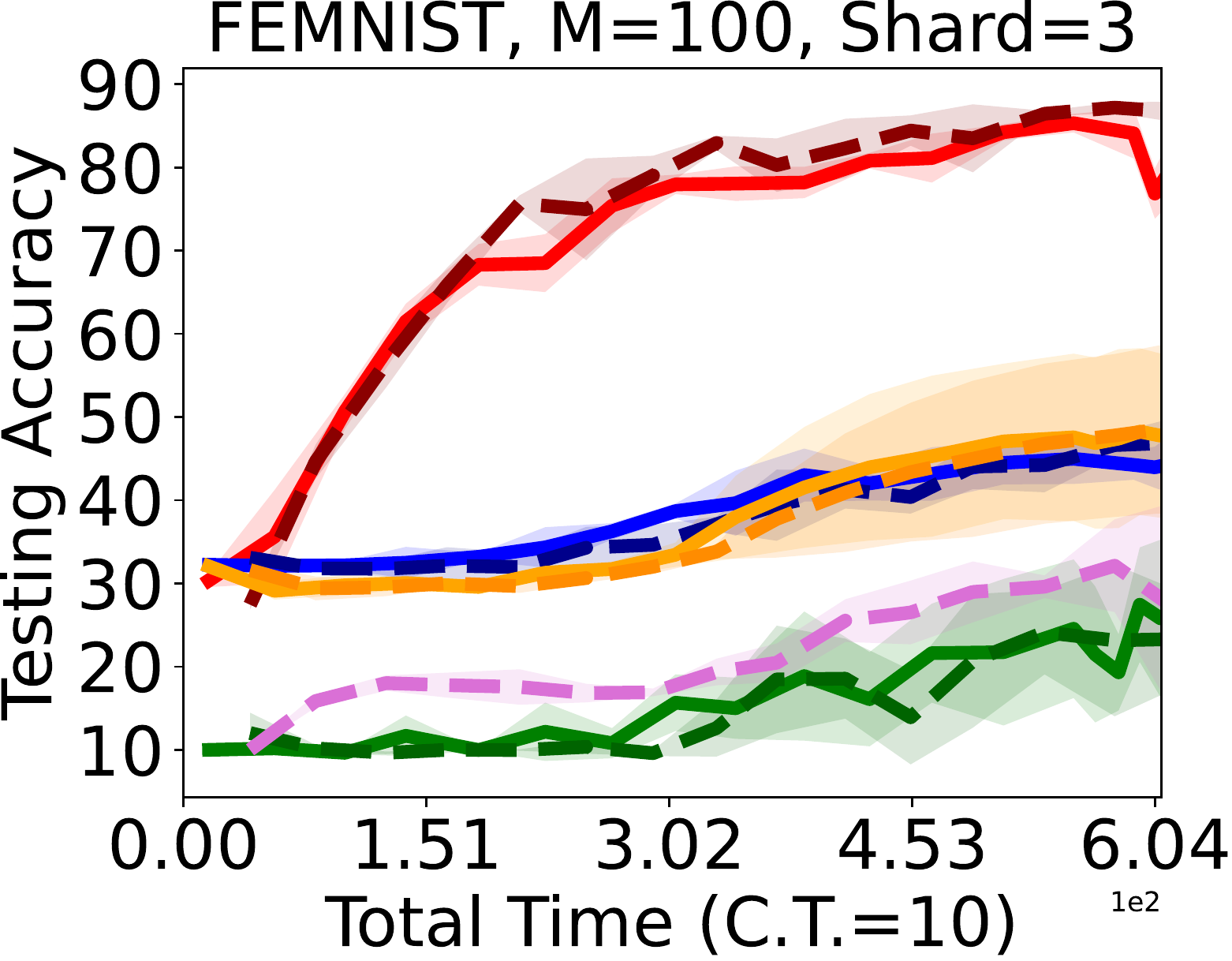}\\ 
        \includegraphics[width=.24\textwidth]{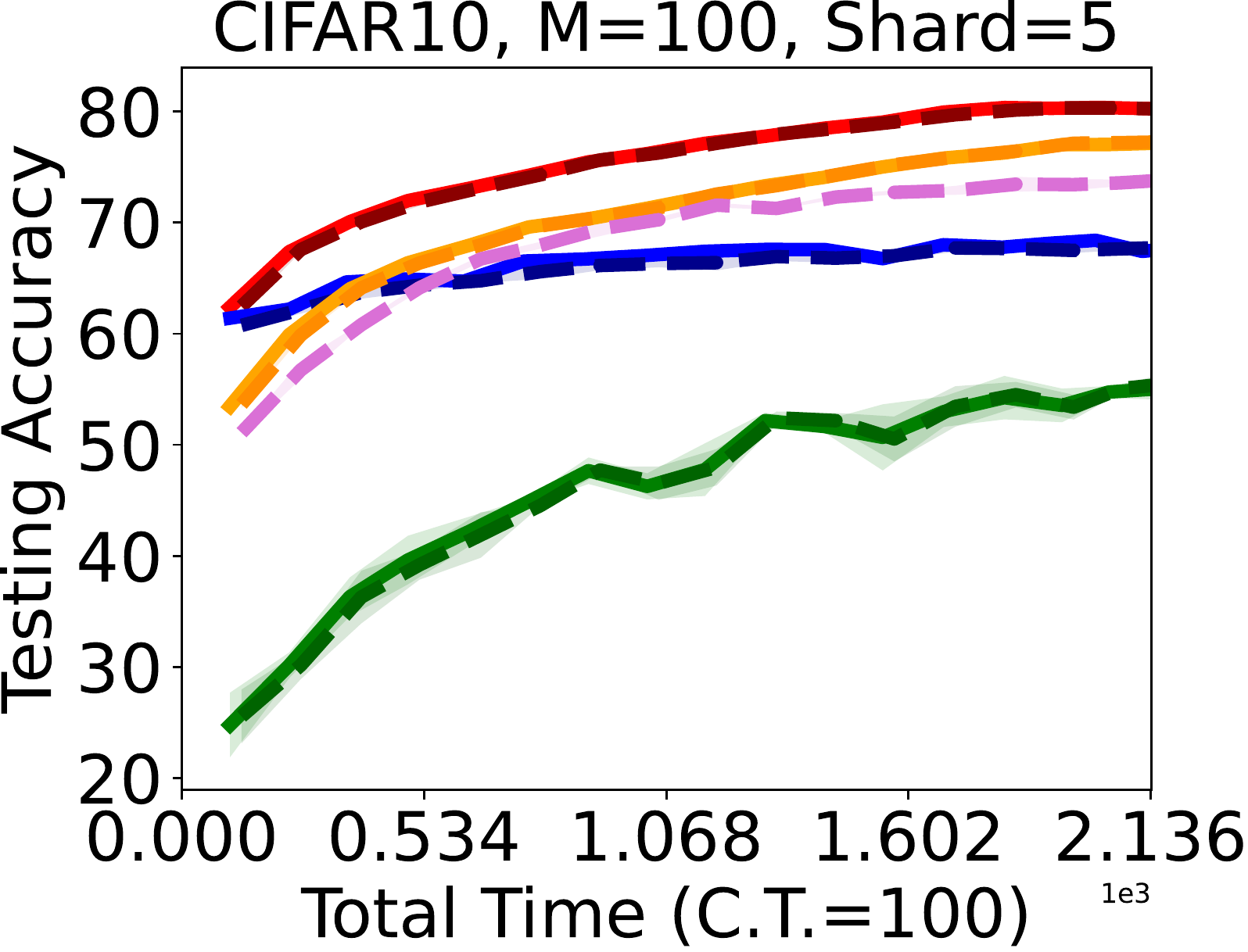} & \includegraphics[width=.24\textwidth]{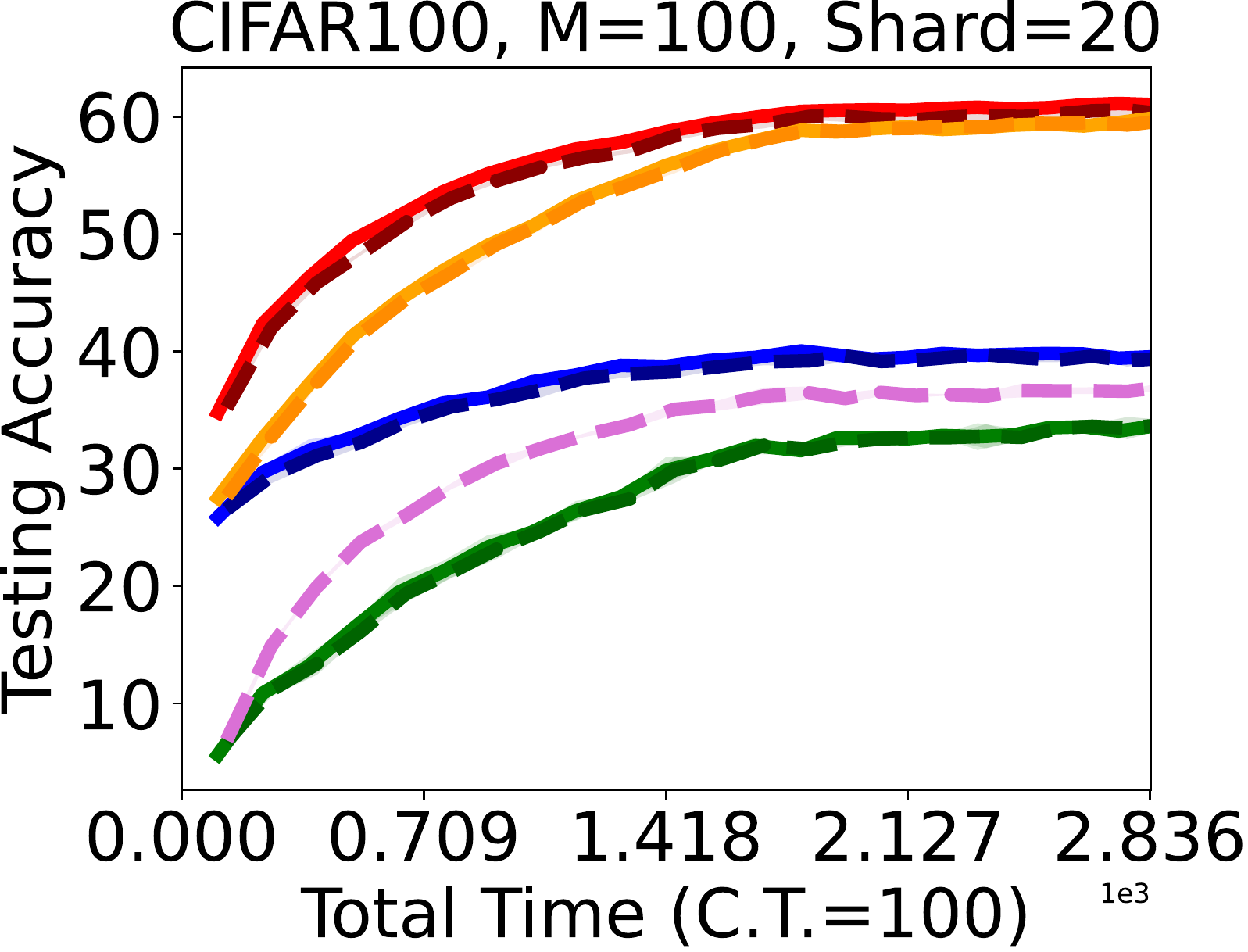} & 
        \includegraphics[width=.24\textwidth]{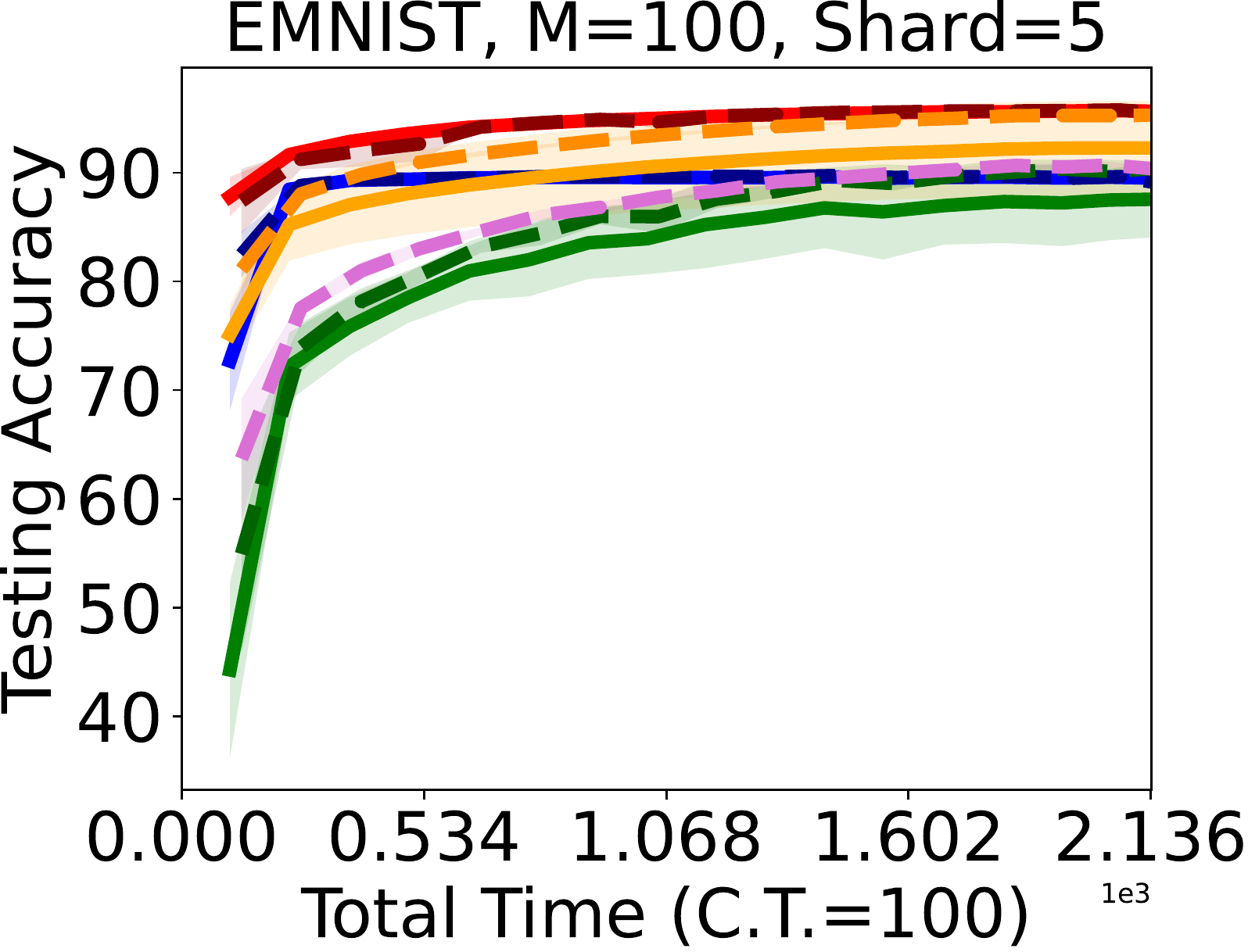}
        & \includegraphics[width=.24\textwidth]{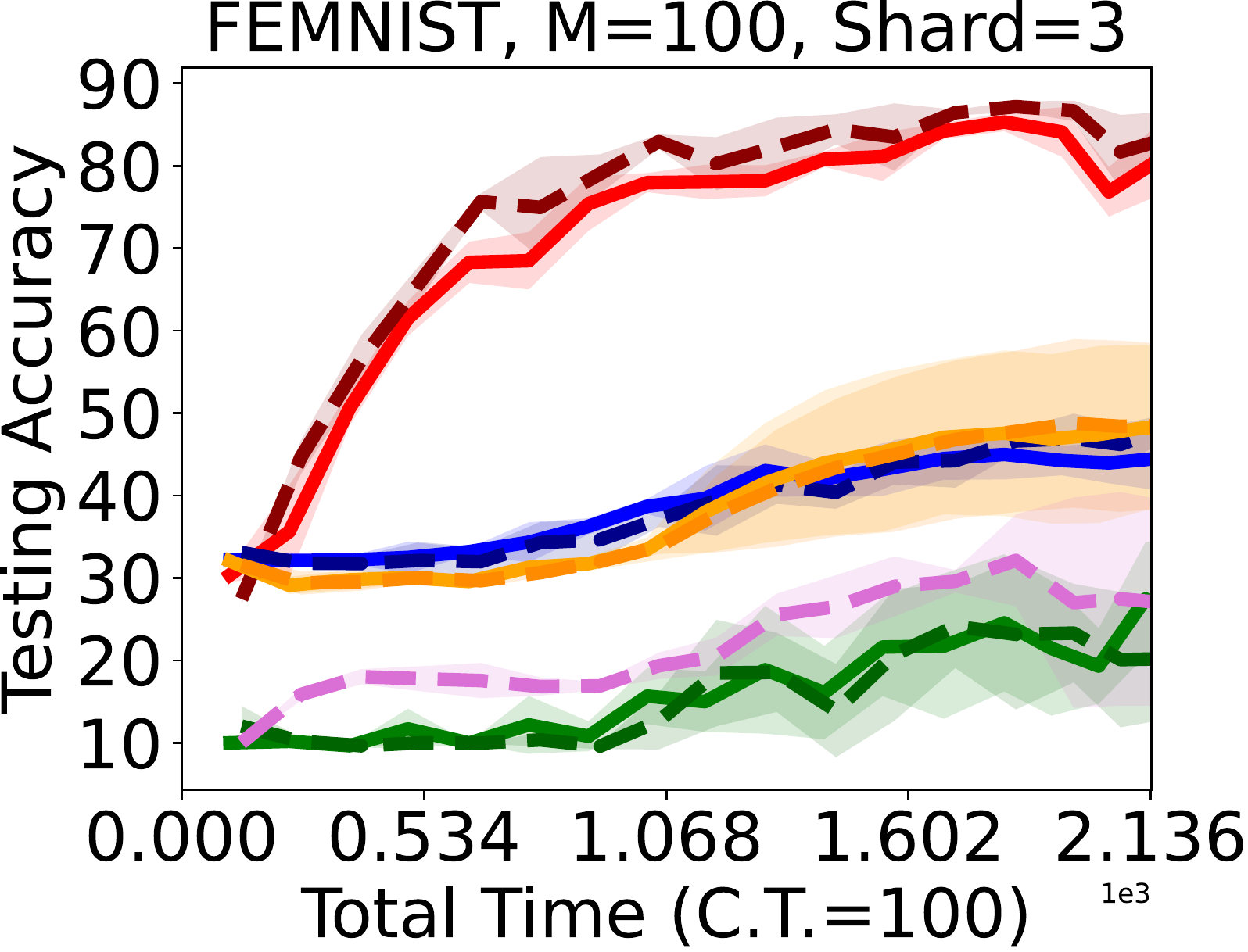}
    \end{tabular}
    \caption{Fixed computation speed setting. $20\%$ of the active nodes participate, i.e. $N = M/5$.}
    \label{fig:cifar-partial-fixed}
\end{figure*}
\begin{figure*}
    \centering
    \begin{tabular}{c@{} c@{} c@{} c}
        \includegraphics[width=.24\textwidth]{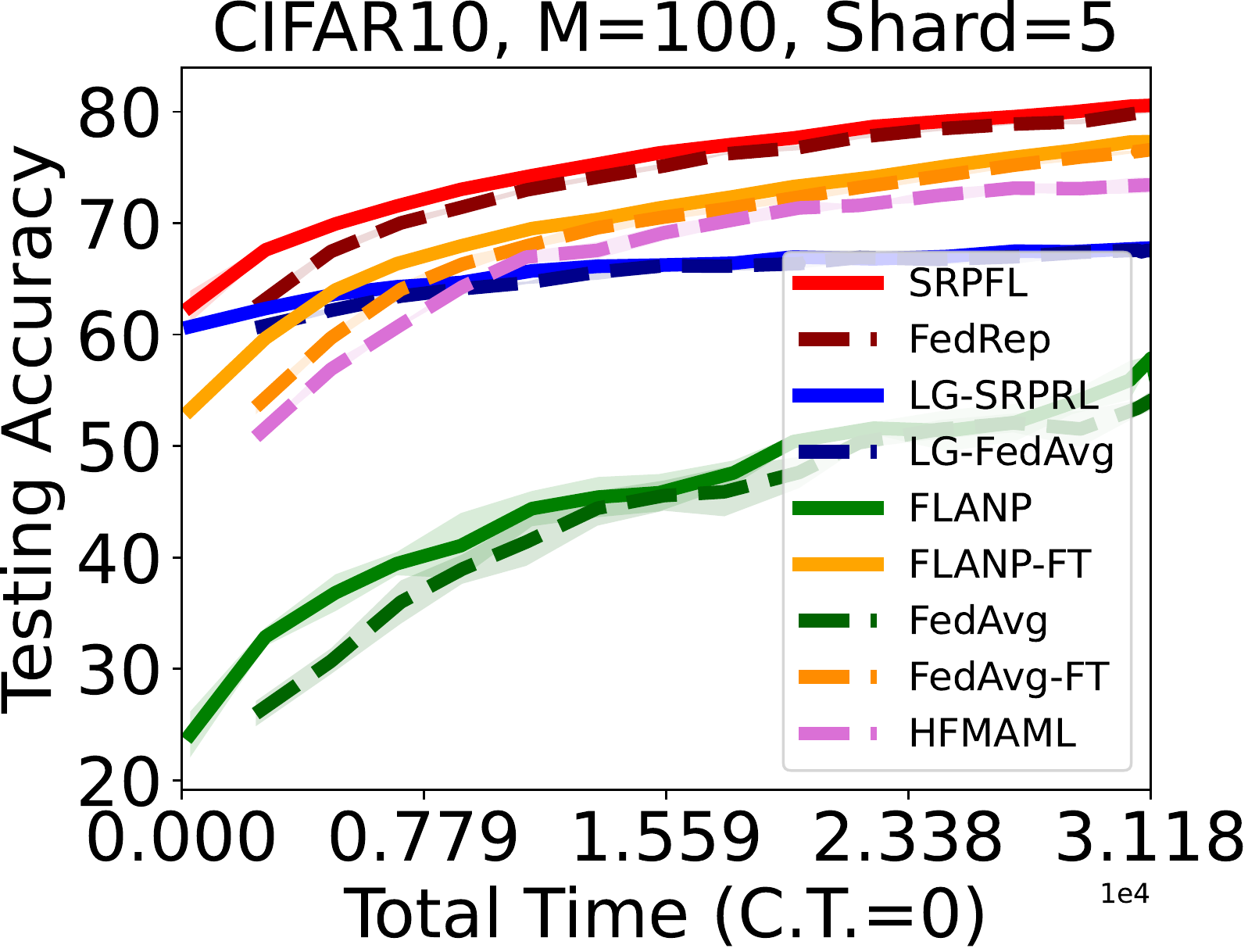} & \includegraphics[width=.24\textwidth]{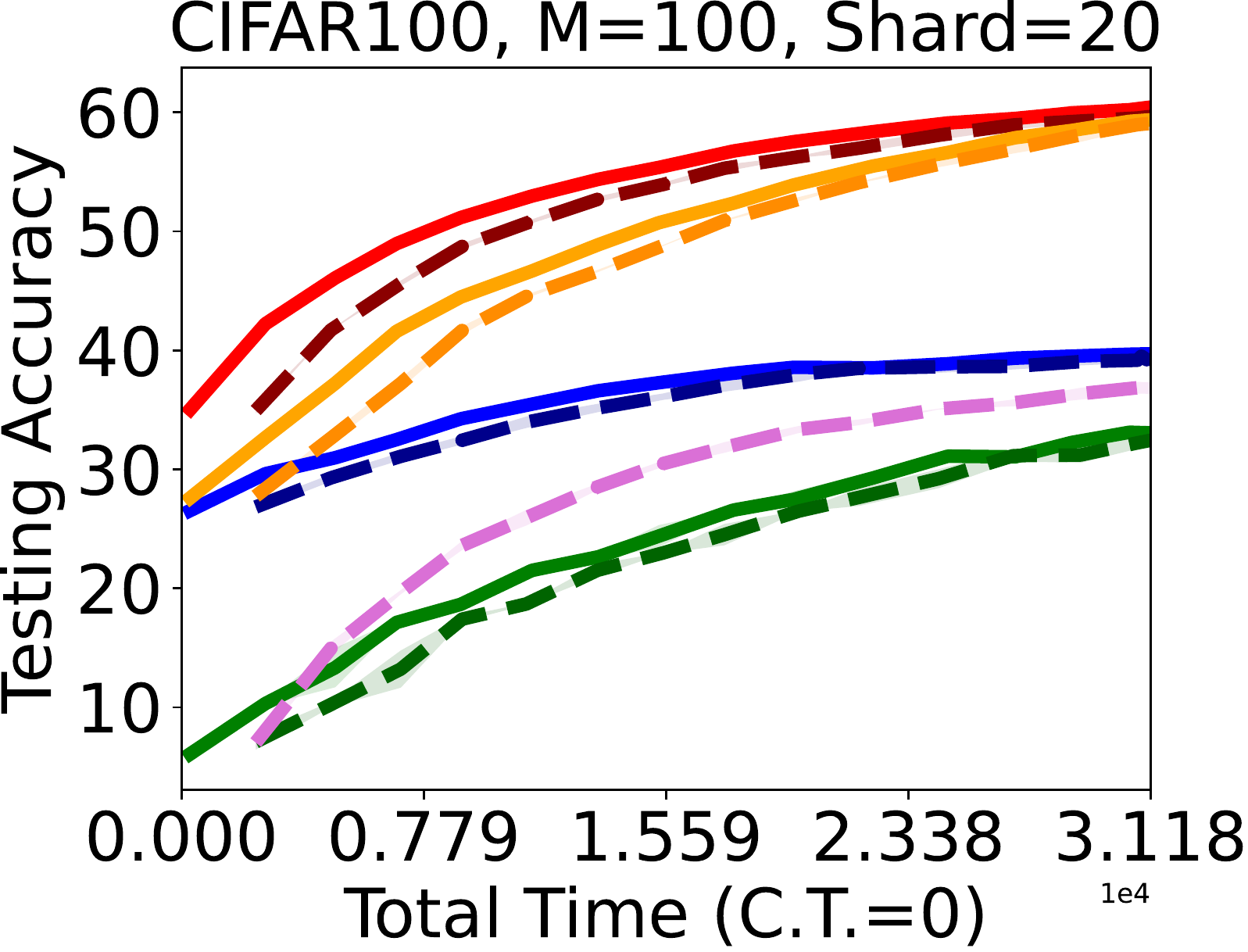} & \includegraphics[width=.24\textwidth]{figures/partial/test_acc_time_emnist_100_shard5-0-partial-participation.pdf} & \includegraphics[width=.24\textwidth]{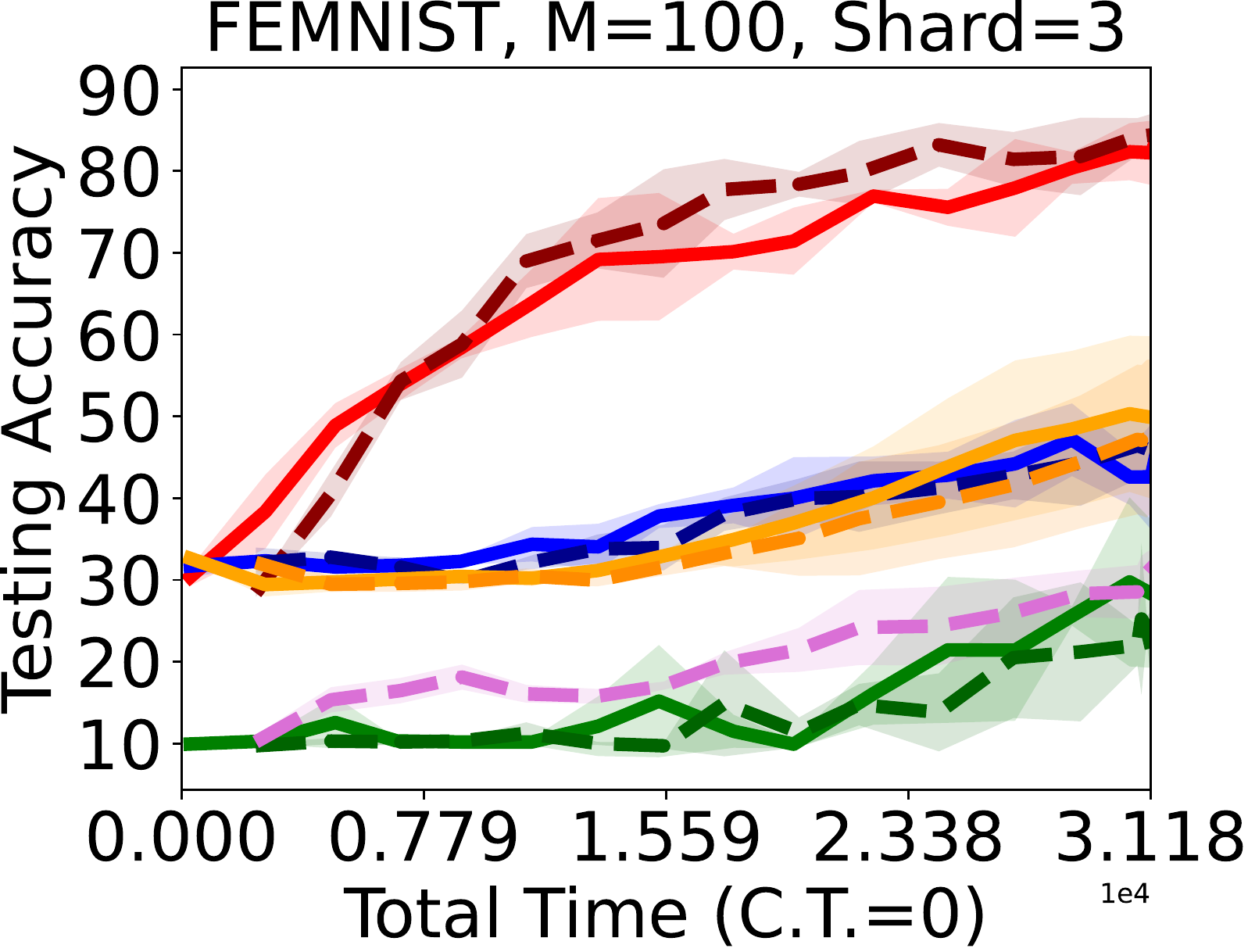}\\
        \includegraphics[width=.24\textwidth]{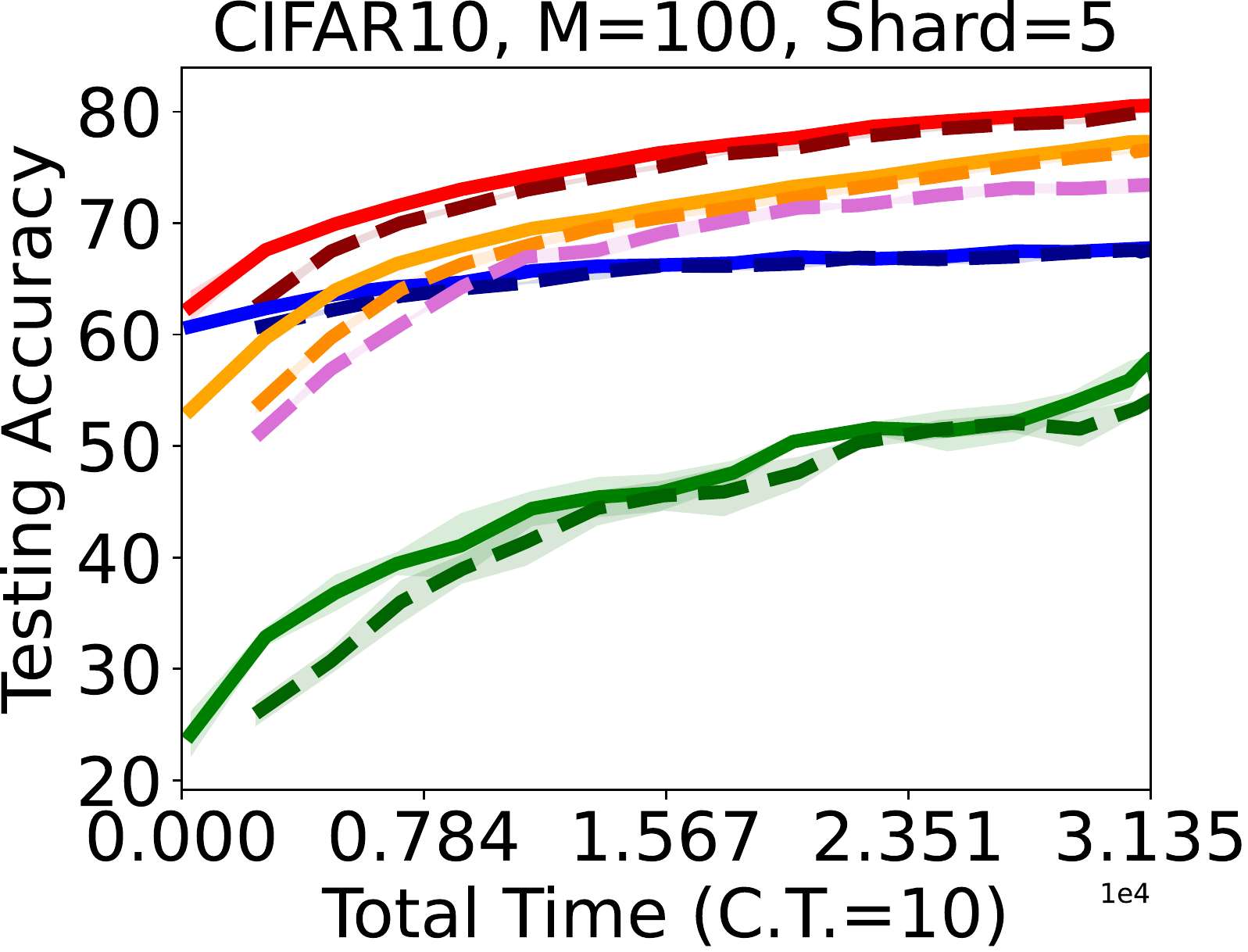} & \includegraphics[width=.24\textwidth]{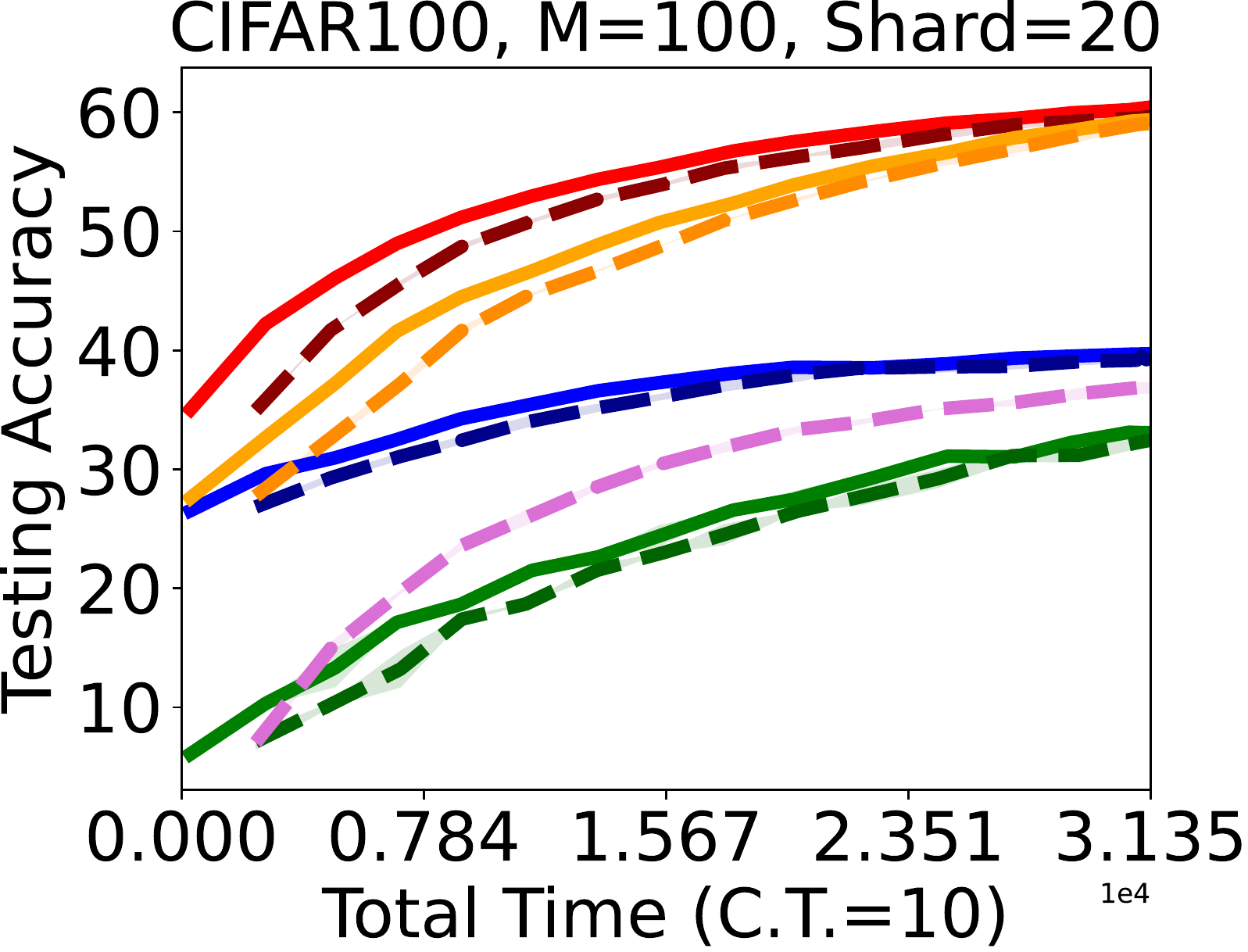} & \includegraphics[width=.24\textwidth]{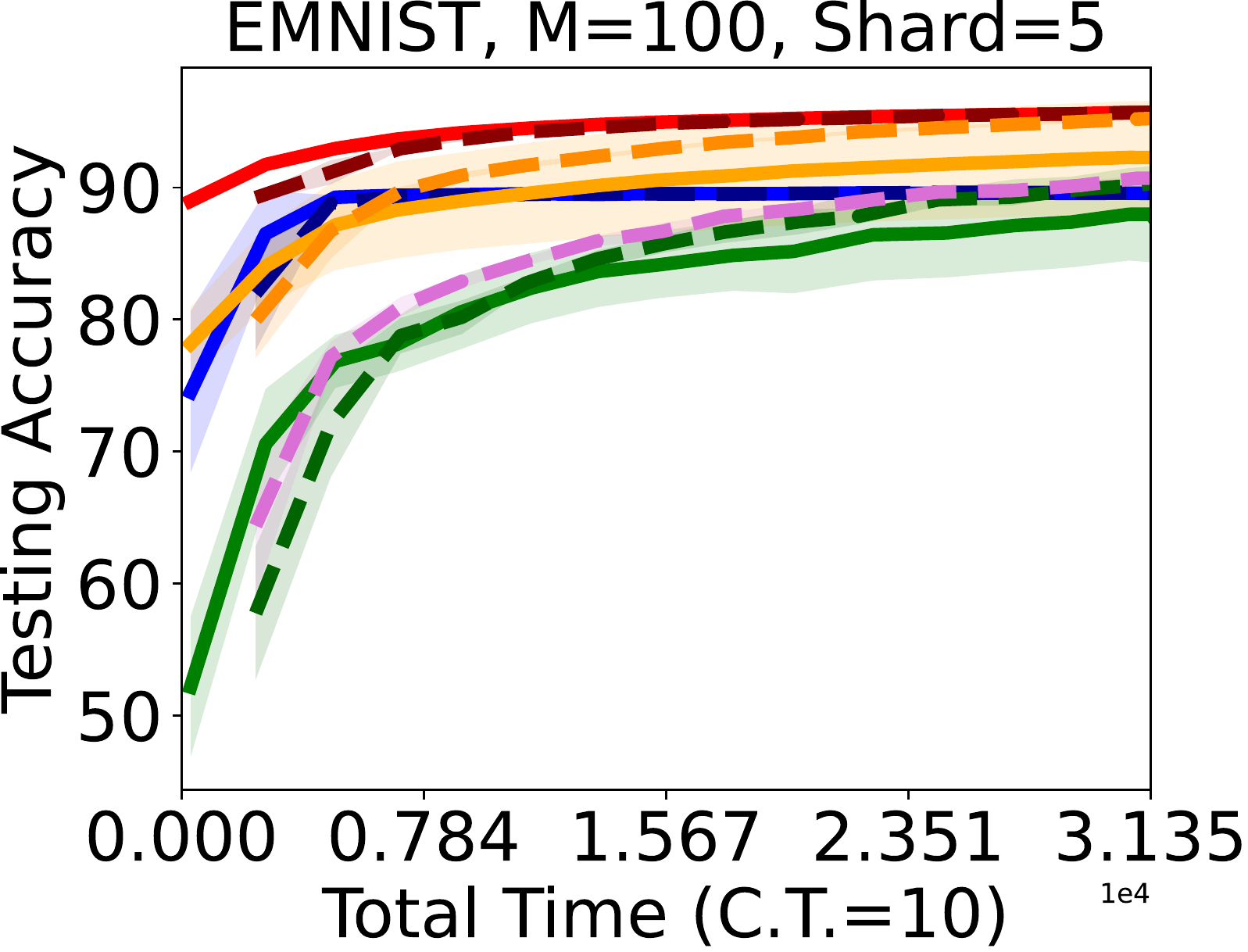}
        & 
        \includegraphics[width=.24\textwidth]{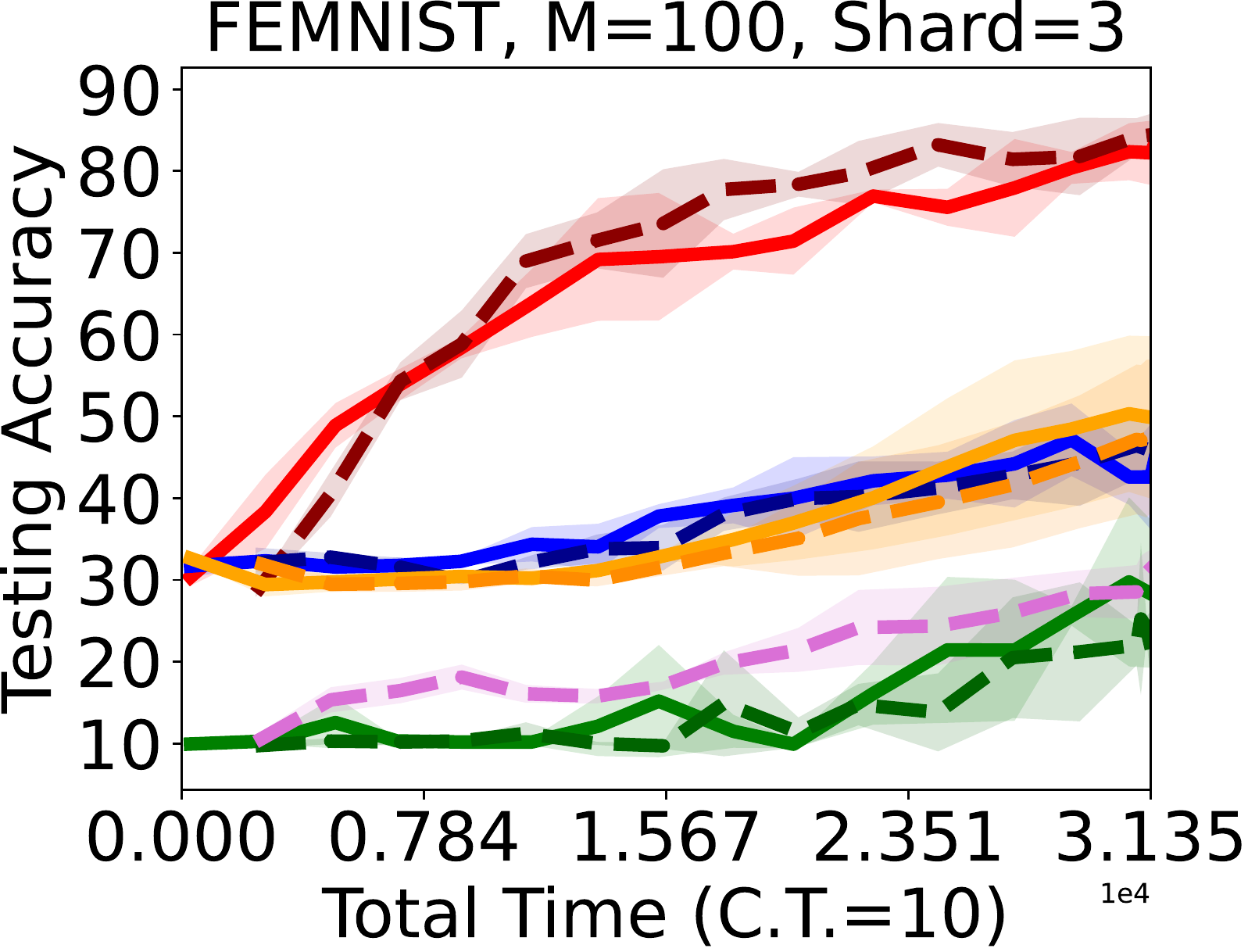}\\ 
        \includegraphics[width=.24\textwidth]{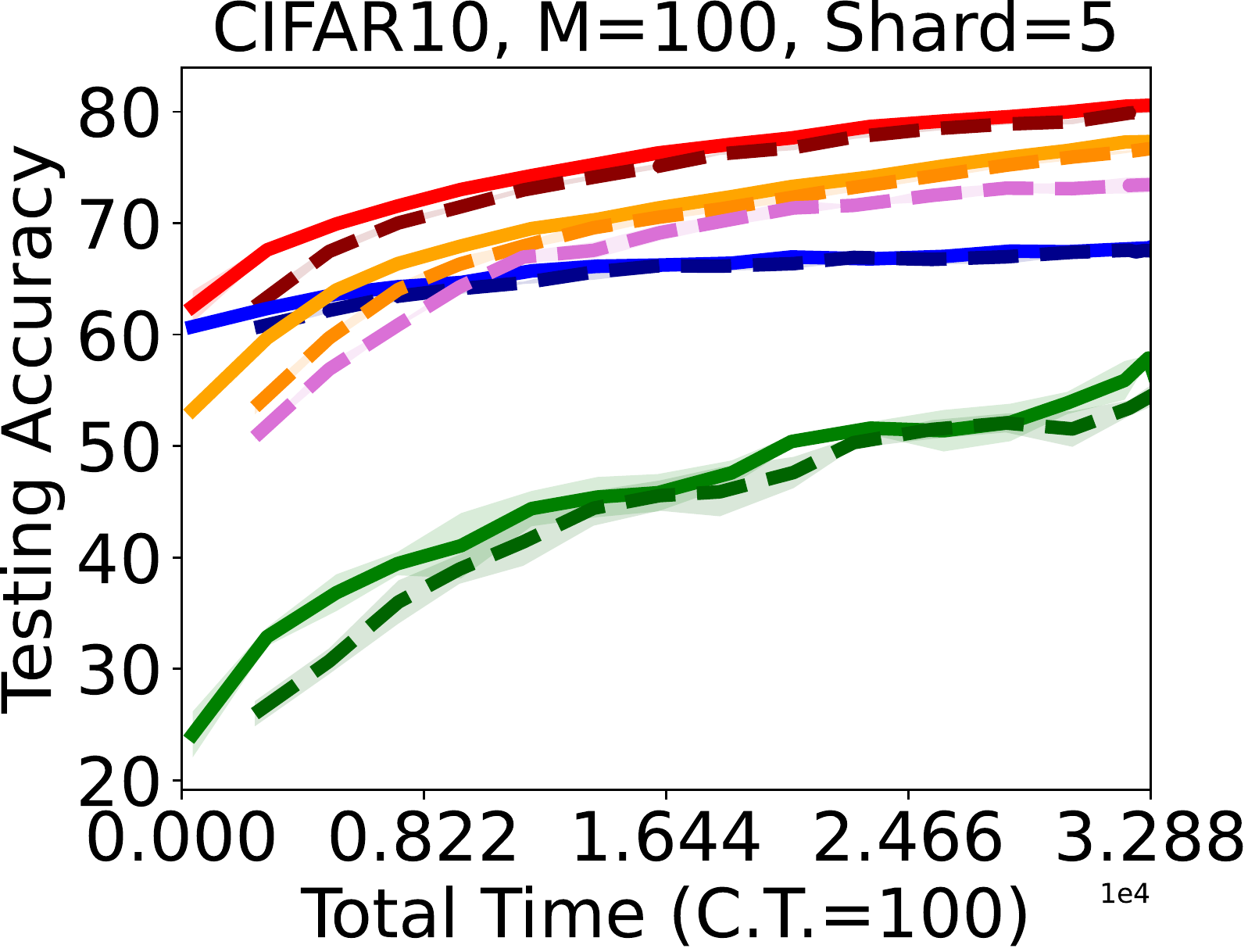} & \includegraphics[width=.24\textwidth]{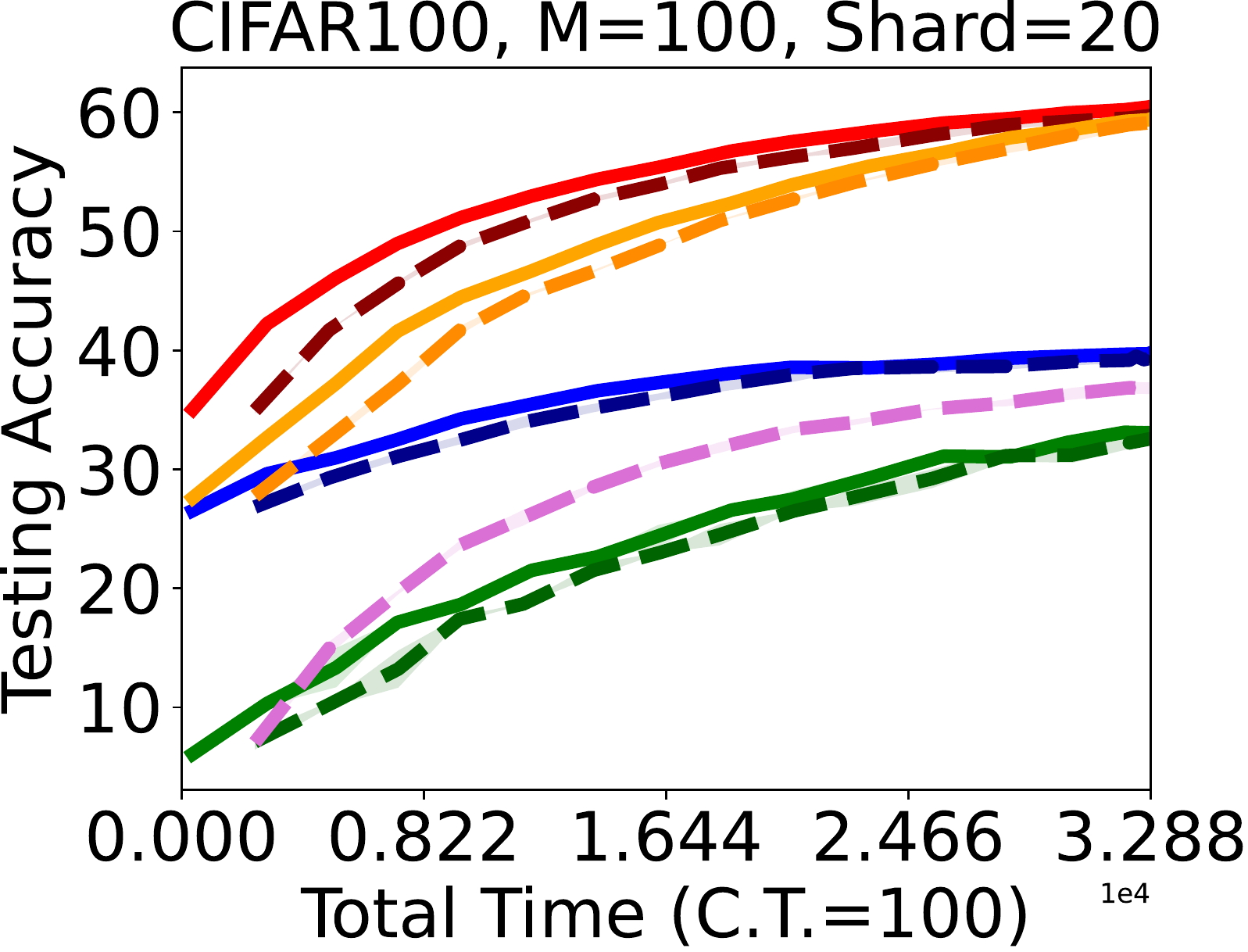} & 
        \includegraphics[width=.24\textwidth]{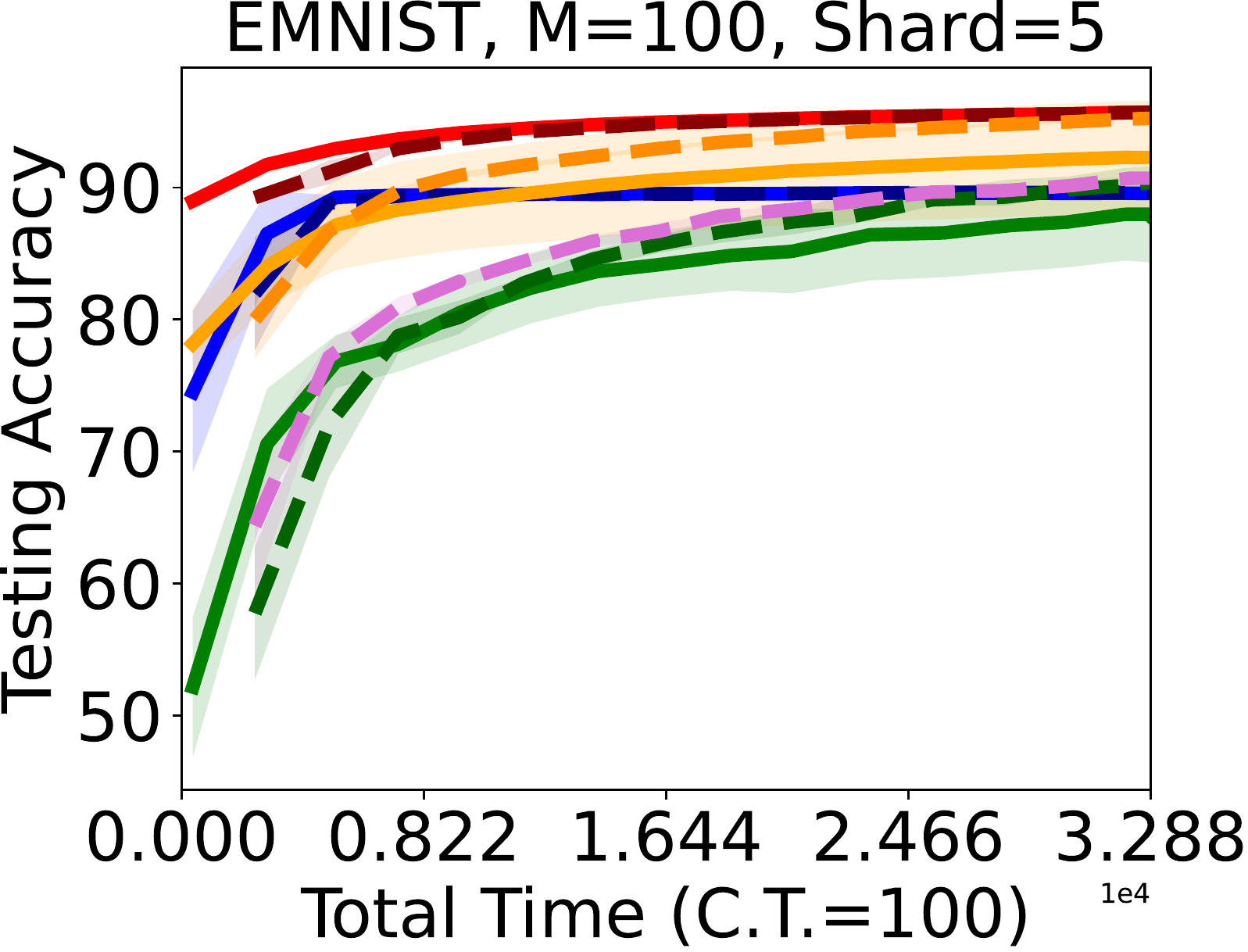}
        & \includegraphics[width=.24\textwidth]{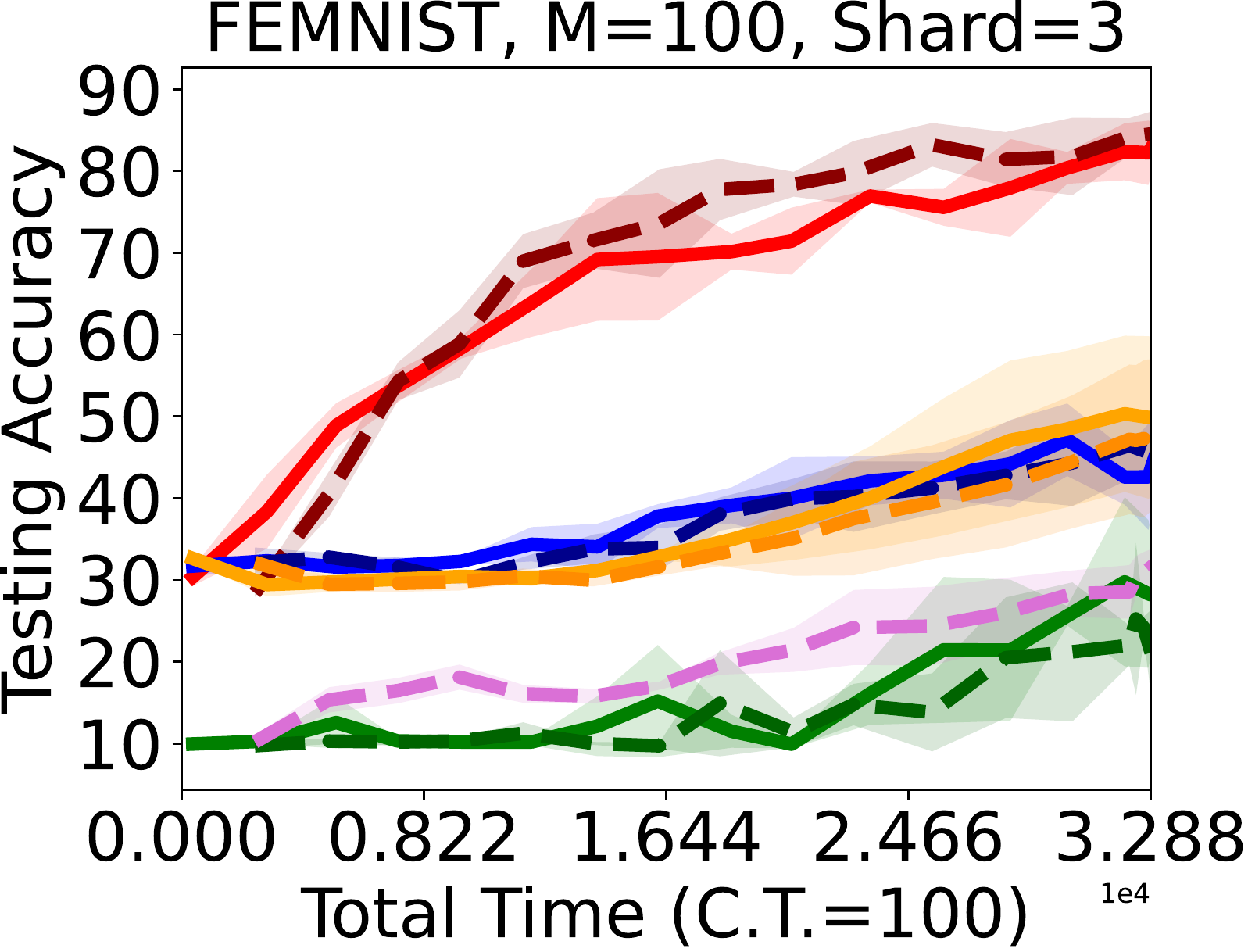}
    \end{tabular}
    \caption{Random computation speed setting. $20\%$ of the active nodes participate, i.e. $N = M/5$.}
    \label{fig:cifar-partial-random}
\end{figure*}
\end{document}